%% file: main.tex
\IfFormatAtLeastTF{2026-06-01}{%
  \RemoveFromHook{package/amsthm/after}[firstaid/aliascounter]%
}{}
\AtBeginDocument{%
  \let\PaperOriginalBibliography\bibliography
  \renewcommand{\bibliography}[1]{\PaperOriginalBibliography{refs}}%
}
\documentclass[11pt]{article}
\input{preamble_arxiv}
\input{macros}
\input{theorems}
\input{generated/result_full_versions}
\input{generated/table_layout_iclr9}

\title{One Spectrum, Two Resources:\\Data--Memory Scaling in Autoregressive Prediction}
\author{%
Chiwun Yang\thanks{\texttt{christiannyang37@gmail.com}. City University of Hong Kong.}
\and
Xiaoyu Li\thanks{\texttt{xiaoyu.li2@unsw.edu.au}. University of New South Wales.}}
\date{}
\hypersetup{
  pdftitle={One Spectrum, Two Resources: Data-Memory Scaling in Autoregressive Prediction},
  pdfauthor={Chiwun Yang, Xiaoyu Li}}

\begin{document}
\maketitle
\begin{abstract}
\input{sections/0_abstract_iclr9}
\end{abstract}

\input{arxiv/sections/paper_body}

\FloatBarrier
\bibliographystyle{plainnat}
\IfFileExists{refs.bib}{\bibliography{refs}}{\bibliography{../refs}}

\newpage
\appendix
\begin{center}
\LARGE \bf APPENDIX
\end{center}
\input{arxiv/sections/appendix_body}

\end{document}

%% file: preamble_arxiv.tex
\usepackage[T1]{fontenc}
\usepackage[utf8]{inputenc}
\usepackage{lmodern}        
\usepackage{iftex}          
\ifPDFTeX                   
  \usepackage[activate={true,nocompatibility},final]{microtype}
\else                       
  \usepackage[protrusion=true,final]{microtype}
\fi

\usepackage[margin=1in]{geometry}

\usepackage{amsmath}
\usepackage{amssymb,amsfonts}
\usepackage{mathtools}      
\usepackage{bm}             
\usepackage{nicefrac}
\allowdisplaybreaks         

\usepackage{amsthm}
\usepackage{thm-restate}    

\usepackage{graphicx}
\graphicspath{{figures/}{./}}                       
\usepackage{booktabs}                               
\usepackage{multirow,makecell,array}
\usepackage[font=small,labelfont=bf]{caption}       
\usepackage{subcaption}                             
\usepackage{placeins}                               
\usepackage{float}                                  
\usepackage{algorithm}
\usepackage{algpseudocode}                          


\usepackage{enumitem}
\setlist[itemize]{leftmargin=2.2em,itemsep=2pt,topsep=2pt}
\setlist[enumerate]{leftmargin=2.2em,itemsep=2pt,topsep=2pt}

\usepackage{xcolor}
\definecolor{LinkColor}{rgb}{0.10,0.40,0.75}        
\definecolor{CiteColor}{rgb}{0.70,0.25,0.20}        
\definecolor{UrlColor} {rgb}{0.20,0.50,0.50}        

\usepackage{tikz}
\usetikzlibrary{positioning,calc,arrows.meta}

\usepackage[round,sort&compress]{natbib}

\usepackage{url}
\usepackage{hyperref}
\hypersetup{
  colorlinks=true,
  linkcolor=LinkColor,
  citecolor=CiteColor,
  urlcolor=UrlColor,
  breaklinks=true,
  bookmarksnumbered=true,
}
\usepackage{bookmark}                               
\usepackage[capitalise,nameinlink,noabbrev,sort&compress]{cleveref}  

\numberwithin{equation}{section}

%% file: macros.tex
\newcommand{\R}{\mathbb{R}}
\newcommand{\N}{\mathbb{N}}

\newcommand{\E}{\mathbb{E}}                       
\renewcommand{\P}{\mathbb{P}}                     

\newcommand{\defeq}{\coloneqq}                    

\DeclarePairedDelimiter{\abs}{\lvert}{\rvert}
\DeclarePairedDelimiter{\norm}{\lVert}{\rVert}

\DeclarePairedDelimiter{\ceil}{\lceil}{\rceil}
\DeclarePairedDelimiter{\floor}{\lfloor}{\rfloor}
\DeclarePairedDelimiterX{\inner}[2]{\langle}{\rangle}{#1,#2}   

\DeclareMathOperator{\KL}{KL}
\DeclareMathOperator{\Ber}{Ber}
\newcommand{\sig}{\sigma}
\newcommand{\Excess}{\mathfrak{E}}
\newcommand{\Linf}{\mathcal{L}_{\infty}}
\newcommand{\ind}{\mathbf{1}}
\newcommand{\dd}{\,\mathrm{d}}

%% file: theorems.tex
\theoremstyle{plain}
\newtheorem{theorem}{Theorem}[section]
\newtheorem{proposition}[theorem]{Proposition}
\newtheorem{lemma}[theorem]{Lemma}
\newtheorem{corollary}[theorem]{Corollary}

\theoremstyle{definition}

\theoremstyle{remark}

\crefname{theorem}{Theorem}{Theorems}          \Crefname{theorem}{Theorem}{Theorems}
\crefname{proposition}{Proposition}{Propositions}
\Crefname{proposition}{Proposition}{Propositions}
\crefname{lemma}{Lemma}{Lemmas}                \Crefname{lemma}{Lemma}{Lemmas}
\crefname{corollary}{Corollary}{Corollaries}   \Crefname{corollary}{Corollary}{Corollaries}
\crefname{conjecture}{Conjecture}{Conjectures} \Crefname{conjecture}{Conjecture}{Conjectures}
\crefname{fact}{Fact}{Facts}                   \Crefname{fact}{Fact}{Facts}
\crefname{definition}{Definition}{Definitions} \Crefname{definition}{Definition}{Definitions}
\crefname{assumption}{Assumption}{Assumptions} \Crefname{assumption}{Assumption}{Assumptions}
\crefname{example}{Example}{Examples}          \Crefname{example}{Example}{Examples}
\crefname{problem}{Problem}{Problems}          \Crefname{problem}{Problem}{Problems}
\crefname{remark}{Remark}{Remarks}             \Crefname{remark}{Remark}{Remarks}
\crefname{claim}{Claim}{Claims}                \Crefname{claim}{Claim}{Claims}
\crefname{algorithm}{Algorithm}{Algorithms}    \Crefname{algorithm}{Algorithm}{Algorithms}

%% file: generated/result_full_versions.tex
\hypersetup{hypertexnames=false}
\addtotheorempostheadhook[theorem]{\leavevmode\phantomsection}
\addtotheorempostheadhook[lemma]{\leavevmode\phantomsection}
\addtotheorempostheadhook[corollary]{\leavevmode\phantomsection}
\addtotheorempostheadhook[proposition]{\leavevmode\phantomsection}

\newcommand{\FullThmSpectrum}{%
\begin{theorem}[One spectrum controls data and memory; formal version of \cref{thm:spectrum}]
\label{full:thm:spectrum}
Let \(q_i>0\), \(\sum_iq_i<\infty\), and let \(\mu\), \(\Phi_\mu\),
and \(\tau_B\) be defined above.  Then
\begin{align}
\mathcal S_q(n)&:=\sum_i\min\{q_i,n^{-1}\}=\Phi_\mu(n^{-1}),
\label{full:eq:spectrum-data}\\
D_q(B)&:=\inf_{\substack{b_i\ge0\\\sum_i b_i\le B}}
\sum_iq_i2^{-2b_i}=\Phi_\mu(\tau_B).
\label{full:eq:spectrum-bits}
\end{align}
For the bounded-logit value experiment,
\begin{equation}
\boxed{\mathfrak R^*_{\rm value}(n,B)
\asymp_R\Phi_\mu(n^{-1})+\Phi_\mu(\tau_B).}
\label{full:eq:spectrum-risk}
\end{equation}
Moreover, for every \(t>0\),
\begin{equation}
\Phi'_{\mu,+}(t)=\mu((t,\infty)),\qquad
\Phi'_{\mu,-}(t)=\mu([t,\infty)).
\label{full:eq:spectrum-derivatives}
\end{equation}
Hence the exact curve \(\Phi_\mu\) recovers the positive spectrum,
including atom multiplicities, up to coordinate relabeling.
\end{theorem}
}

\newcommand{\FullCorSpectrumRates}{%
\begin{corollary}[Spectral scaling exponents; formal version of \cref{cor:spectrum-rates}]
\label{full:cor:spectrum-rates}
If \(\mathcal N_\mu(t)\sim t^{-\rho}L(1/t)\) for \(0<\rho<1\), where
\(L\) is slowly varying, then
\[
\Phi_\mu(n^{-1})\in\operatorname{RV}_{-(1-\rho)},\qquad
\Phi_\mu(\tau_B)\in\operatorname{RV}_{-(1-\rho)/\rho}.
\]
For \(a_j\asymp j^{-(1+\chi)}\) and \(d_j\asymp j^\gamma\),
\(\rho=(1+\gamma)/(1+\chi+\gamma)\), giving rates
\(n^{-\chi/(1+\chi+\gamma)}\) and
\(B^{-\chi/(1+\gamma)}\), up to slowly varying factors.
\end{corollary}
}

\newcommand{\FullThmDyadic}{%
\begin{theorem}[Same block marginals, different scaling laws; formal version of \cref{thm:dyadic}]
\label{full:thm:dyadic}
For every \(\chi,\gamma>0\), there are two positive-conditional-entropy
causal-retrieval sources with exactly the same block-energy multiset and the
same block-dimension multiset at every dyadic level.  They differ only in the
pairing.  All displayed dimensions in this construction are rounded upward
to integers.  With the common fixed-\(K\) route revealed,
\[
\mathfrak R_{\rm easy}(n,B)
\asymp n^{-\chi/(1+\chi)}+B^{-\chi},
\]
whereas
\[
\mathfrak R_{\rm hard}(n,B)
\asymp n^{-\chi/(1+\chi+\gamma)}+B^{-\chi/(1+\gamma)}.
\]
\end{theorem}
}

\newcommand{\FullPropQK}{%
\begin{proposition}[Exact query--key realization; formal version of \cref{prop:qk}]
\label{full:prop:qk}
Under the strict causal mask exposing exactly the \(K\) preceding candidates,
\cref{eq:qk-equivalence} reproduces the route-table scores and attention
probabilities exactly.  If only the query-address block of \(W_Q\) is
trained and \(W_K\) is fixed, Euclidean SGD on \(W_Q\) induces the same
update of \(A\); the attended feature and output-vector gradient are
unchanged.
\end{proposition}
}

\newcommand{\FullPropRouting}{%
\begin{proposition}[Online routing dynamics; formal version of \cref{prop:online-routing}]
\label{full:prop:online-routing}
For \(K\ge2\), zero initialization and softmax cross-entropy steps of size
\(0<\eta\le1\) give the margin and distractor mass after \(t\) row visits:
\[
\beta_{t+1}=\beta_t+\eta\frac{K}{e^{\beta_t}+K-1},\qquad
\delta_t=\frac{K-1}{e^{\beta_t}+K-1}.
\]
After \(b_K=\Theta_\eta(\log(eK))\) visits,
\(\delta_t\asymp_\eta[1+(t-b_K)_+]^{-1}\).  For \(S_A\) independent
calibration blocks and a held-out row,
\[
\mathfrak r_K(S_A)=\E[\delta_N^2],\quad
N\sim\operatorname{Binomial}(S_A,1/K),
\]
and the hitting time to \(\mathfrak r_K(S_A)\le\varepsilon\),
\(0<\varepsilon\le1/16\), is
\(\Theta_\eta(K\{\varepsilon^{-1/2}+\log(eK)\})\).
For \(K=1\), leakage and calibration cost are zero.
\end{proposition}
}

\newcommand{\FullThmCausalRealization}{%
\begin{theorem}[Causal self-attention realization; formal version of \cref{thm:causal-realization}]
\label{full:thm:causal-realization}
For the source in \cref{sec:revision-model}, let Stage A start from zero
and use the routing update in \cref{prop:online-routing} on \(S_A\)
independent calibration blocks.  Freeze the route before Stage V, which uses
\(n\ge1\) independent labeled blocks.  With public \(p_i,\Delta_i\), set
\[
A_n=\{i:q_i\ge n^{-1}\},\quad
\mathcal W=\prod_{i\in A_n}[-\Delta_i,\Delta_i],\quad
D=\operatorname{diag}(p_i)_{i\in A_n},\quad w_1=0.
\]
Apply \cref{eq:actual-output-sgd} to the attended feature restricted to
\(A_n\), with zero output coordinates elsewhere, and use
\[
\mu_{\rm sgd}=\frac9{16}\min_{|z|\le R}\sigma'(z),\qquad
\eta_t=\frac{2}{\mu_{\rm sgd}(t+1)},\qquad
\bar w_n=\frac{2}{n(n+1)}\sum_{t=1}^n t\,w_t.
\]
If \(A_n\) is empty, take \(\bar w_n=0\).
For each public format-range ratio \(\kappa_i\ge1\) and integer width
\(k\ge0\), the scalar decoder is defined as follows.  If
\(\kappa_i2^{-k}\ge1\), its output is zero.  Otherwise quantize uniformly
with \(2^k\) midpoint values on
\([-\kappa_i\Delta_i,\kappa_i\Delta_i]\), then clip the decoded value to
\([-\Delta_i,\Delta_i]\).  The allocation is public and deterministic;
quantize \(\bar w_n\) with total width at most \(B_V\), choosing an allocation
within an arbitrarily small error of the infimum below and flooring its widths.
Let \(\mathcal E_{\rm arithmetic}\) bound the mean squared difference between
the computed and exact logits of this stored network.  Then
\begin{equation}
\E[\mathcal R_\theta(\widehat F)-\Linf]
\lesssim_{R,K,\eta}
\Phi_\mu(n^{-1})+D^{\rm fmt}_{q,\kappa}(B_V)
+\mathfrak r_K(S_A)+\mathcal E_{\rm arithmetic},
\label{full:eq:causal-realization}
\end{equation}
where \(B_V\) is the value-state bit budget and
\[
D^{\rm fmt}_{q,\kappa}(B_V)
=\inf_{b_i\ge0,\,\sum_i b_i\le B_V}
\sum_iq_i\min\{1,\kappa_i^22^{-2b_i}\}.
\]
For \(\kappa_i=1\), this equals \(\Phi_\mu(\tau_{B_V})\), and the network
attains the core surface at total budget \(B\) for power-law spectra, fixed
\(K\), \(o(B)\) routing bits, and lower-order routing/arithmetic errors.
\end{theorem}
}

\newcommand{\FullThmAdaptive}{%
\begin{theorem}[Adaptation with a known ordering; formal version of \cref{thm:adaptive}]
\label{full:thm:adaptive}
Fix \(0<\chi_{\min}\le\chi_{\max}<\infty\), and suppose the canonical mode
order is known.  Assume
\(q_j\asymp j^{-(1+\chi)}\)
with common comparison constants and unknown
\(\chi\in[\chi_{\min},\chi_{\max}]\). For every \(n,B\in\N\), one clipped
likelihood estimator and one deterministic integer bit schedule, both
independent of \(\chi\), satisfy simultaneously throughout this interval
\begin{equation}
  \sup_{\theta_j\in[-\Delta_j,\Delta_j]}
  \E_\theta\!\left[
    \mathcal L_\theta(\widehat f)-\mathcal L_\theta(\theta)
  \right]
  \lesssim
  n^{-\chi/(1+\chi)}+B^{-\chi}.
  \label{full:eq:adaptive-rate}
\end{equation}
\end{theorem}
}

\newcommand{\FullThmTransformer}{%
\begin{theorem}[Finite-precision realization; formal version of \cref{thm:transformer}]
\label{full:thm:transformer}
Fix \(K,M\ge1\), \(R<\infty\), \(\beta\ge0\), and the construction above.

\begin{enumerate}
  \item In exact arithmetic, if \(\abs{w_j^{\rm mode}}\le R\) for every
  \(j\in[M]\), then a query whose matching candidate has any mode \(j\ge1\)
  satisfies
  \begin{equation}
    \abs{z-\vartheta_j}\le2R(K-1)e^{-\beta}.
    \label{full:eq:routing-bound}
  \end{equation}

  \item Set \(w_j^{\rm mode}=\theta_j\) for \(j\in[M]\) and store it with
  \(b_j\in\N_0\) bits on
  \([-\Delta_j,\Delta_j]\).  Assume
  stable max-shifted softmax with exactly representable shifted scores \(0,-\beta\)
  and an exponent range containing their exponentials. Exponentials,
  multiplications, and divisions have relative error at most \(u\);
  sums use balanced pairwise accumulation. For the resulting predictor
  \(f_{\rm fp}\), if
  \((\ceil{\log_2K}+4)u\le1/4\), then the held-out excess cross-entropy obeys
  \begin{equation}
    \mathcal L_\theta(f_{\rm fp})-\Linf
    \le
    C_R\left\{
      (K-1)^2e^{-2\beta}
      +u^2(1+\log K)^2
      +\sum_{j=1}^M q_j2^{-2b_j}
      +\sum_{j>M}q_j
    \right\}.
    \label{full:eq:fp-bound}
  \end{equation}

  \item Within the separable scalar-table representation in which coordinate
  \(j\) has at most \(2^{b_j}\) values, the bit-dependent term in
  \cref{eq:fp-bound} is tight up to constants.
\end{enumerate}
Here and below, \(C_R\) depends only on the logit bound \(R\).
\end{theorem}
}

\newcommand{\FullThmCompute}{%
\begin{theorem}[Precision-aware compute phases; formal version of \cref{thm:compute}]
\label{full:thm:compute}
Let \(\chi>0\) and \(f,\nu\ge0\).  Optimize over configurations using
\(n(M)\asymp M^{1+\chi}\) blocks whose total excess risk and cost satisfy,
uniformly over \(M,b\ge1\),
\[
\mathcal E(M,b)\asymp M^{-\chi}+\kappa(M)^22^{-2b},\qquad
\operatorname{cost}(M,b)\asymp M^{f+1+\chi}b^\nu.
\]
Let \(\mathcal R(C)-\Linf=\inf_{\operatorname{cost}(M,b)\le C}\mathcal E(M,b)\).
Then:
\begin{enumerate}
\item If \(1\lesssim\kappa(M)\lesssim M^\rho\), then
the least model-error-matching precision is \(b=\Theta(\log M)\), and
\begin{equation}
\mathcal R(C)-\Linf\asymp
C^{-\chi/(f+1+\chi)}
(\log C)^{\chi\nu/(f+1+\chi)}
\label{full:eq:poly-conditioning}
\end{equation}
up to other slowly varying factors specified in the assumptions.
\item If \(\kappa(M)=\exp(\Theta(M^\lambda))\), \(\lambda>0\), then
\(b=\Theta(M^\lambda)\), and the loss exponent is
\(\chi/(f+1+\chi+\lambda\nu)\).
\item If \(q_j\asymp j^{-(1+\chi)}\), adaptive per-mode precision has
error \(\Theta(B^{-\chi})\), whereas a uniform-width prefix has optimal
error
\begin{equation}
\Theta((B/\log B)^{-\chi}).
\label{full:eq:uniform-penalty}
\end{equation}
\end{enumerate}
Each converse uses the representation, operation set, storage layout, and
bit-operation cost specified above.
\end{theorem}
}

\newcommand{\FullLemBitComplexity}{%
\begin{lemma}[Transformer forward--backward bit complexity; formal version of \cref{lem:bit-complexity}]
\label{full:lem:bit-complexity}
For the graph above, with positive dimensions and conventional dense matrix
multiplication,
\begin{equation}
  \operatorname{BitOps}_{\rm fb}
  \asymp L\Bigl[
    \bigl(\ell d^2+\ell d m_{\rm ff}+\ell^2d\bigr)
      \mathsf M(b_{\rm mm})
    +H\ell^2\mathsf S(b_{\rm sm})
  \Bigr].
  \label{full:eq:fb-bitops}
\end{equation}
\begin{equation}
  \operatorname{Mem}_{\rm fb}
  \asymp
    b_wL(d^2+dm_{\rm ff})
    +b_aL\ell(d+m_{\rm ff})
    +b_{\rm st}LH\ell^2
    +\operatorname{Mem}_{\rm hp}.
  \label{full:eq:fb-memory}
\end{equation}
The memory \(\operatorname{Mem}_{\rm fb}\) is measured in bits.  The values
\(b_w\), \(b_a\), and \(b_{\rm st}\) count aggregate bits per logical weight,
activation, and saved-attention slot at peak memory; they include any
same-shaped gradient or backward-workspace buffer materialized by this eager
implementation.  The term
\(\operatorname{Mem}_{\rm hp}\) contains master weights, optimizer state, and
all other explicitly retained high-precision data.  A streaming, recomputing
attention implementation replaces the quadratic attention state saved for
backward by linear online-softmax state plus temporary tile workspace, while
leaving the dense attention arithmetic order unchanged.
\end{lemma}
}

\newcommand{\FullThmPrecisionTradeoff}{%
\begin{theorem}[Compute--precision Pareto law; formal version of \cref{thm:precision-tradeoff}]
\label{full:thm:precision-tradeoff}
Assume \cref{eq:mixed-risk-relations,eq:mixed-cost-relations},
\(\kappa_{\rm mm}(M)\asymp M^r\), and
\(\kappa_{\rm sm}(M)\asymp M^{r_{\rm sm}}\), with all displayed relations
two-sided and \(r,r_{\rm sm}\ge0\).  For the specified dense graph,
\(H\le d\) implies \(G\gtrsim S\) and hence \(g\ge s\).
Item 3 concerns general operation counts satisfying the same two-sided
relations; items 1, 2, and 4 apply to both settings.
\begin{enumerate}
  \item The smallest matrix precision that keeps matrix rounding at the model
  error is
  \begin{equation}
    b_{\rm match}(M)
    =\left(r+\frac{\chi}{2}\right)\log_2M+O(1).
    \label{full:eq:matching-bits}
  \end{equation}

  \item If the matrix term dominates the bracket in
  \cref{eq:mixed-cost-relations} at matched precisions---in particular in the
  case \(g\ge s\) when the remaining precision factors do not reverse
  that dominance---then
  \begin{equation}
    \inf_{M,b,b_{\rm sm}:\,\mathrm{cost}\le C}\mathcal E
    \asymp
    C^{-\chi/(1+\chi+g)}
    (\log C)^{\chi\nu/(1+\chi+g)},
    \label{full:eq:matrix-dominated-phase}
  \end{equation}
  up to other declared slowly varying factors.

  \item For general operation counts, if \(s>g\), \(b_{\rm sm}\) is fixed,
  and the softmax term dominates
  computation, then for budgets in the pre-floor range
  \[
    C^{1/(1+\chi+s)}
    \lesssim 2^{2b_{\rm sm}/(\chi+2r_{\rm sm})},
  \]
  \begin{equation}
    \inf_{M,b:\,\mathrm{cost}\le C}\mathcal E
    \asymp C^{-\chi/(1+\chi+s)}.
    \label{full:eq:softmax-dominated-phase}
  \end{equation}
  Reducing matrix precision cannot improve this leading exponent.

  \item At fixed matrix precision \(b\), with softmax precision optimized,
  the model--rounding crossover and the infimum error over model scales are
  \begin{equation}
    M_b\asymp2^{2b/(\chi+2r)},
    \qquad
    \mathcal E_{\rm floor}(b)
    \asymp2^{-2\chi b/(\chi+2r)}.
    \label{full:eq:fixed-bit-floor}
  \end{equation}
  A fixed \(b_{\rm sm}\)-bit softmax similarly supports the scaling phase only
  while
  \begin{equation}
    M\lesssim2^{2b_{\rm sm}/(\chi+2r_{\rm sm})}.
    \label{full:eq:softmax-horizon}
  \end{equation}
\end{enumerate}
The displayed two-sided risk and arithmetic assumptions define the class over
which the infima and matching bounds are taken.
\end{theorem}
}

\newcommand{\FullLemJointCodebook}{%
\begin{lemma}[Arbitrary joint finite-state representations; formal version of \cref{lem:joint-codebook}]
\label{full:lem:joint-codebook}
Let \(\Theta_j\) be independent and uniform on
\([-\Delta_j,\Delta_j]\), let \(p_j>0\), and put
\(q_j=p_j\Delta_j^2\).  For \(B\in\N_0\), every encoder--decoder pair with at most \(2^B\)
reconstruction vectors satisfies
\begin{equation}
  \sum_jp_j\E(\Theta_j-\widehat\Theta_j)^2
  \ge
  \frac{2}{\pi e}
  \inf_{\substack{b_j\ge0\\\sum_jb_j\le B}}
  \sum_jq_j2^{-2b_j}.
  \label{full:eq:joint-codebook}
\end{equation}
The reconstruction vectors may couple all coordinates.
\end{lemma}
}

\newcommand{\FullThmFiniteBit}{%
\begin{theorem}[Finite-bit autoregressive minimax law; formal version of \cref{thm:finite-bit}]
\label{full:thm:finite-bit}
Fix \(R<\infty\).  For every \(n\in\N\), \(B\in\N_0\), and source family
specified by \((p_j,\Delta_j)_{j\ge1}\) as in \cref{sec:source}, there are
constants \(0<c_R\le C_R<\infty\) such that
\begin{equation}
  c_R\{S_q(n)+D_q(B)\}
  \le
  \Excess^*(n,B)
  \le
  C_R\{S_q(n)+D_q(B)\},
  \label{full:eq:finite-bit-law}
\end{equation}
where
\begin{equation}
  S_q(n)\defeq\sum_j\min\{q_j,n^{-1}\},
  \qquad
  D_q(B)\defeq
  \inf_{\substack{b_j\ge0\\\sum_jb_j\le B}}
  \sum_jq_j2^{-2b_j}.
  \label{full:eq:SqDq}
\end{equation}
The lower bound holds for arbitrary \(2^B\)-state joint encoders and arbitrary
fixed shared decoding computation.
\end{theorem}
}

\newcommand{\FullCorPowerProfile}{%
\begin{corollary}[Power-law energies; formal version of \cref{cor:power-profile}]
\label{full:cor:power-profile}
Suppose that \(n,B\ge1\) and, for some \(\chi>0\) and constants
\(0<c_-\le c_+<\infty\),
\[
  c_-j^{-(1+\chi)}
  \le q_j\le
  c_+j^{-(1+\chi)}
  \qquad(j\ge1).
\]
Then
\begin{equation}
  \Excess^*(n,B)
  \asymp
  n^{-\chi/(1+\chi)}+B^{-\chi}.
  \label{full:eq:power-pac}
\end{equation}
The data--bit crossover occurs at
\(B\asymp n^{1/(1+\chi)}\).
\end{corollary}
}

\newcommand{\FullThmMarginalsFail}{%
\begin{theorem}[Marginals do not identify scaling; formal version of \cref{thm:marginals-fail}]
\label{full:thm:marginals-fail}
For every \(\chi,\gamma>0\), there are two instances of
\cref{eq:selection-action} with identical energy multisets and identical cost
multisets at every dyadic level, but
\begin{equation}
  \mathcal A_{\rm easy}(C)\asymp C^{-\chi},
  \qquad
  \mathcal A_{\rm hard}(C)\asymp C^{-\chi/(1+\gamma)}.
  \label{full:eq:coupling-separation}
\end{equation}
The same two exponents hold when selection is replaced by continuous bits
with distortion \(q_j2^{-2b_j}\) and cost \(c_jb_j\).  Hence no summary
that retains only the two marginal multisets identifies the optimal exponent
in either representation class.
\end{theorem}
}

\newcommand{\FullThmJointSpectrum}{%
\begin{theorem}[Joint energy--difficulty characterization; formal version of \cref{thm:joint-spectrum}]
\label{full:thm:joint-spectrum}
For a finite collection of modes, \(\eta\) is a maximal invariant under
permutations of mode labels: it retains all information except the arbitrary
names of the modes.  For a finite or countable collection, consider the
separable problem \cref{eq:weighted-distortion} and assume
\[
  \sum_ja_j<\infty,
  \qquad
  \sum_{j:a_j/c_j>t}c_j<\infty
  \quad\text{for every }t>0.
\]
Then the spectrum \(\mu=\sum_jc_j\delta_{a_j/c_j}\) determines
\(D(B)\) for every \(B\ge0\).  Conversely, the complete curve
\(B\mapsto D(B)\) determines \(\mu\) on \((0,\infty)\).
Consequently, \(\mu\) is a minimal exact unlabeled summary of the complete
budget--distortion curve for this representation class.
\end{theorem}
}

%% file: generated/table_layout_iclr9.tex
\newsavebox{\IclrNineTableBox}
\newcommand{\IclrNineTable}[2]{%
  \sbox{\IclrNineTableBox}{#2}%
  \typeout{ICLR9TABLE #1 natural=\the\wd\IclrNineTableBox available=\the\linewidth}%
  \ifdim\wd\IclrNineTableBox>\linewidth
    \resizebox{\linewidth}{!}{\usebox{\IclrNineTableBox}}%
  \else
    \usebox{\IclrNineTableBox}%
  \fi}

\newsavebox{\IclrPairA}
\newsavebox{\IclrPairB}
\newlength{\IclrPairWidthA}
\newlength{\IclrPairWidthB}
\newlength{\IclrPairGap}
\ExplSyntaxOn
\NewDocumentCommand{\IclrPreparePair}{mm}{
  \sbox{\IclrPairA}{#1}
  \sbox{\IclrPairB}{#2}
  \dim_set:Nn \IclrPairWidthA {
    \fp_to_dim:n { 0.98 * \dim_to_fp:n {\linewidth}
      * \dim_to_fp:n {\wd\IclrPairA}
      / (\dim_to_fp:n {\wd\IclrPairA} + \dim_to_fp:n {\wd\IclrPairB}) }
  }
  \dim_set:Nn \IclrPairGap { \fp_to_dim:n { 0.02 * \dim_to_fp:n {\linewidth} } }
  \setlength{\IclrPairWidthB}{\dimexpr\linewidth-\IclrPairGap-\IclrPairWidthA\relax}
}
\ExplSyntaxOff

%% file: sections/0_abstract_iclr9.tex
How much learned memory is needed to benefit from more data? We show that
the two resources are governed by one predictive-energy spectrum in a
positive-entropy autoregressive retrieval source. Each coordinate contributes
its query probability times the squared radius of its unknown logit.
Writing \(\mu\) for the resulting energy spectrum, we prove the minimax law
$\mathfrak R^*_{\rm value}(n,B)\asymp_R
\Phi_\mu(n^{-1})+\Phi_\mu(\tau_B),
\Phi_\mu(t)=\int\min\{x,t\}\,\mu(\mathrm dx),$
for \(n\) prediction blocks and a learned state with at most \(2^B\) values.
Data set the resolution \(1/n\); memory sets the level \(\tau_B\) reached
by optimal bit allocation. The complete curve also recovers the positive
spectrum. Energy--dimension pairing is essential: two causal sources with
identical block-energy and block-dimension marginals have different data and
memory exponents. A masked query--key attention head learns the route and
values, realizing the law with explicit routing, format, and arithmetic errors. Further
results give exponent-adaptive allocation, finite-precision realization,
and compute--precision laws under two-sided arithmetic assumptions.
Experiments recover the data--memory collapse and coupling exponents,
explain the routing and allocation mechanisms, and examine weight-only
quantization across six pretrained-model scales.

\vspace{-3mm}

%% file: arxiv/sections/paper_body.tex
\input{arxiv/introduction}
\input{arxiv/related_work}
\input{sections/revision_2_model_iclr9}
\input{arxiv/sections/revision_3_unified}
\input{arxiv/sections/revision_4_coupling}
\input{arxiv/sections/revision_5_learning}

\input{arxiv/sections/revision_5b_branches}

\input{arxiv/sections/revision_5c_compute}

\input{arxiv/sections/revision_6_experiments}

\input{arxiv/sections/revision_7_conclusion}

%% file: arxiv/introduction.tex
\section{Introduction}
\label{sec:revision-intro}

More data reveal more predictive structure, but a predictor benefits only
from what its learned state can retain. When should we collect more data,
and when should we increase learned memory? Empirical scaling laws relate
loss to data, model size, and compute
\citep{kaplan2020scaling,hoffmann2022compute}. To understand how these
resources work together, we seek a property of the prediction problem that
determines both the information available in finite data and the
information preserved by a finite learned state.

Consider a prediction task with many independently varying pieces of
knowledge. A piece that is queried often or strongly changes the next-token
probability deserves more learning and storage resources than one that is
rarely used or has little predictive effect. Our model expresses each piece
as an unknown scalar prediction parameter, called a \emph{coordinate}.
The key quantity is predictive energy per independent coordinate. In our
autoregressive source, coordinate \(i\) has query probability \(p_i\) and an
unknown logit, or log-odds of a binary target, in \([-R_i,R_i]\). Its predictive energy
\(q_i=p_iR_i^2\) measures its potential contribution to prediction error.
Collecting these energies, counting repeated values separately, gives the
spectrum. At a chosen error resolution, small coordinates contribute their
whole energy and larger coordinates contribute only the unresolved portion.
Adding these contributions defines a spectral curve. Finite data and finite learned memory read
\emph{the same spectral curve at different resolutions}.

This connection gives a joint resource law. Let \(n\) count independent
training blocks and let \(B\) count the bits retained after training, so
the learner can finish in at most \(2^B\) different states. We measure
excess cross-entropy above the loss of a predictor that knows the true
conditional probabilities. First consider value estimation with the
relevant coordinate revealed. Data set an error resolution of order \(1/n\);
memory allocates bits until retained coordinates reach a common residual-error
level. The smallest worst-case expected excess loss is, up to constants,
the sum of the spectral curve at these two resolutions. This minimax result
allows unrestricted training computation and joint coding across coordinates.
Section~\ref{sec:revision-model} defines the source, risk, and two resolutions
before stating the full formula.

The law makes the resource balance explicit. List coordinate energies in
decreasing order and suppose \(q_{(i)}\asymp i^{-(1+\chi)}\), where
\(\chi>0\) is their decay exponent and \(\asymp\) denotes equality up to
constant factors. The two terms scale as
\(n^{-\chi/(1+\chi)}+B^{-\chi}\) and balance at
\(B\asymp n^{1/(1+\chi)}\). Below this memory scale, retaining learned
information is the bottleneck; above it, observations matter more.
More generally, an energy group can spread its total energy across several
independent coordinates. Increasing that dimension produces more parameters
to learn, each with less predictive energy. The general scaling corollary
quantifies this effect through the distribution of coordinate energies.
The full spectral curve applies beyond power laws and recovers every
positive coordinate energy with its multiplicity.

This perspective reveals information that separate scaling descriptors
discard. Concentrating energy in one coordinate makes it easier to learn
and retain than spreading it across many independent coordinates.
We construct two causal sources with exactly the same block-energy and
block-dimension multisets at every dyadic level. Pairing high energy
with small blocks or large blocks changes the coordinate spectrum and
both resource exponents. The pairing, not either marginal alone,
determines how predictive information is distributed.

The source also makes the learning mechanism explicit. Each regeneration
block contains \(K\) candidate tokens, a query, and a stochastic next token.
An unknown route selects the relevant candidate, whose bounded Bernoulli
logit is also unknown. Calibration blocks reveal the correct candidate;
prediction blocks teach its value; evaluation uses independent held-out
blocks. Fixed embeddings and projections contain neither unknown object.
A fixed-key, trainable-query attention head learns the route, followed
by frequency-preconditioned output-parameter learning. Its masked
query--key scores and gradients match the construction in the proof,
connecting the spectral law to a concrete two-stage learner.

The same resource picture organizes representation and computation.
A universal allocation schedule adapts to an unknown power-law exponent
when coordinate order is known. Finite-precision evaluation quantifies
routing, storage, and arithmetic errors. For the specified dense
computational graph, operation counts and numerical assumptions connect
the precision needed for a target error to bit complexity and conditional
compute phases. These results explain how statistical resolution becomes
a representational and computational requirement.

\begin{itemize}
\item \textbf{One spectrum, two resources.}
An exact truncated-energy curve expresses both terms of arbitrary-state
minimax risk, recovers the positive spectrum, and determines the data and
learned-memory exponents. A masked-attention learner realizes the spectral
surface with explicit routing, format, and arithmetic errors; adaptive
allocation and precision analysis connect this law to implementation.
\item \textbf{Why separate marginals fail.}
Same-marginal causal sources have different coordinate spectra and different
data and memory exponents. Their energy--dimension pairing identifies
the information missing from separate scaling descriptors.
\end{itemize}

Our experiments vary sample and state budgets on a fixed source, test
same-marginal separation, and compare learned, known, and frozen routing.
Interventions and perturbed energy profiles examine the mechanisms behind
the measured curves. Continued training tests the interaction between added
data and deployed storage, while weight-only post-training quantization
across six pretrained Transformer scales measures the effect of numerical
representation at fixed checkpoints.

%% file: arxiv/related_work.tex
\section{Related Work}
\label{sec:related-work}

Our question connects scaling laws, finite-rate representation, and
distribution-specific learning: which source statistic determines the
value of additional observations and additional learned memory?
The following literatures provide the statistical and computational
ingredients. \Cref{tab:prior-comparison} summarizes the closest
theorem-level connections.

\paragraph{Empirical and mechanistic scaling laws.}
Empirical studies found extended power-law regimes in neural prediction
as data, model size, and computation grow
\citep{hestness2017deep,kaplan2020scaling,henighan2020autoregressive}.
Compute-optimal work asks how a training budget should be divided between
parameters and tokens \citep{hoffmann2022compute}, while routed models
distinguish total parameters from parameters used per token
\citep{clark2022routed}. Extrapolation studies, broken-power-law fits,
and replications show that fitted exponents depend on the model family,
resource definition, and training prescription
\citep{alabdulmohsin2022revisiting,caballero2023broken,porian2024discrepancies}.
Theoretical accounts derive learning curves from task dimension,
spectral decay, or simplified dynamics
\citep{sharma2020manifold,bordelon2020spectrum,bahri2024explaining};
solvable random-feature and dynamical models separate data-, width-,
and time-limited phases
\citep{maloney2022solvable,bordelon2024dynamical,paquette2024phases}.
A discrete-skill model obtains aggregate power laws from decreasing
skill frequencies \citep{michaud2023quantization}.
Knowledge-capacity studies instead translate factual prediction loss into
retained information and estimate bits per parameter
\citep{allenzhu2025capacity,morris2025memorize}. A compression model with
frequency-ranked knowledge derives both data and model scaling laws
\citep{pan2025compression}, while an optimization analysis of multilayer
Transformers derives resource-dependent generalization phases
\citep{yang2025scaling}. We connect the data and learned-memory axes through
a common source spectrum. Its tail determines both rates, and its construction
identifies the energy--dimension pairing on which those rates depend.

\input{arxiv/prior_comparison}

\paragraph{Finite-rate estimation and memory.}
Rate--distortion theory characterizes the smallest coding rate compatible
with a prescribed distortion, and classical quantization theory develops
the geometry of finite codebooks
\citep{shannon1959fidelity,gray1998quantization}.
Quantized Gaussian sequence estimation gives sharp arbitrary-codebook
storage--risk trade-offs, smoothness adaptation over Sobolev classes,
and multiple data--communication regimes
\citep{zhu2014quantized,zhu2018sobolev,zhu2018distributed}.
Related theories bound working memory, sequential estimator states,
retained sample information, or Bayesian minimum excess risk
\citep{steinhardt2015memory,berg2021finitememory,
feldman2025memorization,hafezkolahi2021ratedistortion}.
Writing \(B\) for retained bits, we constrain the final learned state to at most \(2^B\) values while
allowing unrestricted training computation.
Functional quantization is especially close to our spectrum calculation:
regularly varying covariance eigenvalues determine sharp Hilbert-space
quantization rates, and product quantizers attain them
\citep{luschgy2004sharp,luschgy2010optimal}.
Our additional step is a uniform finite-resource risk bound for random
coordinate observations: each observation queries one unknown binary
prediction parameter, and all logits vary independently within bounded intervals. It accounts
for coordinates with few or no visits while allowing any final \(2^B\)-state
encoder. The same curve then expresses sample and retained-state error.
The same-marginal construction resolves an
identification question: separate block-energy and block-dimension
distributions do not determine the coordinate spectrum.

\paragraph{Neural quantization and precision.}
Mixed-precision training retains high precision where it is numerically
useful \citep{micikevicius2018mixed}. Mean-field, Hessian-aware, and
learned-precision methods quantify depth, sensitivity, and heterogeneous
allocation effects
\citep{blumenfeld2019meanfield,dong2020hawqv2,savarese2022heterogeneous}.
Per-tensor analyses and integer programs contain the
dynamic-range-squared times \(4^{-b}\) distortion law and
resource-constrained bit allocation
\citep{chen2020lowbit,yao2021hawqv3}.
Practical large language model (LLM) quantization uses outlier handling,
second-order rounding, activation rescaling, and data-dependent
discrepancy control
\citep{dettmers2022llmint8,frantar2023gptq,xiao2023smoothquant,
chee2025discquant}.
Precision-aware scaling studies measure how these effects change with
model size
\citep{dettmers2023fourbit,kumar2025precision,sun2025floating,
zhang2026precision}; kernel-regime analysis establishes width-dependent
convergence for 1-bit networks \citep{daliri2025onebit}.
Our finite-precision analysis composes local numerical distortion
with finite-sample uncertainty and the learned-state budget.
For the specified computational graph, it also quantifies the
arithmetic cost associated with the required precision.

\paragraph{Minimax generalization and attention optimization.}
Worst-case neural-network theory controls function classes through
Vapnik--Chervonenkis dimension, pseudo-dimension, norms, margins,
and covering numbers
\citep{bartlett2019vc,bartlett2017spectral}.
Compression and description-length analyses relate concise predictors
to generalization, including probably approximately correct
(PAC)--Bayes bounds
\citep{arora2018compression,daniely2019description,lotfi2022pacbayes}.
Distribution-specific minimax theory ties rates to compositional
smoothness, Besov regularity, or intrinsic dimension
\citep{schmidthieber2020nonparametric,suzuki2019adaptivity,
nakada2020intrinsic}.
Our statistical result follows this distribution-specific approach,
with held-out autoregressive cross-entropy and an arbitrary finite
learned state.

For sequence models, norm-based Transformer bounds, analyses of Markov
next-token prediction, dependent-token pretraining, and nearly tight
sample complexity address generalization under architecture or dependence
\citep{trauger2024sequence,yuksel2025sample,li2025generalization,
yang2026transformer}.
Attention optimization studies max-margin token selection, retrieval
learning, and frequency-dependent associative memory
\citep{tarzanagh2023maxmargin,li2024mechanics,
cabannes2024associative}.
Structure-aware preconditioning gives a data--compute law for Gaussian
softmax regression under an expected empirical-gradient oracle
\citep{goel2026softmax}; other softmax-regression objectives admit efficient
gradient methods with structured output constraints \citep{chu2024copyright}.
Theory for prefix tuning and next-step forecasting identifies how attention
training and its generalization depend on the learned representation
\citep{liang2025prefix,ke2025curse}. Complementary work asks which token state
can be removed while preserving attention or long-context prediction:
sparse approximation controls attention-matrix error, while parallel and
value-aware compressors allocate a limited context state
\citep{deng2024sparse,xiong2025parallelcomp,yang2026compressible}.
We count the samples used to learn the relevant candidate and its
prediction parameter, then the bits used to retain the fitted state.
Appendix~\ref{app:discussion} gives the detailed
comparison of the statistical, representational, and computational settings.

\FloatBarrier

%% file: arxiv/prior_comparison.tex
\input{sections/joint_prior_comparison}

%% file: sections/joint_prior_comparison.tex
\begin{table}[tbp]
\centering
\caption{\textbf{Closest results and the resource question they answer.}
Here \(m\) is Gaussian sequence dimension, \(N_c\) is codebook size, and
\(b\) is bits per Gaussian coordinate, \(B\) is total retained bits, and
\(n\) is the number of training blocks in our source. CE denotes cross-entropy. The last row identifies
the observation model and finite-resource conclusion needed here.}
\label{tab:prior-comparison}
\small
\setlength{\tabcolsep}{4pt}
\begin{tabular}{@{}>{\raggedright\arraybackslash}p{.16\linewidth}
>{\raggedright\arraybackslash}p{.27\linewidth}
>{\raggedright\arraybackslash}p{.25\linewidth}
>{\raggedright\arraybackslash}p{\dimexpr.32\linewidth-6\tabcolsep\relax}@{}}
\toprule
Study & Observations and parameters & Code and risk & Conclusion \\
\midrule
\citet{zhu2014quantized} & Gaussian sequence; deterministic means in a
Euclidean ball & Arbitrary codebook, \(m b\) bits; normalized squared error &
Sharp asymptotic estimation--storage trade-off as \(m\to\infty\). \\
\citet{luschgy2004sharp} & Gaussian random element; covariance eigenvalues &
\(N_c\)-point quantizer; expected squared Hilbert error &
High-resolution distortion from eigenvalue decay; \(B=\log_2 N_c\). \\
\citet{allenzhu2025capacity} & Random discrete factual tuples; trained language
model weights & Finite weight state; factual generation negative log-likelihood &
Achieved loss lower-bounds the information that the weights must retain. \\
\citet{pan2025compression} & Hierarchical syntax--knowledge mixture;
frequency-ranked knowledge & Bayesian/universal code; predictive redundancy
under data or model capacity & Power-law knowledge frequencies generate data
and model scaling laws. \\
\citet{goel2026softmax} & Gaussian linear regression; softmax attention &
Expected empirical-gradient oracle; squared loss; no bit constraint &
Geometric optimization with a finite-context statistical term. \\
This work & Random coordinate observations; independently bounded scalar logits &
Any \(2^B\)-state learner; held-out excess Bernoulli CE &
Uniform \((n,B)\) risk through one spectrum; pairing changes both exponents. \\
\bottomrule
\end{tabular}
\end{table}

%% file: sections/revision_2_model_iclr9.tex
\section{The Energy-per-Coordinate Spectrum: Minimax Risk and Identifiability}
\label{sec:revision-model}
\vspace{-3mm}
We first define the source and finite-state value-estimation problem;
Section~\ref{sec:revision-learning} additionally learns the relevant route.

\vspace{-3mm}
\subsection*{Problem setup}
\paragraph{Coordinates and unknown knowledge.}
A coordinate is an independent scalar prediction parameter; mode \(j\)
groups \(d_j\) coordinates \(i=(j,r)\), \(r\in[d_j]\), where
\([m]=\{1,\ldots,m\}\). Separating energy from parameter count,
the finite/countable set \(\mathcal I\) has public frequencies
\(p_i>0\), \(\sum_i p_i=1\), and radii \(0<R_i\le R<\infty\).
Unknown logits
\(\theta\in\Theta:=\prod_{i\in\mathcal I}[-R_i,R_i]\) specify binary
target probabilities.

\vspace{-3mm}
\paragraph{One causal prediction block.}
Sample candidates \(C_1,\ldots,C_K\) independently from \(p\), then an
independent query row \(Q\), uniform on \([K]\), followed by target \(Y\).
A fixed unknown matching map \(\pi:[K]\to[K]\) selects
\(I=C_{\pi(Q)}\), with
\(Y\mid I\sim\operatorname{Bernoulli}(\sigma(\theta_I))\).
Here \(\sigma(t)=(1+e^{-t})^{-1}\) is the sigmoid, so \(\theta_i\)
is the log-odds at \(I=i\), and \(\Pr(I=i)=p_i\).
Delimiters separate independent, causally ordered blocks; \(n\) counts
blocks, not tokens. Calibration reveals \((Q,\pi(Q))\), prediction
targets teach logits, and testing uses independent blocks. Revealed
routes/coordinates give value data \(\mathcal D_n=((I_t,Y_t))_{t=1}^n\),
separating estimation from retrieval.

\vspace{-3mm}
\paragraph{Loss and the retained-state constraint.}
For coordinate predictions \(f=(f_i)\), define
\(\ell(y,z)=-y\log\sigma(z)-(1-y)\log(1-\sigma(z))\) and
\(\mathcal L_\theta(f)=\mathbb E_\theta[\ell(Y,f_I)]\).
Loss is in nats, storage in bits. Bounded logits give positive conditional
entropy, Bayes loss \(\mathcal L_\infty=\mathcal L_\theta(\theta)>0\), and
excess loss comparable to frequency-weighted squared logit error, with
\(R\)-dependent constants.
A learner encodes \(\mathcal D_n\) into
\(S=\phi(\mathcal D_n)\in\mathcal S\), \(|\mathcal S|\le2^B\),
and predicts with a fixed decoder \(g(S,i)\in[-R,R]\).
Arbitrary computation and joint coordinate coding are allowed, with all
data-dependent output in \(S\). For integer \(B\ge0\), minimax excess risk is
\[
\mathfrak R^*_{\rm value}(n,B)
=\inf_{\phi,g:\,|\mathcal S|\le2^B}\;
\sup_{\theta\in\Theta}
\mathbb E_\theta\!\left[
\mathcal L_\theta\bigl(g(\phi(\mathcal D_n),\cdot)\bigr)
-\mathcal L_\theta(\theta)\right].
\]
Expectations average training/learner and independent test randomness.
Encoder/decoder side information is fixed and public; learned bits exclude
fixed network storage. Constants in \(\asymp_R\) depend only on \(R\).

\paragraph{Predictive energy and its spectrum.}
Define \(q_i=p_iR_i^2\), with \(\sum_iq_i<\infty\).
Doubling frequency doubles energy; halving logit radius quarters it.
Equal energies can have different frequencies and ranges. With \(\delta_x\)
a unit point mass,
\[
\mu=\sum_i\delta_{q_i},\quad
\mathcal N_\mu(t)=\mu((t,\infty)),\quad
\Phi_\mu(t)=\int\min\{x,t\}\,\mu(\mathrm dx)
=\int_0^t\mathcal N_\mu(s)\,\mathrm ds.
\]
Thus \(\mathcal N_\mu(t)\) counts energies above \(t\), including repeats;
\(\Phi_\mu(t)=\sum_i\min\{q_i,t\}\) caps unresolved contributions.
For water level \(\tau>0\), let
\(\mathcal B_\mu(\tau)=\frac12\int\log_2(x/\tau)_+\mu(\mathrm dx)\)
and define \(\tau_B\) as the generalized solution of
\(\mathcal B_\mu(\tau_B)=B\), with \(\tau_0=\max_iq_i\).
Here \((x)_+=\max\{x,0\}\) applies after the logarithm: coordinates
above \(\tau\) receive \(\frac12\log_2(q_i/\tau)\) bits, others zero.
\(\mathcal B_\mu\) counts continuous bits needed to reach \(\tau\);
finitely many energies exceed each positive level, even in countable sets.

\paragraph{Energy groups and scaling profiles.}
An energy block is a coordinate group, distinct from a sampled token block.
For total energy \(a_j\), dimension \(d_j\), and equal-energy coordinates,
\(q_{jr}=a_j/d_j\) and \(\mu=\sum_jd_j\delta_{a_j/d_j}\): dimension
sets both atom count and size. Profiles \(a_j\asymp j^{-(1+\chi)}\),
\(d_j\asymp j^\gamma\), with \(\chi>0\), \(\gamma\ge0\), describe
energy decay and dimension growth; the theorem itself needs no power law.

%% file: arxiv/sections/revision_3_unified.tex
\subsection{One Spectrum Controls Data and Learned State}
\label{sec:revision-unified}

\begin{restatable}[One spectrum controls data and memory; informal version of \cref{full:thm:spectrum}]{theorem}{ThmSpectrum}
\label{thm:spectrum}
Let \(q_i>0\), \(\sum_iq_i<\infty\), and let \(\mu\), \(\Phi_\mu\),
and \(\tau_B\) be defined above.  Then
\begin{align}
\mathcal S_q(n)&:=\sum_i\min\{q_i,n^{-1}\}=\Phi_\mu(n^{-1}),
\label{eq:spectrum-data}\\
D_q(B)&:=\inf_{\substack{b_i\ge0\\\sum_i b_i\le B}}
\sum_iq_i2^{-2b_i}=\Phi_\mu(\tau_B).
\label{eq:spectrum-bits}
\end{align}
For the bounded-logit value experiment,
\begin{equation}
\boxed{\mathfrak R^*_{\rm value}(n,B)
\asymp_R\Phi_\mu(n^{-1})+\Phi_\mu(\tau_B).}
\label{eq:spectrum-risk}
\end{equation}
Moreover, for every \(t>0\),
\begin{equation}
\Phi'_{\mu,+}(t)=\mu((t,\infty)),\qquad
\Phi'_{\mu,-}(t)=\mu([t,\infty)).
\label{eq:spectrum-derivatives}
\end{equation}
Hence the exact curve \(\Phi_\mu\) recovers the positive spectrum,
including atom multiplicities, up to coordinate relabeling.
\end{restatable}

The statistical term limits what observations resolve; the representation
term limits what any final state can preserve. The latter is a continuous
allocation problem; the finite-state minimax comparison accounts for the
rounding needed to implement an integer-bit code.

Binary testing lower-bounds the sample error \(\min\{q_i,1/n\}\);
an entropy argument lower-bounds representation error even for joint codes.
Estimation followed by quantization attains their sum up to constants.

Geometrically, \(\mathcal N_\mu(t)\) counts coordinates whose energy exceeds
resolution \(t\), while \(\Phi_\mu(t)\) accumulates the unresolved energy.
Data read this curve at \(t=1/n\).  Storage pays
\(\frac12\log_2(q_i/\tau_B)\) bits for every coordinate above its water
level and reads the same curve at \(t=\tau_B\). For one coordinate, the
two spectral terms are \(\min\{q,1/n\}\) and \(q2^{-2B}\).
Water filling equalizes remaining error on coordinates receiving bits.
The derivatives recover every positive energy and its multiplicity.

\begin{restatable}[Spectral scaling exponents; informal version of \cref{full:cor:spectrum-rates}]{corollary}{CorSpectrumRates}
\label{cor:spectrum-rates}
If \(\mathcal N_\mu(t)\sim t^{-\rho}L(1/t)\) for \(0<\rho<1\), where
\(L\) is slowly varying, then
\[
\Phi_\mu(n^{-1})\in\operatorname{RV}_{-(1-\rho)},\qquad
\Phi_\mu(\tau_B)\in\operatorname{RV}_{-(1-\rho)/\rho}.
\]
For \(a_j\asymp j^{-(1+\chi)}\) and \(d_j\asymp j^\gamma\),
\(\rho=(1+\gamma)/(1+\chi+\gamma)\), giving rates
\(n^{-\chi/(1+\chi+\gamma)}\) and
\(B^{-\chi/(1+\gamma)}\), up to slowly varying factors.
\end{restatable}

Here \(\operatorname{RV}_\alpha\) denotes regular variation with index
\(\alpha\): a function \(g\) satisfies \(g(cx)/g(x)\to c^\alpha\)
as \(x\to\infty\), for every \(c>0\). Thus the density of small spectral
atoms determines both exponents and relates them through one tail index.

%% file: arxiv/sections/revision_4_coupling.tex
\subsection{Why Separate Marginals Cannot Identify Scaling}
\label{sec:revision-coupling}

A fixed amount of predictive energy is harder to learn and retain when it
is spread across more independent coordinates. Imagine two lists: the total
energy assigned to each group and the number of independent coordinates in
each group. A marginal descriptor retains these lists but forgets which
energy goes with which dimension. We compare two assignments of the same
lists: concentrate the high energies in one-coordinate groups, or spread
them across the larger groups. The next theorem makes this comparison exact
at every dyadic level, a level indexed by an integer \(\ell\) with group
counts proportional to \(2^\ell\).

\begin{restatable}[Same block marginals, different scaling laws; informal version of \cref{full:thm:dyadic}]{theorem}{ThmDyadic}
\label{thm:dyadic}
For every \(\chi,\gamma>0\), there are two positive-conditional-entropy
causal-retrieval sources with exactly the same block-energy multiset and the
same block-dimension multiset at every dyadic level.  They differ only in the
pairing.  All displayed dimensions in this construction are rounded upward
to integers.  With the common fixed-\(K\) route revealed,
\[
\mathfrak R_{\rm easy}(n,B)
\asymp n^{-\chi/(1+\chi)}+B^{-\chi},
\]
whereas
\[
\mathfrak R_{\rm hard}(n,B)
\asymp n^{-\chi/(1+\chi+\gamma)}+B^{-\chi/(1+\gamma)}.
\]
\end{restatable}

At dyadic level \(\ell\), both sources have \(2^\ell\) high- and
low-energy blocks with energies proportional to
\(2^{-(1+\chi)\ell}\) and
\(2^{-(1+\chi+\delta)\ell}\), and \(2^\ell\) copies each of dimensions
\(1\) and \(2^{\gamma\ell}\).  Choose \(\delta>\gamma\chi\).  The easy
source pairs high energy with dimension one; the hard source pairs it with
dimension \(2^{\gamma\ell}\).  Thus both block marginals agree exactly, not
merely asymptotically.

A component with block-energy exponent \(s\) and dimension exponent \(g\)
has spectral tail index \((1+g)/(1+s+g)\).  The high-energy component
therefore has index \(1/(1+\chi)\) in the easy source and
\((1+\gamma)/(1+\chi+\gamma)\) in the hard source; the condition on
\(\delta\) makes the low-energy component lower order.  Applying
\cref{cor:spectrum-rates} yields all four exponents.  The full construction
and converse are in Appendix~\ref{app:joint-summary}.

Pairing changes the coordinate-energy distribution while preserving both
block marginals. This is the information a scaling descriptor must retain;
\cref{eq:spectrum-derivatives} recovers it from the integrated tail.

%% file: arxiv/sections/revision_5_learning.tex
\section{Learning the Causal Self-Attention Predictor}
\label{sec:revision-learning}

The value theorem reveals the relevant coordinate. We now learn how to
select it from the candidate tokens, then learn its prediction parameter.
Training has two stages: route calibration (Stage A) followed by value
learning (Stage V). Both use exact-arithmetic stochastic gradient descent
(SGD); their sample and state costs are accounted for separately.

\paragraph{From a route table to query--key attention.}
Let \(x_q,x_s\in\mathbb R^d\) be the query token and candidate token at
position \(s\). Their fixed embeddings include orthogonal address blocks:
\(P_Q\in\mathbb R^{K\times d}\) extracts the query row and
\(P_C\in\mathbb R^{K\times d}\) extracts the candidate position, so
\(P_Qx_q=e_q\) and \(P_Cx_s=e_s\), where \(e_s\) is the \(s\)-th
standard basis vector. Write \(d_h\ge K\) for the attention-head dimension.
An isometric embedding \(J\in\mathbb R^{d_h\times K}\), satisfying
\(J^\top J=I_K\), places these addresses in the head. The trainable matrix
\(A\in\mathbb R^{K\times K}\) stores one score for each query--position
pair. Define the fixed key map and trainable query map by
\begin{equation}
W_K=\sqrt{d_h}JP_C,\qquad W_Q(A)=JA^\top P_Q,\qquad
\frac{\langle W_Q(A)x_q,W_Kx_s\rangle}{\sqrt{d_h}}=A_{qs}.
\label{eq:qk-equivalence}
\end{equation}

\begin{restatable}[Exact query--key realization; informal version of \cref{full:prop:qk}]{proposition}{PropQK}
\label{prop:qk}
With a strict causal mask over the \(K\) preceding candidates,
\cref{eq:qk-equivalence} exactly realizes the route scores and attention.
Fix \(W_K\) and train only \(W_Q\)'s query-address block. Euclidean SGD then
gives the same \(A\) update, attended feature, and output-vector gradient.
\end{restatable}

The fixed embedding, key projection, and value extraction contain neither the
unknown route nor the unknown logits.  Those quantities enter only the
trained query-address block and output coordinates.

\paragraph{Learning the route.}
For a fixed query row, calibration increases the matching score relative
to its distractors. Zero initialization keeps the distractor scores equal.
Their common gap from the matching score is the margin \(\beta_t\) after
\(t\) visits to that row; \(\delta_t\) is the total attention probability
on incorrect candidates. The following result quantifies how calibration
reduces this leakage.

\begin{restatable}[Online routing dynamics; informal version of \cref{full:prop:online-routing}]{proposition}{PropRouting}
\label{prop:online-routing}
For \(K\ge2\), zero initialization and softmax cross-entropy steps of size
\(0<\eta\le1\) give the margin and distractor mass after \(t\) row visits:
\[
\beta_{t+1}=\beta_t+\eta\frac{K}{e^{\beta_t}+K-1},\qquad
\delta_t=\frac{K-1}{e^{\beta_t}+K-1}.
\]
After \(b_K=\Theta_\eta(\log(eK))\) visits,
\(\delta_t\asymp_\eta[1+(t-b_K)_+]^{-1}\).  For \(S_A\) independent
calibration blocks and a held-out row,
\[
\mathfrak r_K(S_A)=\E[\delta_N^2],\quad
N\sim\operatorname{Binomial}(S_A,1/K).
\]
For \(0<\varepsilon\le1/16\), the time to \(\mathfrak r_K\le\varepsilon\)
is \(\Theta_\eta(K\{\varepsilon^{-1/2}+\log(eK)\})\);
at \(K=1\), leakage and calibration cost vanish.
\end{restatable}

\paragraph{Learning and storing the values.}
After calibration, we freeze the route and fit an output vector on
\(A_n=\{i:q_i\ge n^{-1}\}\). This active set contains coordinates whose
energy is large enough to resolve at sample size \(n\). The natural box is
\(\mathcal W=\prod_{i\in A_n}[-R_i,R_i]\); output-vector entries
outside \(A_n\) are fixed at zero. The attended feature \(h_t\) contains the
attention-weighted coordinate indicators restricted to \(A_n\).
Let \(\mu_{\rm sgd}=\frac9{16}\min_{|z|\le R}\sigma'(z)>0\) be the
step-size constant used in the formal theorem. Starting from
\(w_1=0\), use \(\eta_t=2/[\mu_{\rm sgd}(t+1)]\) and return
\(\bar w_n=2\sum_{t=1}^n t w_t/[n(n+1)]\). If \(A_n\) is empty,
return zero. Projection \(\Pi_{\mathcal W}^D\) uses the frequency-weighted
norm \(\|v\|_D^2=\sum_{i\in A_n}p_i v_i^2\). On the active coordinates,
\(D\) is restricted to \(A_n\), and the update is
\begin{equation}
w_{t+1}=\Pi_{\mathcal W}^{D}\!\left[
w_t-\eta_tD^{-1}\{\sigma(w_t^\top h_t)-Y_t\}h_t\right],
\qquad D=\operatorname{diag}(p).
\label{eq:actual-output-sgd}
\end{equation}

The final vector is encoded in coordinatewise scalar formats using
\(B_V\) value-state bits. The format-range ratio \(\kappa_i\) compares
its stored range with the logit radius \(R_i\); unit ratios recover the
representation problem of Section~\ref{sec:revision-model}. We write
\(\widehat F\) for the resulting causal predictor and
\(\mathcal R_\theta(\widehat F)\) for its population cross-entropy on a
complete prediction block. The term \(\mathcal E_{\rm arithmetic}\)
accounts for deployment arithmetic error under the formal theorem's
numerical assumptions.

\begin{restatable}[Causal self-attention realization; informal version of \cref{full:thm:causal-realization}]{theorem}{ThmCausalRealization}
\label{thm:causal-realization}
The weighted-average SGD with scalar formats in
\cref{full:thm:causal-realization} satisfies
\begin{equation}
\E[\mathcal R_\theta(\widehat F)-\Linf]
\lesssim_{R,K,\eta}
\Phi_\mu(n^{-1})+D^{\rm fmt}_{q,\kappa}(B_V)
+\mathfrak r_K(S_A)+\mathcal E_{\rm arithmetic},
\label{eq:causal-realization}
\end{equation}
where \(B_V\) is the value-state budget and \(\kappa_i\ge1\) the public
format-range ratio, with
\[
D^{\rm fmt}_{q,\kappa}(B_V)
=\inf_{b_i\ge0,\,\sum_i b_i\le B_V}
\sum_iq_i\min\{1,\kappa_i^22^{-2b_i}\}.
\]
For \(\kappa_i=1\), this is \(\Phi_\mu(\tau_{B_V})\), and the core
surface holds at total budget \(B\) for power-law spectra, fixed \(K\),
\(o(B)\) routing bits and lower-order routing/arithmetic errors.
\end{restatable}

Conditioned on the route, output risk is strongly convex on fitted
coordinates in the frequency-weighted norm; preconditioning balances their learning rates.
Appendix~\ref{app:secondary-results} gives gradient, state, and format details.
\paragraph{The total resources.}
Write \(B_A\) for the bits encoding the route, so total learned state is
\(B=B_V+B_A\), and total training blocks are \(N=n+S_A\).
Let \(q_{(r)}\) list coordinate energies in decreasing order, with decay
exponent \(\xi>0\). Charging both stages preserves the leading resource
exponents: for \(q_{(r)}\asymp r^{-(1+\xi)}\), fixed \(K\), unit midpoint formats and
the deployment-error condition in \cref{prop:total-resources}, risk
\(\varepsilon\) is achieved with
\(N=n+S_A=O(\varepsilon^{-(1+\xi)/\xi})\) total blocks and
\(B=B_V+B_A=O(\varepsilon^{-1/\xi})\) total learned bits.
Calibration blocks and route-state bits are lower-order terms.

%% file: arxiv/sections/revision_5b_branches.tex
\section{Adaptive Representation and Finite-Precision Realization}
\label{sec:revision-branches}

Two representation questions remain after the resource law. How should
bits be allocated when the energy decay is unknown? How accurately does a
finite-precision attention network evaluate the retained parameters?
The first result constructs an estimator and a storage schedule; the second
maps supplied parameters to prediction error after numerical evaluation.
Here scalar coordinates are indexed by \(j\), and \(\Delta_j\) denotes
the radius written \(R_i\) in Section~\ref{sec:revision-model}.
The loss \(\mathcal L_\theta(f)\) is the same population cross-entropy.

\paragraph{Adaptive representation.}
\label{sec:adaptive}
\begin{restatable}[Adaptation with a known ordering; informal version of \cref{full:thm:adaptive}]{theorem}{ThmAdaptive}
\label{thm:adaptive}
Fix \(0<\chi_{\min}\le\chi_{\max}<\infty\), and suppose the canonical mode
order is known.  Assume
\(q_j\asymp j^{-(1+\chi)}\)
with common comparison constants and unknown
\(\chi\in[\chi_{\min},\chi_{\max}]\). For every \(n,B\in\N\), one clipped
likelihood estimator and one deterministic integer bit schedule, both
independent of \(\chi\), satisfy simultaneously throughout this interval
\begin{equation}
  \sup_{\theta_j\in[-\Delta_j,\Delta_j]}
  \E_\theta\!\left[
    \mathcal L_\theta(\widehat f)-\mathcal L_\theta(\theta)
  \right]
  \lesssim
  n^{-\chi/(1+\chi)}+B^{-\chi}.
  \label{eq:adaptive-rate}
\end{equation}
\end{restatable}

Choose \(A>\chi_{\max}/2\) and the largest prefix \(m\) satisfying
\begin{equation}
b_j=\floor*{A\log_2(m/j)}_+\quad(j\le m),\qquad
\sum_{j\le m}b_j\le B.
\label{eq:universal-schedule}
\end{equation}
Then \(m\asymp B\), and both quantization and omitted-tail errors are
\(O(m^{-\chi})\) uniformly over the exponent interval. The logarithmic
schedule spends most bits on the leading coordinates while matching the
tail error for every allowed decay rate.

\paragraph{Finite-precision self-attention.}
\label{sec:transformer}
We construct a network that evaluates a supplied logit table for the first
\(M\) coordinates and uses zero for the rest. Here \(M\) is the retained
prefix length, \(t_\star\) is the matching candidate position, and
\(\beta\ge0\) controls the matching attention score. The candidate's mode
label \(j_t\) identifies its scalar table entry. Use public basis addresses \(a_t\in\R^K\), mode vectors
\(e_j^{(M)}=e_j\) for \(j\le M\) and zero otherwise, and
\(d=K+M+1\).  Candidate and query tokens are
\(x_t=(a_t,e_{j_t}^{(M)},0)\) and \(x_q=(a_{t_\star},0,1)\).
The fixed-key construction in \cref{prop:qk} gives
\begin{equation}
\frac{\langle W_Qx_q,W_Kx_t\rangle}{\sqrt{d_h}}
=\beta\inner{a_{t_\star}}{a_t}.
\label{eq:attention-score}
\end{equation}
With the value projection onto mode coordinates, a strict causal mask, and
an output vector \(w\) supported only on those coordinates,
\begin{equation}
h_q=\sum_{t=1}^K\alpha_t\bar e_{j_t}^{(M)},\qquad
\alpha_t=\frac{\exp(\beta\ind\{t=t_\star\})}{e^\beta+K-1},\qquad
z=w^\top h_q,
\label{eq:attention-map}
\end{equation}
where \(\bar e_j^{(M)}=(0,e_j^{(M)},0)\).  Write
\(w_j^{\rm mode}=w_{K+j}\), and set \(\vartheta_j=w_j^{\rm mode}\)
for \(j\le M\) and \(\vartheta_j=0\) otherwise.

The numerical calculation below uses unit roundoff \(u\), an upper bound
on the relative error of the listed scalar operations. Its four error terms
respectively measure selecting a distractor, arithmetic rounding, storing
the retained logits, and predicting omitted coordinates with zero.

\begin{restatable}[Finite-precision realization; informal version of \cref{full:thm:transformer}]{theorem}{ThmTransformer}
\label{thm:transformer}
Fix \(K,M\ge1\), \(R<\infty\), \(\beta\ge0\), and the construction above.

\begin{enumerate}
  \item In exact arithmetic, if \(\abs{w_j^{\rm mode}}\le R\) for every
  \(j\in[M]\), then a query whose matching candidate has any mode \(j\ge1\)
  satisfies
  \begin{equation}
    \abs{z-\vartheta_j}\le2R(K-1)e^{-\beta}.
    \label{eq:routing-bound}
  \end{equation}

  \item Set \(w_j^{\rm mode}=\theta_j\) for \(j\in[M]\) and store it with
  \(b_j\in\N_0\) bits on
  \([-\Delta_j,\Delta_j]\).  Assume
  stable max-shifted softmax with exactly representable shifted scores \(0,-\beta\)
  and an exponent range containing their exponentials. Exponentials,
  multiplications, and divisions have relative error at most \(u\);
  sums use balanced pairwise accumulation. For the resulting predictor
  \(f_{\rm fp}\), if
  \((\ceil{\log_2K}+4)u\le1/4\), then the held-out excess cross-entropy obeys
  \begin{equation}
    \mathcal L_\theta(f_{\rm fp})-\Linf
    \le
    C_R\left\{
      (K-1)^2e^{-2\beta}
      +u^2(1+\log K)^2
      +\sum_{j=1}^M q_j2^{-2b_j}
      +\sum_{j>M}q_j
    \right\}.
    \label{eq:fp-bound}
  \end{equation}

  \item Within the separable scalar-table representation in which coordinate
  \(j\) has at most \(2^{b_j}\) values, the bit-dependent term in
  \cref{eq:fp-bound} is tight up to constants.
\end{enumerate}
Here and below, \(C_R\) depends only on the logit bound \(R\).
\end{restatable}

The bound separates routing, arithmetic, quantization, and omitted energy.
Pairwise summation gives logarithmic dependence on \(K\) in arithmetic
error before squaring. This result evaluates supplied logits;
\cref{thm:causal-realization} controls their learning. Dense storage
accounting and proofs appear in
Appendices~\ref{app:secondary-results}, \ref{app:adaptive}, and
\ref{app:transformer}.

%% file: arxiv/sections/revision_5c_compute.tex
\section{Bit Complexity and Compute--Precision Scaling}
\label{sec:revision-compute}
\label{sec:compute}
\label{sec:compute-exp}

\begin{figure}[htb]
\centering
\includegraphics[width=\linewidth]{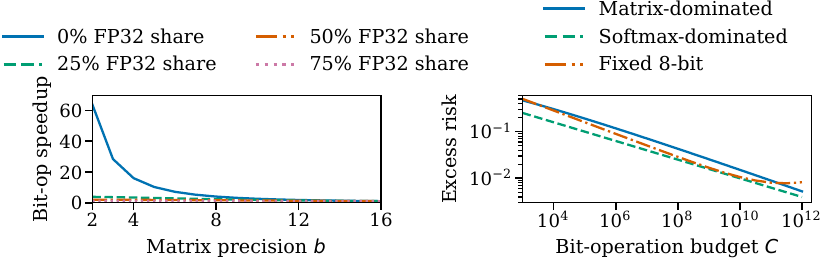}

\caption{\textbf{Retained precision sets resource ceilings and error floors.}
Left: analytic bit-operation ratio with retained high-precision work.
Right: risk--cost laws under the stated arithmetic model; fixed softmax
precision creates a numerical floor.}
\label{fig:compute-summary}

\end{figure}

Representation accuracy is useful only if the chosen arithmetic can
preserve it. This section separates the precision needed as the represented
prefix grows, the work of a specified dense graph, and the savings from
lowering precision at a fixed shape.

\paragraph{Error and cost as model size grows.}
Here \(M\) is the number of represented scalar modes, \(b\) is arithmetic
precision, and \(\chi>0\) is the model-error decay exponent. The operation
exponent \(f\ge0\) describes per-block work; \(\nu\ge0\) describes how
scalar-operation cost grows with bit width. The function \(\kappa(M)\)
amplifies stored perturbations into prediction error. Suppose representing
\(M\) modes has error \(\Theta(M^{-\chi})\), each processed block costs
\(\Theta(M^f)\) scalar operations, \(\Theta(M^{1+\chi})\) blocks are
necessary and sufficient, a \(b\)-bit operation costs \(\Theta(b^\nu)\),
and stored perturbation \(2^{-b}\) produces error
\(\Theta(\kappa(M)^22^{-2b})\). Under these two-sided assumptions,
\(\mathcal R(C)\) is the best held-out risk at arithmetic
budget \(C\).

\paragraph{Counting a dense forward--backward pass.}
To relate these abstract costs to an implementation, fix a conventional
Transformer graph. Its matrix arithmetic and softmax arithmetic can use
different precisions, denoted \(b_{\rm mm}\) and \(b_{\rm sm}\).
For a dense Transformer with sequence length \(\ell\), width \(d\), feed-forward width
\(m_{\rm ff}\), \(H\) heads, and \(L\) layers, its conventional
forward--backward arithmetic is, up to constants,
\[
L\!\left[
(\ell d^2+\ell d m_{\rm ff}+\ell^2d)\mathsf M(b_{\rm mm})
+H\ell^2\mathsf S(b_{\rm sm})
\right].
\]
Here \(\mathsf M(b_{\rm mm})\) and \(\mathsf S(b_{\rm sm})\) are the
bit-operation costs of a matrix multiply--accumulate and a softmax scalar
operation at their respective precisions. This forward--backward count
excludes optimizer updates. Recomputation reduces saved attention state
while preserving the dense arithmetic order.

\begin{restatable}[Precision-aware compute phases; informal version of \cref{full:thm:compute}]{theorem}{ThmCompute}
\label{thm:compute}
For the two-sided risk--cost family of \cref{full:thm:compute}, with
\(\chi>0\), \(f,\nu\ge0\), and \(n(M)\asymp M^{1+\chi}\):
\begin{enumerate}
\item If \(1\lesssim\kappa(M)\lesssim M^\rho\), then
the least model-error-matching precision is \(b=\Theta(\log M)\), and
\begin{equation}
\mathcal R(C)-\Linf\asymp
C^{-\chi/(f+1+\chi)}
(\log C)^{\chi\nu/(f+1+\chi)}
\label{eq:poly-conditioning}
\end{equation}
up to other slowly varying factors specified in the assumptions.
\item If \(\kappa(M)=\exp(\Theta(M^\lambda))\), \(\lambda>0\), then
\(b=\Theta(M^\lambda)\), and the loss exponent is
\(\chi/(f+1+\chi+\lambda\nu)\).
\item If \(q_j\asymp j^{-(1+\chi)}\), adaptive per-mode precision has
error \(\Theta(B^{-\chi})\), whereas a uniform-width prefix has optimal
error
\begin{equation}
\Theta((B/\log B)^{-\chi}).
\label{eq:uniform-penalty}
\end{equation}
\end{enumerate}
Each converse uses the representation, operation set, storage layout, and
bit-operation cost specified above.
\end{restatable}

The polynomial-conditioning exponent \(\rho\) in this theorem describes
\(\kappa(M)\), independently of the spectral tail index used in
Section~\ref{sec:revision-model}. Matching numerical error to \(M^{-\chi}\) requires
\begin{equation}
b\ge\log_2\kappa(M)+\frac{\chi}{2}\log_2M-O(1).
\label{eq:precision-requirement}
\end{equation}

\paragraph{Savings at a fixed shape.}
We now hold model dimensions fixed and change only the numerical formats.
This compares two implementations of the same graph. Retained high-precision
work produces an Amdahl-type ceiling: reducing one part's cost leaves the
other part unchanged. At fixed shape,
let \(G,S\) count matrix and softmax operations, \(\nu,\nu_{\rm sm}\)
their precision-cost exponents, \(Q\) the quantized scalar count, and
\(R_{\rm hp}\) the retained high-precision bits. Lowering the matrix and
stored-scalar precision from baseline \(b_0\) to \(b\) gives
\[
\operatorname{BitOpsRatio}
=\frac{G b_0^\nu+S b_{\rm sm}^{\nu_{\rm sm}}}
       {G b^\nu+S b_{\rm sm}^{\nu_{\rm sm}}},
\qquad
\operatorname{Compression}_{\rm bit}
=\frac{Qb_0+R_{\rm hp}}{Qb+R_{\rm hp}}.
\]
In this calculation, \(Q\) counts quantized scalars rather than query rows,
and \(S\) counts softmax operations rather than learned states. Ratios
larger than one denote reduced bit operations or storage.
Thus fixed-precision softmax, master weights, optimizer state, or saved
high-precision activations can dominate even when matrix precision falls.

Appendix~\ref{app:secondary-results} derives the dense forward--backward
count, storage requirements, and precision--risk trade-off.

%% file: arxiv/sections/revision_6_experiments.tex
\section{Experiments}
\label{sec:revision-experiments}

\paragraph{Experimental design.}
We test spectral risk, coupling, and attention learning. Training and held-out blocks are
independent, and comparisons are paired by seed or evaluation item.
The fixed-source synthetic studies evaluate excess cross-entropy (CE), the
loss above the source's Bayes predictor and use 95\% bias-corrected and accelerated (BCa) bootstrap
intervals over paired seeds. Complementary mini-batch and terminal-iterate
studies test robustness to the learning implementation. Appendices~\ref{app:experiment-details}
and~\ref{app:joint-resource-evaluation} specify the algorithms,
evaluation units, and uncertainty procedures.

\begin{figure}[htb]
\centering
\includegraphics[width=\linewidth]{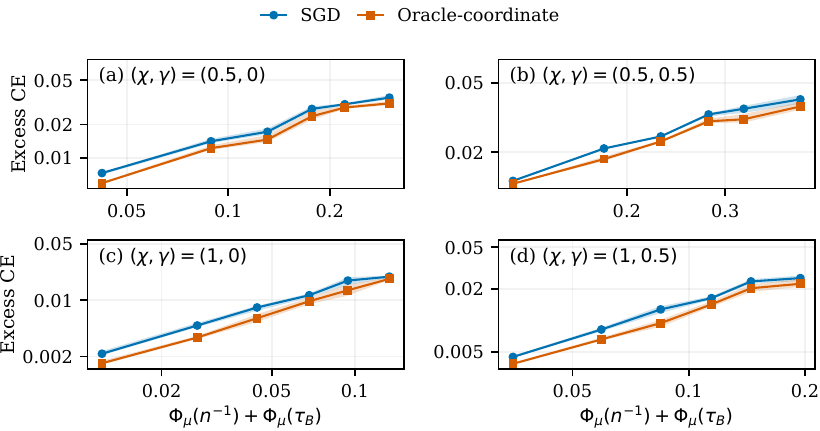}

\caption{\textbf{One spectrum organizes data and memory.}
Four fixed sources, each with 256 energy groups, at \(K=32\): held-out
excess CE versus
\(G_\mu(n,B)=\Phi_\mu(n^{-1})+\Phi_\mu(\tau_B)\).
Blue circles: weighted-average SGD; orange squares: the oracle-coordinate
clipped estimator. Each curve averages 30 budget pairs into six bins chosen
only from \(G_\mu\), within each seed. Bands: 95\% BCa over eight paired seeds.}
\label{fig:spectrum-collapse}

\end{figure}

\paragraph{Data--memory surface and exponent agreement.}
Does the same source statistic organize both resources in the stated learner?
We fix four sources, each with 256 energy groups rather than 256 training
examples, and cross five sample budgets with six value-state
budgets. Each sample budget restarts the theorem's weighted-average SGD in FP64;
all state budgets encode that same fitted vector. Full-source excess CE follows the
resource ordering of \(G_\mu\) for both SGD and the oracle-coordinate estimator
(\cref{fig:spectrum-collapse}). Unbinned risk-to-\(G_\mu\) ratios reveal the
remaining finite-resource variation (\cref{fig:joint-fixed-diagnostics}).
The six-profile exponent grid and candidate-count replication appear
in Appendix~\ref{app:extended-results}, alongside the spectral slopes
(\cref{tab:spectrum-slopes}).

\paragraph{Same-marginal coupling.}
Can identical block-energy and block-dimension marginals produce different
data laws as well as different memory laws? On the same normalized dyadic
source, the fitted easy/hard data exponents are \(0.509/0.339\) for \(\chi=1\)
and \(0.681/0.489\) for \(\chi=2\), against predictions \(1/2,1/3\)
and \(2/3,1/2\), respectively. Fits use the same last four sample budgets for
all 16 seeds. \Cref{fig:mechanisms}(a,d) shows every budget;
panels (b,e) and \cref{tab:coupling-main} retain the complementary storage
separation. Pairing changes both resource exponents, not just their constants.

\begin{figure}[htb]
\centering
\includegraphics[width=\linewidth]{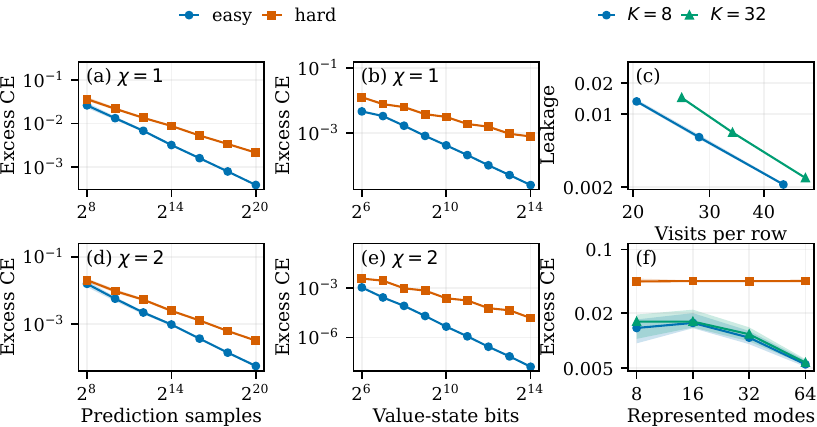}

\caption{\textbf{Pairing separates data and memory; routing enables learning.}
(a,d) Fixed-source data sweeps for \(\chi=1,2\); blue/orange denote easy/hard.
(b,e) Their complementary storage sweeps, evaluated by population excess CE.
(c) Squared leakage versus row visits, \(K=8/32\) (blue/green).
(f) Finite-source excess CE with learned/known/frozen routes
(blue circles/green triangles/orange squares), without an added asymptotic tail.
Intervals: 95\% BCa over 16 data seeds (a,d), 16 source amplitudes (b,e),
or eight paired seeds (c,f).}
\label{fig:mechanisms}

\end{figure}

\paragraph{Learned routing and interventions.}
Routing experiments train routing and output parameters through the actual network forward
map. Its 256 combinations of source profile, represented prefix, candidate
count, and seed separately
restart each learner at three nominal sample budgets and reuse fitted states
across bit budgets.
Routing leakage decreases with actual visits, and learned routing approaches
the known-route learner while a frozen uniform route does not
(\cref{fig:mechanisms}(c,f)). \Cref{tab:decomposition-main,tab:interventions-main} separates
data, tail, routing, and format terms and changes routing, preconditioning,
or precision one factor at a time. At the unit-bit-ratio setting in
Appendix~\cref{tab:revision-unified}, learned and known routing yield
\(0.0191\pm0.0002\) and \(0.0192\pm0.0002\) augmented loss, versus
\(0.0605\pm0.0008\) with a frozen route; these augmented losses
include the common \(M^{-\chi}\) proxy. Removing preconditioning gives
\(0.0230\pm0.0003\). The half-bit-ratio table reveals the uniform-precision
penalty at a tighter memory budget, complementing this unit-bit-ratio comparison.

\begin{figure}[htb]
\centering
\includegraphics[width=\linewidth]{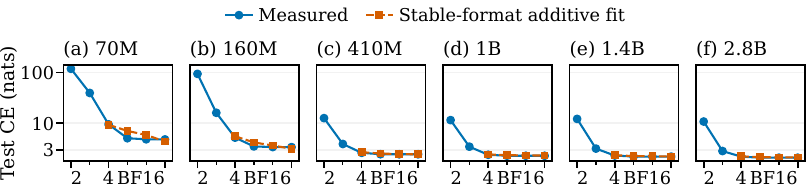}

\caption{\textbf{Data and deployed storage in Pythia.}
(a--f) 70M, 160M, 410M, 1B, 1.4B, 2.8B test CE after \(2^{26}\) extra tokens.
Blue: 2/3/4/6/8-bit and BF16 (left to right); dashed orange: additive fits to
first-three-budget 4/8-bit/BF16 validation. All points hold out data; 6-bit
also holds out format. Shared log-y; bands: 95\% BCa, eight paired training seeds.}
\label{fig:validation-summary}
\label{fig:revision-unified}

\end{figure}

\paragraph{Stored formats and softmax precision.}
We compare a shared quantization scale, scales per energy group or per
coordinate, and adaptive bit widths. Three conditioning regimes vary the
format-range amplification relative to logit radii. A softmax sweep records
numerical format, underflow, and matching
mass at \(K\in\{8,32,128,512\}\). At eight bits, adaptive widths reduce format
distortion to \(7.82\times10^{-10}\), compared with
\(8.46\times10^{-6}\) for uniform block widths and \(1.08\times10^{-5}\)
for a shared scale. Appendix~\cref{fig:colt-precision-bridge} shows the
format trade-off and retained high-precision work: bit allocation controls
distortion, while the fraction quantized determines resource savings.

\paragraph{Profile robustness.}
We test broken-power-law, logarithmically corrected, and perturbed-order
profiles using each source's exact finite spectrum. The empirical allocator uses fitted logits and counts to estimate energies,
subtracting a sampling-noise correction. Known-energy allocation instead
minimizes representation distortion using the true energies; uniform widths
provide a second reference. Comparisons use eight paired seeds. Appendix~\ref{app:joint-resource-evaluation}
gives both objectives, distinct from the deterministic unknown-exponent
schedule of \cref{thm:adaptive}. It retains the advantage over uniform
allocation across the three displayed deviations
(Appendix~\cref{fig:revision-robustness}).

\input{generated/table_main_four_panels}

\paragraph{Pretrained Transformers.}
Weight-only post-training quantization (PTQ) tests the deployment effect
across six Pythia checkpoints at \texttt{step143000}. We compare tensor/channel
8-bit integer formats (INT8), 4-bit integer (INT4) round-to-nearest (RTN),
and half-quadratic quantization (HQQ) with groups of 128 or 32 weights (g128/g32)
against bfloat16 (BF16), using paired WikiText-103 windows and LAMBADA items.
The HQQ-g32 penalty generally decreases with scale, except at 1.4B;
the 2.8B LAMBADA interval includes zero (\cref{tab:pythia-main}).
Appendix~\cref{fig:revision-pythia,tab:revision-pythia} retain all formats
and BF16 baselines. Softmax remains in 32-bit floating point (FP32), so the
comparison isolates weight representation.

We additionally continue all six checkpoints along four nested
WikiText-103 token budgets, using eight training seeds per model, and
materialize group-128 2/3/4/6/8-bit and BF16 deployment packages.
The complete scan separates model-scale and format effects
(\cref{fig:validation-summary}(a--f)): 2/3-bit quantization causes the
largest loss increases at 70M/160M, while the four larger models have
substantially smaller penalties. Including 2-bit calibration yields fits
within the specified \([0,20]\)-nat prediction range at 410M--2.8B.
At 70M/160M, the 2-bit losses exceed this range; the fitting procedure
does not yield converged, in-range solutions. Each scale is fitted separately.

The sensitivity analysis fits 4/8-bit and BF16 validation configurations and
holds out 6-bit plus the largest token budget. On untouched test windows,
the additive response improves data-extrapolation root-mean-square error
(RMSE) relative to the data-only fit at every scale. The reductions are
\(0.910\,[0.871,0.944]\) for 70M and \(0.380\,[0.352,0.394]\) for 160M;
the four larger models show smaller absolute improvements. For 6-bit
interpolation, the additive response has higher RMSE at 70M/160M and lower
RMSE at 410M--2.8B. Thus model scale changes both quantization sensitivity
and the predictive value of a smooth joint-resource response.
Appendix~\cref{fig:joint-pythia-full} shows every format and budget with
seed-level uncertainty; \cref{tab:pythia-higher-bit-extrapolation,tab:pythia-higher-bit-interpolation}
report the corresponding prediction errors for each model.

\FloatBarrier

%% file: generated/table_main_four_panels.tex
\begin{table}[htb]
\centering
\caption{\textbf{Deployment, pairing, and learning mechanisms.}
(a) HQQ-g32 minus BF16 CE: 96 WikiText-103 windows, 512 LAMBADA items.
(b) Storage slopes, 16 seeds; E/H: easy/hard. (c,d) SGD decomposition/interventions:
\(10^{-3}\) nats, eight seeds; (d) has \((\chi,\gamma)=(1,.5)\), \(M=64\),
\(K=32\), bit ratio .5 (Table~\ref{tab:revision-unified}: 1).
\(\pm\): rounded 95\% BCa envelopes (Unif.: constant across seeds).
Bold/underline: lowest/second-lowest in (a,d).}
\label{tab:main-bridges}
\IclrPreparePair
  {{\small\setlength{\tabcolsep}{8pt}\input{generated/revision_pythia_delta_compact}}}
  {{\small\setlength{\tabcolsep}{8pt}\input{generated/revision_coupling_main}}}
\typeout{ICLR9PAIR main-top A=\the\wd\IclrPairA B=\the\wd\IclrPairB widths=\the\IclrPairWidthA,\the\IclrPairWidthB gap=\the\IclrPairGap}
\begin{subtable}[t]{\IclrPairWidthA}
\centering\caption{Pythia weight-only PTQ; lower is better.}
\label{tab:pythia-main}
\resizebox{\IclrPairWidthA}{!}{\usebox{\IclrPairA}}
\end{subtable}\hspace{\IclrPairGap}%
\begin{subtable}[t]{\IclrPairWidthB}
\centering\caption{Same-marginal coupling.}
\label{tab:coupling-main}
\resizebox{\IclrPairWidthB}{!}{\usebox{\IclrPairB}}
\end{subtable}

\IclrPreparePair
  {{\small\setlength{\tabcolsep}{2pt}\input{generated/revision_decomposition_main}}}
  {{\small\setlength{\tabcolsep}{2pt}\input{generated/revision_interventions_main}}}
\typeout{ICLR9PAIR main-bottom A=\the\wd\IclrPairA B=\the\wd\IclrPairB widths=\the\IclrPairWidthA,\the\IclrPairWidthB gap=\the\IclrPairGap}
\begin{subtable}[t]{\IclrPairWidthA}
\centering\caption{Signed full-source excess-CE terms.}
\label{tab:decomposition-main}
\resizebox{\IclrPairWidthA}{!}{\usebox{\IclrPairA}}
\end{subtable}\hspace{\IclrPairGap}%
\begin{subtable}[t]{\IclrPairWidthB}
\centering\caption{Interventions; excess CE, lower is better.}
\label{tab:interventions-main}
\resizebox{\IclrPairWidthB}{!}{\usebox{\IclrPairB}}
\end{subtable}
\vspace{-0.30em}
\end{table}

%% file: generated/revision_pythia_delta_compact.tex
\begin{tabular}{lcc}
\toprule
Model & WikiText & LAMBADA \\
\midrule
70M & 3.009 {\footnotesize $\pm 0.037$} & 3.012 {\footnotesize $\pm 0.386$} \\
160M & 1.325 {\footnotesize $\pm 0.020$} & 0.589 {\footnotesize $\pm 0.225$} \\
410M & 0.156 {\footnotesize $\pm 0.003$} & 0.319 {\footnotesize $\pm 0.082$} \\
1B & \underline{0.054} {\footnotesize $\pm 0.002$} & \underline{0.021} {\footnotesize $\pm 0.028$} \\
1.4B & 0.066 {\footnotesize $\pm 0.003$} & 0.149 {\footnotesize $\pm 0.045$} \\
2.8B & \textbf{0.048} {\footnotesize $\pm 0.003$} & \textbf{0.007} {\footnotesize $\pm 0.034$} \\
\bottomrule
\end{tabular}

%% file: generated/revision_coupling_main.tex
\begin{tabular}{lcc}
\toprule
$(\chi,\mathrm{E/H})$ & Theory & Held-out CE \\
\midrule
$(1,E)$ & 1.000 & 0.972 {\footnotesize $\pm 0.0016$} \\
$(1,H)$ & 0.500 & 0.503 {\footnotesize $\pm 0.0003$} \\
$(2,E)$ & 2.000 & 2.016 {\footnotesize $\pm 0.0011$} \\
$(2,H)$ & 1.000 & 1.002 {\footnotesize $\pm 0.0002$} \\
\bottomrule
\end{tabular}

%% file: generated/revision_decomposition_main.tex
\begin{tabular}{@{}cccc@{}}
\toprule
Data & Tail & Route & Format\\
\midrule
3.584 {\scriptsize $\pm .219$} &
1.218 {\scriptsize $\pm .029$} &
.0270 {\scriptsize $\pm .0031$} &
1.491 {\scriptsize $\pm .275$}\\
\bottomrule
\end{tabular}

%% file: generated/revision_interventions_main.tex
\begin{tabular}{@{}ccccc@{}}
\toprule
Full & Known & No pre. & Frozen & Unif.\\
\midrule
\textbf{19.01} {\scriptsize $\pm .27$} &
\underline{19.06} {\scriptsize $\pm .26$} &
22.47 {\scriptsize $\pm .21$} &
60.50 {\scriptsize $\pm .74$} &
89.65 {\scriptsize $\pm .00$}\\
\bottomrule
\end{tabular}

%% file: arxiv/sections/revision_7_conclusion.tex
\section{Conclusion}
\label{sec:conclusion}

Finite data and finite learned memory resolve the same predictive-energy
spectrum at different thresholds. For the causal source studied here, this
gives minimax risk, scaling exponents, and a data--memory crossover.
The complete curve identifies the positive spectrum, retaining information
that a fitted exponent alone loses.

Energy--dimension pairing determines the spectrum: identical block
marginals can conceal different data and memory laws. A masked query--key
head learns the relevant route and values, connecting this statistical
structure to attention. Adaptive allocation, finite-precision realization,
and conditional compute laws quantify how the representation uses resources.

Experiments trace these mechanisms through spectral collapse, coupling,
routing, and numerical precision. The organizing principle is the
distribution of predictive information across independent coordinates:
it explains both what data can reveal and what a finite predictor can retain.

%% file: arxiv/sections/appendix_body.tex
\input{sections/revision_8_appendix_discussion_iclr9}
\input{sections/appendix_secondary_results}

\input{sections/revision_9_proofs}
\input{sections/2_prelim}
\input{sections/3_finite_bit}
\input{sections/5_descriptor_iclr9}
\input{sections/9_appendix}
\input{arxiv/sections/appendix_extended_results}
\input{arxiv/sections/10_experiment_appendix}
\input{sections/joint_resource_methods}
\input{sections/joint_resource_proofs}

%% file: sections/revision_8_appendix_discussion_iclr9.tex
\section{Discussion, Scope, and Closest Precedents}
\label{app:discussion}

\paragraph{One curve for two resources.}
The central object is the spectrum of predictive energy per independent
coordinate.  Its integrated tail answers both what data cannot resolve and
what a finite learned state cannot retain.  Data evaluate it at resolution
\(1/n\); storage chooses a water level.  The same-marginal construction
shows why the energy--dimension pairing is indispensable.

\paragraph{Knowledge capacity as an inverse resource curve.}
To express the information relation, place a fixed prior on the source
parameters \(Z\), including the route when it is unknown. Let \(S\) be
the learned state formed from training blocks, and let \(X\) contain all
candidates and the query of an independent held-out block. Conditional on
\((X,Z)\), the target \(Y\) is independent of the training state. Entropy
and mutual information, denoted by \(H\) and \(I\), are measured in nats.
The conditional mutual-information identity and the finite-state constraint give
\[
H(Y\mid X,S)-H(Y\mid X,Z)=I(Y;Z\mid X,S),
\qquad I(Z;S)\le H(S)\le B\log 2.
\]
The first difference compares the Bayes predictive losses given the learned
state and given the source parameters; a particular decoder can incur
additional loss. The prior is used for this information comparison, whereas
the main minimax theorem takes a worst case over fixed parameters.
This identity places knowledge-capacity and prediction-risk laws on opposite
directions of the same information path. Allen-Zhu and Li
\citep{allenzhu2025capacity} turn factual generation loss into a lower bound
on the information retained by model weights---a loss-to-bits question. Our
minimax law turns \(n\) observations and \(B\) retained bits into the smallest
held-out excess loss---a bits-to-risk question. Deterministic factual recall
and positive-entropy prediction yield complementary geometries: the former
prices discrete facts, while local KL curvature in the latter produces the
coordinate distortion \(q_i2^{-2b_i}\), reverse water filling, and the common
spectrum. Frequency-ranked knowledge also yields data and model scaling under
a compression model \citep{pan2025compression}; here the paired predictive
energies determine the complete finite-\((n,B)\) curve, and the same-marginal
construction isolates the pairing information required for its exponents.
Empirical capacity estimates \citep{allenzhu2025capacity,morris2025memorize}
calibrate how efficiently a particular architecture stores information per
parameter; the spectrum supplies the source-dependent risk curve to which
such an architectural calibration can be attached. Sparse attention, prefix
learning, and context compression
\citep{liang2025prefix,deng2024sparse,xiong2025parallelcomp,yang2026compressible}
then connect the abstract state budget to the practical question of which
predictive coordinates a Transformer should retain.

\paragraph{Learning the routing and value parameters.}
Public addresses determine where comparisons occur, unknown routing is
learned in the query projection, and unknown values are learned in the output
coordinates.  Fixing the key projection preserves the route-table scores and
routing-gradient trajectory. This identifies the trainable parameters
through which self-attention acquires the source's predictive structure.

\paragraph{A hierarchy of resource laws.}
The arbitrary-state lower bound applies to every learner with at most
\(2^B\) final value states.  The routing lower bound belongs to the stated
row-symmetric online trajectory; the format lower bound applies to the
required stored interface; the bit-operation lower bound concerns the
specified dense graph and arithmetic model. Each result identifies the cost
of a distinct bottleneck and makes the progression from statistical
information to numerical computation explicit.

\paragraph{From learning to finite-precision prediction.}
The two-stage learning theorem trains routing and output coordinates using
public coordinate frequencies, supervised calibration, and exact-arithmetic
updates. The finite-precision theorem starts from supplied mode logits and
quantifies their realization by the network; its converse concerns the
separable scalar-table representation. The compute theorem studies
precision allocation under a two-sided numerical response assumption.
Together, these results form a three-step account of acquiring parameters,
representing them numerically, and allocating arithmetic resources. A unified
low-precision training theorem would complete this chain at the level of the
optimization trajectory.

\paragraph{Precision in practical Transformer stacks.}
The benefit of a low-bit matrix path depends on the fraction of storage and
arithmetic it changes. FP32 softmax, accumulation, master weights, optimizer
states, and saved activations contribute retained high-precision work.
Streaming attention reduces retained state while preserving the dense
arithmetic order. Device-level speedups are governed by the corresponding
throughput and memory-traffic costs.

\paragraph{Open directions.}
Three extensions would carry the spectrum principle into broader sequence
models: jointly learning frequencies, ordering, queries, and keys; analyzing
low-precision stochastic updates along the routing and value trajectories;
and estimating predictive-energy spectra from unrestricted corpora. These
directions turn the present exact source model into a program for predicting
resource tradeoffs in learned representations.

%% file: sections/appendix_secondary_results.tex
\section{Extended Representation and Compute Results}
\label{app:secondary-results}

This appendix gives the complete statements and derivations summarized in
Sections~\ref{sec:revision-branches} and~\ref{sec:revision-compute}.

\input{sections/appendix_representation_details_iclr9}
\input{sections/7_bit_complexity}

%% file: sections/appendix_representation_details_iclr9.tex
\subsection{Adaptive Schedule and Information Used}

The schedule in \cref{eq:universal-schedule} knows the canonical index order
and the interval \([\chi_{\min},\chi_{\max}]\), but not \(\chi\).
It adapts to the decay rate while treating energy ordering as side information.
Data-driven allocation with an unknown ordering is related to adaptive
quantized Sobolev estimation \citep{zhu2018sobolev}; learning that ordering
together with the bit schedule is a separate statistical problem.
The full proof and restatement are in Appendix~\ref{app:adaptive}.

\subsection{Finite-Precision Forward Error and Storage}

The matching attention weight in \cref{eq:attention-map} is
\(\alpha_\star=e^\beta/(e^\beta+K-1)\).  Since the output coordinates lie
in \([-R,R]\),
\(\abs{z-\vartheta_{j_\star}}\le2R(1-\alpha_\star)
\le2R(K-1)e^{-\beta}\).
Uniformly quantizing \([-\Delta_j,\Delta_j]\) with \(2^{b_j}\) values gives
\begin{equation}
\abs{\widetilde w_j^{\rm mode}-w_j^{\rm mode}}\le2\Delta_j2^{-b_j}.
\label{eq:scalar-quantization}
\end{equation}
For arithmetic error, let \(u\) be the unit roundoff of the scalar
operations. Max shifting produces unnormalized weights in
\(\{1,e^{-\beta}\}\). Balanced pairwise summation has depth
\(h=\ceil{\log_2K}\); put
\(\gamma_h=hu/(1-hu)\).  Positive pairwise summation incurs relative
error at most \(\gamma_h\).  Composing exponential, denominator, and division
errors gives, when \((h+4)u\le1/4\),
\begin{equation}
\norm{\widehat\alpha-\alpha}_1\le8(h+4)u.
\label{eq:softmax-forward-error}
\end{equation}
The final weighted sum can have mixed signs, so an absolute bound is needed:
\begin{equation}
\abs{\operatorname{fl}(\widetilde w^\top\widehat\alpha)
-\widetilde w^\top\alpha}\le C_Ru(1+\log K).
\label{eq:output-forward-error}
\end{equation}
The absolute bound also handles cancellation. Combining routing,
quantization, and arithmetic errors, squaring and averaging over the mode,
and applying bounded log-loss curvature yields \cref{eq:fp-bound}, including
the unrepresented tail.  The detailed proof and restatement of
\cref{thm:transformer} are in Appendix~\ref{app:transformer}.

\paragraph{Learned state versus physical storage.}
Here \(d=K+M+1\) is the token width of the particular realization in
\cref{sec:transformer}, with \(M\) retained scalar output coordinates.
If the score matrix, value projection, and output head are dense arrays,
their physical scalar-slot count is
\begin{equation}
N=2d^2+d.
\label{eq:dense-storage}
\end{equation}
Only the \(M\) output coordinates depend on the source, and their description
length is \(\sum_{j=1}^M b_j\).  Thus \(N\), \(M\), and \(B\) are
different resources even in this fixed-depth, one-head realization with
public orthogonal addresses.  The finite-precision theorem initializes those
coordinates with supplied logits; \cref{thm:causal-realization} separately
analyzes gradient training of routing and output values.

%% file: sections/7_bit_complexity.tex
\subsection{Forward--Backward Bit Complexity}
\label{sec:bit-complexity}

\label{sec:fb-count}

Consider eager, non-checkpointed reverse-mode differentiation (standard
backpropagation) through a dense
Transformer with sequence length \(\ell\), hidden width
\(d\), feed-forward width \(m_{\rm ff}\), \(H\) attention heads whose widths
sum to \(d\), and \(L\) layers.  Matrix multiplications use a format indexed
by \(b_{\rm mm}\), the stable
max-shifted softmax uses compute precision \(b_{\rm sm}\), and a materialized
attention probability uses \(b_{\rm st}\) storage bits.  The index
\(b_{\rm mm}\) determines all operand, output, and accumulator widths in this
format family.  Let
\(\mathsf M(b_{\rm mm})\) be the bit-operation cost of one such
multiply--accumulate.  Let \(\mathsf S(b_{\rm sm})\) be the amortized
per-attention-entry cost of stable softmax and its reverse-mode
Jacobian--vector product, including exponentials, reductions, division, and
pointwise arithmetic.  We assume normalization, activation, and cast costs per
entry do not asymptotically exceed these matrix and softmax costs.  We
do not count the optimizer update.  This conventional query/key/value/output
stack is a separate computational graph from the smaller realization in
\cref{sec:transformer}.  Let \(b_w\) and \(b_a\) denote aggregate bits per
stored weight and activation slot at peak memory, and let
\(\operatorname{Mem}_{\rm hp}\) denote retained high-precision state.

\begin{restatable}[Transformer forward--backward bit complexity; informal version of \cref{full:lem:bit-complexity}]{lemma}{LemBitComplexity}
\label{lem:bit-complexity}
For the graph above, with positive dimensions and conventional dense matrix
multiplication,
\begin{equation}
  \operatorname{BitOps}_{\rm fb}
  \asymp L\Bigl[
    \bigl(\ell d^2+\ell d m_{\rm ff}+\ell^2d\bigr)
      \mathsf M(b_{\rm mm})
    +H\ell^2\mathsf S(b_{\rm sm})
  \Bigr].
  \label{eq:fb-bitops}
\end{equation}
\begin{equation}
  \operatorname{Mem}_{\rm fb}
  \asymp
    b_wL(d^2+dm_{\rm ff})
    +b_aL\ell(d+m_{\rm ff})
    +b_{\rm st}LH\ell^2
    +\operatorname{Mem}_{\rm hp}.
  \label{eq:fb-memory}
\end{equation}
The memory \(\operatorname{Mem}_{\rm fb}\) is measured in bits.  The values
\(b_w\), \(b_a\), and \(b_{\rm st}\) count aggregate bits per logical weight,
activation, and saved-attention slot at peak memory; they include any
same-shaped gradient or backward-workspace buffer materialized by this eager
implementation.  The term
\(\operatorname{Mem}_{\rm hp}\) contains master weights, optimizer state, and
all other explicitly retained high-precision data.  A streaming, recomputing
attention implementation replaces the quadratic attention state saved for
backward by linear online-softmax state plus temporary tile workspace, while
leaving the dense attention arithmetic order unchanged.
\end{restatable}

The three matrix terms in \cref{eq:fb-bitops} come respectively from the
input/output projections, the two feed-forward projections, and the
score--value attention products.  Reverse mode performs two matrix products
for each forward matrix product, so it changes constants rather than orders.
For a softmax row \(p\) and incoming vector \(v\), its derivative is applied
as \((\operatorname{diag}(p)-pp^\top)v=
p\odot(v-\inner{p}{v}\mathbf 1)\).
Thus the Jacobian--vector product costs \(\Theta(H\ell^2)\), without forming
an \(\ell\times\ell\) Jacobian for every row.  The proof and a shape-by-shape
operation count are in \cref{app:bit-complexity}.

\subsection{The Compute--Precision Pareto Law}
\label{sec:precision-pareto}

Let \(M\) denote the number of represented modes.  Under a specified shape
schedule \((\ell,d,m_{\rm ff},H,L)(M)\), \cref{lem:bit-complexity} gives
\(G(M)\asymp L(\ell d^2+\ell d m_{\rm ff}+\ell^2d)\) and
\(S(M)\asymp LH\ell^2\).
We write \(G(M)\asymp M^g\) for quantizable matrix work and
\(S(M)\asymp M^s\) for softmax work, where \(g,s\ge0\).  Reaching model
error \(M^{-\chi}\) requires \(n(M)\asymp M^{1+\chi}\) independent blocks.
Let \(\mathcal E\) denote held-out excess risk, and let
\(\kappa_{\rm mm}\) and \(\kappa_{\rm sm}\) measure how matrix and softmax
rounding errors are amplified.  We assume the following two-sided error and
cost relations, which separate effective matrix precision \(b\) from effective
softmax precision \(b_{\rm sm}\):
\begin{equation}
  \mathcal E(M,b,b_{\rm sm})
  \asymp
  M^{-\chi}
  +\kappa_{\rm mm}(M)^22^{-2b}
  +\kappa_{\rm sm}(M)^22^{-2b_{\rm sm}}.
  \label{eq:mixed-risk-relations}
\end{equation}
\begin{equation}
  C
  \asymp
  M^{1+\chi}\left[
    G(M)b^\nu+S(M)b_{\rm sm}^{\nu_{\rm sm}}
  \right].
  \label{eq:mixed-cost-relations}
\end{equation}
The exponents \(\nu,\nu_{\rm sm}\ge0\) belong to the stated arithmetic
model.  The abstract format family has total width proportional to its
effective precision, equivalently unit roundoff \(\Theta(2^{-b})\); concrete
floating-point formats are discussed separately below.

\begin{restatable}[Compute--precision Pareto law; informal version of \cref{full:thm:precision-tradeoff}]{theorem}{ThmPrecisionTradeoff}
\label{thm:precision-tradeoff}
Assume \cref{eq:mixed-risk-relations,eq:mixed-cost-relations},
\(\kappa_{\rm mm}(M)\asymp M^r\), and
\(\kappa_{\rm sm}(M)\asymp M^{r_{\rm sm}}\), with all displayed relations
two-sided and \(r,r_{\rm sm}\ge0\).  For the specified dense graph,
\(H\le d\) implies \(G\gtrsim S\) and hence \(g\ge s\).
Item 3 concerns general operation counts satisfying the same two-sided
relations; items 1, 2, and 4 apply to both settings.
\begin{enumerate}
  \item The smallest matrix precision that keeps matrix rounding at the model
  error is
  \begin{equation}
    b_{\rm match}(M)
    =\left(r+\frac{\chi}{2}\right)\log_2M+O(1).
    \label{eq:matching-bits}
  \end{equation}

  \item If the matrix term dominates the bracket in
  \cref{eq:mixed-cost-relations} at matched precisions---in particular in the
  case \(g\ge s\) when the remaining precision factors do not reverse
  that dominance---then
  \begin{equation}
    \inf_{M,b,b_{\rm sm}:\,\mathrm{cost}\le C}\mathcal E
    \asymp
    C^{-\chi/(1+\chi+g)}
    (\log C)^{\chi\nu/(1+\chi+g)},
    \label{eq:matrix-dominated-phase}
  \end{equation}
  up to other declared slowly varying factors.

  \item For general operation counts, if \(s>g\), \(b_{\rm sm}\) is fixed,
  and the softmax term dominates
  computation, then for budgets in the pre-floor range
  \[
    C^{1/(1+\chi+s)}
    \lesssim 2^{2b_{\rm sm}/(\chi+2r_{\rm sm})},
  \]
  \begin{equation}
    \inf_{M,b:\,\mathrm{cost}\le C}\mathcal E
    \asymp C^{-\chi/(1+\chi+s)}.
    \label{eq:softmax-dominated-phase}
  \end{equation}
  Reducing matrix precision cannot improve this leading exponent.

  \item At fixed matrix precision \(b\), with softmax precision optimized,
  the model--rounding crossover and the infimum error over model scales are
  \begin{equation}
    M_b\asymp2^{2b/(\chi+2r)},
    \qquad
    \mathcal E_{\rm floor}(b)
    \asymp2^{-2\chi b/(\chi+2r)}.
    \label{eq:fixed-bit-floor}
  \end{equation}
  A fixed \(b_{\rm sm}\)-bit softmax similarly supports the scaling phase only
  while
  \begin{equation}
    M\lesssim2^{2b_{\rm sm}/(\chi+2r_{\rm sm})}.
    \label{eq:softmax-horizon}
  \end{equation}
\end{enumerate}
The displayed two-sided risk and arithmetic assumptions define the class over
which the infima and matching bounds are taken.
\end{restatable}

The theorem distinguishes increasing precision from fixed-precision deployment.
Polynomial conditioning requires only logarithmically increasing precision,
so adaptive precision changes a logarithmic factor, not the power exponent.
Conversely, permanently fixing the precision creates a numerical floor.  For
a concrete 32-bit floating-point (FP32) softmax with unit roundoff
\(u_{\rm sm}\),
provided the stated exponent-range and relative-error assumptions remain valid,
\cref{eq:softmax-horizon} reads
\(M\lesssim u_{\rm sm}^{-2/(\chi+2r_{\rm sm})}\); the total format width is
not substituted for effective accuracy bits.  Exponential conditioning, treated in
\cref{thm:compute}, is the case in which precision itself creates a new
power-law phase.

\paragraph{Limits imposed by high-precision work.}
\label{sec:precision-limits}

At a fixed model shape, lowering a matrix path from baseline precision
\(b_0\) to \(b\) gives the relative bit-operation count
\begin{equation}
  \operatorname{BitOpsRatio}
  =
  \frac{G b_0^\nu+S b_{\rm sm}^{\nu_{\rm sm}}}
       {G b^\nu+S b_{\rm sm}^{\nu_{\rm sm}}}.
  \label{eq:bitops-ratio}
\end{equation}
If \(Q\) stored scalars can be quantized and \(R_{\rm hp}\) bits must remain
at high precision, the corresponding storage ratio is
\begin{equation}
  \operatorname{Compression}_{\rm bit}
  =\frac{Qb_0+R_{\rm hp}}{Qb+R_{\rm hp}}.
  \label{eq:memory-ratio}
\end{equation}
These formulas expose the ceiling imposed by high-precision normalization,
master weights, optimizer states, or saved activations.  Actual graphics
processor speed also depends on memory traffic, kernel support, and hardware
throughput.

The two analytic consequences appear in
\cref{fig:compute-summary}: its left panel shows the retained-high-precision
ceiling in \cref{eq:bitops-ratio}, and its right panel compares the three
risk--cost regimes above.

\FloatBarrier

%% file: sections/revision_9_proofs.tex
\section{Proofs for the Main Results}
\label{app:revision-proofs}

The proofs follow the paper's mechanism order: the spectral identities and
their exponents, the query--key realization and online route, the network
upper bound, and the same-marginal separation.

Throughout this appendix, \(i\) indexes the scalar coordinates defined in
\cref{sec:revision-model}. Their energies are \(q_i=p_iR_i^2\), the
spectrum is \(\mu=\sum_i\delta_{q_i}\), and
\(\Phi_\mu(t)=\sum_i\min\{q_i,t\}\). The water level \(\tau_B\)
is selected by the storage budget in that setup. Each restatement below
uses the same objects as its corresponding main-text result.

\FullThmSpectrum
\begin{proof}
Summability gives \(\mathcal N_\mu(t)<\infty\) for every \(t>0\), and the
largest positive energy is attained.  Tonelli's theorem yields
\[
\Phi_\mu(t)
=\sum_i\int_0^t\mathbf 1\{s<q_i\}\,\mathrm ds
=\sum_i\min\{q_i,t\},
\]
which proves \cref{eq:spectrum-data}.  For a finite storage budget, reverse
water filling has rates
\[
b_i^*=\frac12\log_2(q_i/\tau_B)_+.
\]
Here the positive part applies to the logarithm.  To justify the optimizer
without interchanging a limit and an infimum, write
\[
B(\tau)=\frac1{2\ln2}\sum_i[\ln(q_i/\tau)]_+.
\]
On every compact subinterval of \((0,q_{\max}]\), only finitely many
summands can be nonzero.  Thus \(B(\tau)\) is continuous, strictly decreasing
below \(q_{\max}\), equals zero at \(q_{\max}\), and tends to infinity as
\(\tau\downarrow0\).  This proves existence and uniqueness of \(\tau_B\)
for \(B>0\); for \(B=0\) take \(\tau_0=q_{\max}\).
For \(\lambda=2(\ln2)\tau_B\), scalar convexity gives, for every \(b_i\ge0\),
\[
q_i2^{-2b_i}+\lambda b_i
\ge q_i2^{-2b_i^*}+\lambda b_i^*.
\]
All distortion sums are bounded by \(\sum_iq_i\), and the two bit sums are
finite.  Summing the inequality over all coordinates shows that every
feasible allocation has distortion at least
\(\sum_iq_i2^{-2b_i^*}+\lambda(B-\sum_i b_i)\ge\sum_iq_i2^{-2b_i^*}\).
This proves optimality directly in the countable problem.  Since
\(q_i2^{-2b_i^*}=\min\{q_i,\tau_B\}\), this proves
\cref{eq:spectrum-bits}.  The bounded-logit curvature, Assouad, and arbitrary
codebook bounds, proved in \cref{full:thm:finite-bit}, give
\(\mathfrak R^*_{\rm value}(n,B)\asymp_R\mathcal S_q(n)+D_q(B)\), proving
\cref{eq:spectrum-risk}.  Finally, for \(0<h<t/2\), both one-sided difference
quotients of \(\min\{q_i,t\}\) lie in \([0,1]\) and vanish whenever
\(q_i<t/2\).  There are only finitely many remaining coordinates.
Taking their limits gives respectively \(\mathbf1\{q_i>t\}\) and
\(\mathbf1\{q_i\ge t\}\), proving \cref{eq:spectrum-derivatives}.
Their difference is exactly the multiplicity of the atom at \(t\).
\end{proof}

\FullCorSpectrumRates
\begin{proof}
We first prove the regular-variation assertion under its stated hypothesis,
then treat power comparisons separately.  Abbreviate
\(N(t)=\mathcal N_\mu(t)\).  For fixed \(x>0\),
\(N(tx)/N(t)\to x^{-\rho}\) as \(t\downarrow0\).
Choose \(0<\epsilon<\min\{\rho,1-\rho\}\).
The ratio at \(x=1/2\), iterated over dyadic intervals and combined with
monotonicity, gives
\[
\frac{N(tx)}{N(t)}\le C_\epsilon x^{-\rho-\epsilon}\quad(0<x\le1),
\qquad
\frac{N(tx)}{N(t)}\le C_\epsilon x^{-\rho+\epsilon}\quad(1\le x\le t_0/t)
\]
for a fixed sufficiently small \(t_0\).  For example, choose \(t_0\) so that
the ratio \(N(s/2)/N(s)\) lies between \(2^{\rho-\epsilon}\) and
\(2^{\rho+\epsilon}\) for all \(s\le t_0\), and telescope; rounding \(x\)
to its neighboring powers of two changes only \(C_\epsilon\).
Consequently dominated convergence in
\(\Phi_\mu(t)/(tN(t))=\int_0^1N(tx)/N(t)\,\mathrm dx\) gives
\[
\Phi_\mu(t)\sim\frac{t^{1-\rho}L(1/t)}{1-\rho}.
\]
Layer integration of the storage budget gives
\[
\mathcal B_\mu(\tau)
=\frac1{2\ln2}\int_\tau^{q_{\max}}
\frac{\mathcal N_\mu(t)}t\,\mathrm dt
\sim\frac{\tau^{-\rho}L(1/\tau)}{2\rho\ln2}.
\]
Indeed, after substituting \(t=\tau x\), the second domination bound makes
the integral converge to \(\int_1^\infty x^{-\rho-1}\,\mathrm dx=1/\rho\).
The part above \(t_0\) is finite and negligible relative to \(N(\tau)\).
Thus \(B(c\tau)/B(\tau)\to c^{-\rho}\).
For \(a>0\), sandwiching the inverse between
\((a^{-1/\rho}-\varepsilon)\tau_B\) and
\((a^{-1/\rho}+\varepsilon)\tau_B\), and then sending
\(\varepsilon\downarrow0\), proves
\(\tau_{aB}/\tau_B\to a^{-1/\rho}\).
The same monotone sandwich for \(\Phi_\mu\) gives
\(\Phi_\mu(\tau_{aB})/\Phi_\mu(\tau_B)\to a^{-(1-\rho)/\rho}\).
This proves both regular-variation indices without assuming a differentiable
slowly varying factor.

For the block profile, dimensions are positive integers, so the stipulated
comparison \(d_j\asymp j^\gamma\) has \(\gamma\ge0\).
Coordinate energy \(a_j/d_j\asymp j^{-(1+\chi+\gamma)}\) occurs
\(d_j\asymp j^\gamma\) times.  Bounds on the two comparison constants show
that every index below \(c t^{-1/(1+\chi+\gamma)}\) is active and no index
above \(C t^{-1/(1+\chi+\gamma)}\) is active.  Summing the multiplicities
therefore gives \(N(t)\asymp t^{-\rho}\) with
\(\rho=(1+\gamma)/(1+\chi+\gamma)\).
Integrating these two-sided bounds directly yields
\[
\Phi_\mu(t)\asymp t^{1-\rho},\qquad
B(\tau)\asymp\tau^{-\rho},\qquad
\Phi_\mu(\tau_B)\asymp B^{-(1-\rho)/\rho}.
\]
These power orders do not require regular variation of the individual
profile or of an atomic counting tail.
\end{proof}

\FullPropQK
\begin{proof}
Because \(P_Qx_q=e_q\), \(P_Cx_s=e_s\), and \(J^\top J=I_K\), direct
substitution gives \cref{eq:qk-equivalence}.  If
\(G=\nabla_A\ell\), the physical query-projection gradient is
\(\nabla_{W_Q}\ell=JG^\top P_Q\).  Restricting training to this address block
therefore maps one Euclidean step to \(W_Q(A-\eta G)\).  The scores,
candidate-only masked softmax, and attended value feature coincide.  The
output loss sees the same feature, so its gradient with respect to \(w\) also
coincides.  Orthogonality of the residual subspace to \(w\) makes the residual
logit zero.
\end{proof}

\FullPropRouting
\begin{proof}
Zero initialization preserves equality among distractor scores.  If \(u_t\)
is the correct score and \(v_t\) a distractor score, one softmax
cross-entropy update gives
\[
u_{t+1}=u_t+\eta(1-p_\star),\qquad
v_{t+1}=v_t-\eta p_0,
\]
which proves the margin recurrence.  Introduce
\[
\Psi_K(\beta)=\frac{e^\beta-1+(K-1)\beta}{K}.
\]
Put \(d_t=\beta_{t+1}-\beta_t\).  Since \(\beta_t\ge0\),
\(0<d_t\le\eta\), and
\(\Psi_K'(\beta_t)d_t=\eta\).  Moreover
\(1\le\Psi_K'(\beta_t+v)/\Psi_K'(\beta_t)\le e^\eta\) for
\(0\le v\le d_t\).  Integrating and summing gives
\(\eta t\le\Psi_K(\beta_t)\le e^\eta\eta t\).
Define the positive burn-in time
\[
b_K=\min\{t:\beta_t\ge\ln(2K)\}.
\]
The preceding bounds, and the overshoot \(d_t\le\eta\), imply
\(b_K\asymp_\eta\ln(eK)\), including \(K=2\).
Set \(x_t=e^{\beta_t}/(K-1)\).  Its increment satisfies
\[
x_td_t\le x_{t+1}-x_t\le e^\eta x_td_t,
\qquad
x_td_t=\frac{\eta K}{K-1}\frac{x_t}{1+x_t}.
\]
At burn-in \(2\le x_{b_K}\le4e^\eta\), and thereafter each increment is
between two positive constants depending only on \(\eta\).  Hence
\(\delta_{b_K+k}=1/(1+x_{b_K+k})\asymp_\eta(1+k)^{-1}\).
Before burn-in, \(x_t<4e^\eta\), so \(\delta_t\asymp_\eta1\).
This proves the shifted reciprocal bound at every integer time.

For a held-out row, independence and uniform query addresses give
\(N\sim\operatorname{Binomial}(S_A,1/K)\); write \(m=S_A/K\).
If \(m\ge4b_K\), the binomial exponential-moment bound gives
\(\Pr\{N<m/2\}\le e^{-m/8}\).  On the complementary event,
\(\delta_N^2\le C_\eta(1+m)^{-2}\).  Thus
\[
\E\delta_N^2\le C_\eta(1+m)^{-2}+e^{-m/8},
\]
which is at most \(\varepsilon\) when
\(m\ge C_\eta\{\ln(eK)+\varepsilon^{-1/2}\}\).
For the reverse bound, the increment estimate holds globally from
\(x_0\le1\), giving \(x_t\le1+2\eta e^\eta t\).
Convexity of \((2+2\eta e^\eta t)^{-2}\) and Jensen's inequality imply
\(\E\delta_N^2\ge(2+2\eta e^\eta m)^{-2}\).  Since
\(\varepsilon\le1/16\), this forces \(m\ge c_\eta\varepsilon^{-1/2}\).
For the logarithmic term, \(\beta_t\le2\eta t\).  If
\(m\le\ln(K-1)/(8\eta)\), Markov's inequality gives probability at least
\(1/2\) to \(N\le2m\), where \(\beta_N\le\frac12\ln(K-1)\) and
\(\delta_N\ge1/2\).  Then \(\E\delta_N^2\ge1/8\), a contradiction.
This forces \(m\gtrsim_\eta\ln K\); bounded \(K\), including \(K=2\),
is already covered by the first lower bound.  Taking the maximum of the two
lower bounds proves the claimed order of the hitting time.
With a single candidate routing is exact without calibration.
\end{proof}

In the following restatement, \(\Delta_i=R_i\) denotes the same public
logit radius as in the main-text setup.

\FullThmCausalRealization
\begin{proof}
The finite active set has size \(m=|A_n|\le n\sum_iq_i\).
We first analyze the exact output before storage, conditioning on Stage A.

\paragraph{Feature covariance and curvature.}
For \(K\ge2\), write \(I\) for the matching candidate's mode and
\(I_1,\ldots,I_{K-1}\) for the distractor modes.  They are independent with
law \(p\), also conditional on the query row and its calibration history.
Restricting the coordinate vectors to \(A_n\), the feature is
\[
h=(1-\delta_q)e_I+\frac{\delta_q}{K-1}\sum_{s=1}^{K-1}e_{I_s}.
\]
Expanding its second moment separates equal-candidate and distinct-candidate
products:
\[
\E[hh^\top\mid Q=q]
=c_{d,q}D+(1-c_{d,q})p_Ap_A^\top,
\qquad c_{d,q}=(1-\delta_q)^2+\frac{\delta_q^2}{K-1},
\]
where \(p_A=(p_i)_{i\in A_n}\).  With
\(r_{\rm cal}=K^{-1}\sum_q\delta_q^2\) and
\(\bar c_d=K^{-1}\sum_qc_{d,q}\), the event
\(\mathcal G=\{r_{\rm cal}\le1/16\}\) gives
\[
\bar c_d\ge\left(1-\frac1K\sum_q\delta_q\right)^2\ge\frac9{16}.
\]
Let \(F(w)=\E[\log(1+e^{w^\top h})-Yw^\top h\mid\text{Stage A}]\).
Every \(w\in\mathcal W\) has \(|w^\top h|\le R\).  Hence, on
\(\mathcal G\), its Hessian satisfies
\(\nabla^2F(w)\succeq\mu_{\rm sgd} D\), with \(\mu_{\rm sgd}\) as in the statement.
The stochastic gradient \(g_t=(\sigma(w_t^\top h_t)-Y_t)h_t\) is unbiased
conditional on the past, and, for \(m\ge1\),
\[
\E[\|g_t\|_{D^{-1}}^2\mid w_t,\text{Stage A}]
\le\operatorname{tr}(D^{-1}\E hh^\top)
=\bar c_dm+(1-\bar c_d)\sum_{i\in A_n}p_i\le m.
\]
For \(K=1\), use \(h=e_I\), \(\bar c_d=1\), and
\(r_{\rm cal}=0\); all the same inequalities hold.

\paragraph{Projected SGD and its output.}
The compact box has a minimizer \(w^*\) of \(F\).  Nonexpansiveness of
projection in the \(D\)-norm and strong convexity give, conditional on
\(\mathcal G\) and the calibration history,
\[
2\eta_t\E[F(w_t)-F(w^*)]
\le(1-\mu_{\rm sgd}\eta_t)V_t-V_{t+1}+\eta_t^2m,
\qquad V_t=\E\|w_t-w^*\|_D^2.
\]
Multiply by \(t/(2\eta_t)\) and substitute the specified step size:
\[
t\E[F(w_t)-F(w^*)]
\le\frac{\mu_{\rm sgd}}4\{t(t-1)V_t-t(t+1)V_{t+1}\}
+\frac{mt}{\mu_{\rm sgd}(t+1)}.
\]
The distance terms telescope, with zero initial coefficient and a
nonpositive final term.  Convexity of \(F\) therefore yields
\[
\E[F(\bar w_n)-F(w^*)]\le\frac{2m}{\mu_{\rm sgd}(n+1)}.
\]

\paragraph{Comparison with the source.}
Let \(\theta^A_i=\theta_i\mathbf1\{i\in A_n\}\).  It belongs to
\(\mathcal W\), and
\[
|\theta_I-(\theta^A)^\top h|
\le |\theta_I|\mathbf1\{I\notin A_n\}+2R\delta_q.
\]
Squaring, averaging, and applying upper logit curvature give
\[
F(w^*)-\Linf\le F(\theta^A)-\Linf
\le C_R\left(\sum_{i\notin A_n}q_i+r_{\rm cal}\right).
\]
On \(\mathcal G^c\), every exact output remains in \([-R,R]\), so its
excess loss is bounded by a constant depending only on \(R\).
Furthermore \(\Pr(\mathcal G^c)\le16\E r_{\rm cal}\), and independence
of a held-out query gives \(\E r_{\rm cal}=\mathfrak r_K(S_A)\).
Combining the events proves
\[
\E[F(\bar w_n)-\Linf]
\le C_R\left\{\frac mn+\sum_{i\notin A_n}q_i+
\mathfrak r_K(S_A)\right\}
=C_R\{\Phi_\mu(n^{-1})+\mathfrak r_K(S_A)\}.
\]
When \(m=0\), the zero output pays at most \(C_R\sum_iq_i
=C_R\Phi_\mu(n^{-1})\), so the same bound holds.

\paragraph{Finite-state representation.}
For any \(x\in[-\Delta_i,\Delta_i]\), the stated \(k\)-bit decoder obeys
\[
|Q_{i,k}(x)-x|^2
\le\Delta_i^2\min\{1,\kappa_i^22^{-2k}\}.
\]
Indeed, the zero branch has error at most \(\Delta_i\); in the grid branch
the midpoint error is at most \(\kappa_i\Delta_i2^{-k}\), and clipping
decreases it.  Both the branch and the codebook are public.  Rounding
\(b_i\) down to \(k_i\) multiplies this upper bound by at most four.
Only finitely many \(k_i\) are positive, and their product codebook has
at most \(2^{\sum_i k_i}\le2^{B_V}\) entries.  A zero-width coordinate is
decoded as zero because \(\kappa_i\ge1\).

Extend \(\bar w_n\) by zero outside \(A_n\), and apply the public
scalar decoder on every allocated coordinate. Let
\(e_i=Q_{i,k_i}(\bar w_{n,i})-\bar w_{n,i}\) and let
\(H=\sum_s a_{Qs}e_{C_s}\) be the full attended coordinate feature.
Conditional on the fitted state, Jensen's inequality and the iid candidate
marginals give
\[
\E[(e^\top H)^2\mid\text{fitted state}]
\le\E\!\left[\sum_s a_{Qs}e_{C_s}^2\mid\text{fitted state}\right]
=\sum_i p_i e_i^2.
\]
It follows that quantizing the average costs at most
\(4\sum_iq_i\min\{1,\kappa_i^22^{-2b_i}\}\) in mean squared logit error.
Lower curvature converts the preceding learning bound to squared logit
error; the squared triangle inequality adds storage and arithmetic error,
and global upper curvature converts back to cross-entropy.  Taking the
infimum over allocations proves \cref{eq:causal-realization}.

Finally, a calibrated row is determined by its visit count and its matching
position, with one additional symbol for an unvisited row.  It therefore
uses at most
\[
B_A=K\{\lceil\log_2(S_A+1)\rceil+\lceil\log_2(K+1)\rceil\}
\]
bits: the fixed decoder reconstructs its scores from the recurrence.
The retained value and routing state together use at most \(B_V+B_A\)
bits; the stated arithmetic error accounts for numerical reconstruction.
For \(K=1\) the route uses no state.  When \(\kappa_i=1\), the format
functional is exactly \(D_q(B_V)=\Phi_\mu(\tau_{B_V})\), giving the final
claim at value budget \(B_V\). For ranked coordinate energies
\(q_{(r)}\asymp r^{-(1+\xi)}\), \(D_q(B_V)\asymp B_V^{-\xi}\);
in the equal-energy block profile, \(\xi=\chi/(1+\gamma)\).
If \(B_A=o(B)\) and
\(B_V=B-B_A\), then \(B_V\asymp B\), proving the total-budget conclusion.
\end{proof}

\FullThmDyadic
\begin{proof}
Fix \(\delta>\gamma\chi\).  At each level \(\ell\ge1\), take \(2^\ell\)
blocks of energy \(c2^{-(1+\chi)\ell}\) and \(2^\ell\) blocks of energy
\(c2^{-(1+\chi+\delta)\ell}\).  The dimension multiset consists of
\(2^\ell\) copies each of \(1\) and \(d_\ell=\lceil2^{\gamma\ell}\rceil\).
In the easy pairing the higher energy has dimension one; in the hard pairing
it has dimension \(d_\ell\).  Divide a block's energy equally among its
coordinates.  Both multisets agree exactly at every level, including after
the integer rounding of dimensions.

To realize these spectra as sources, choose any fixed radius \(R_0>0\)
within the logit bound, and choose \(c>0\) so that total energy is \(R_0^2\).
Give a coordinate of energy \(q_i\) probability \(p_i=q_i/R_0^2\) and radius
\(\Delta_i=R_0\).  These probabilities sum to one.  Independent candidate
draws and the same fixed-\(K\) retrieval map give the source in
\cref{sec:revision-model}; bounded Bernoulli logits give strictly positive
conditional entropy.

A generic component with exponent \(s>0\) and dimension exponent \(g\ge0\)
has \(\Theta(2^{(1+g)\ell})\) coordinates of energy
\(\Theta(2^{-(1+s+g)\ell})\) at level \(\ell\).
The active levels at threshold \(t\) lie between two cutoffs differing by a
constant number of levels, each of size
\(\log_2(1/t)/(1+s+g)\).  A geometric sum therefore gives
\[
N_{s,g}(t)\asymp t^{-\rho(s,g)},\qquad
\rho(s,g)=\frac{1+g}{1+s+g}.
\]
Use the two-sided integral comparison at the end of the preceding spectral
proof, rather than regular variation: dyadic spectra can have log-periodic
fluctuations.  This gives data and storage exponents
\(s/(1+s+g)\) and \(s/(1+g)\), respectively.
The easy components have \((s,g)=(\chi,0),(\chi+\delta,\gamma)\).
The inequality \(\delta>\gamma\chi\) makes the second component's two
exponents strictly larger.  The hard components have
\((s,g)=(\chi,\gamma),(\chi+\delta,0)\), again with the second faster.
For data, component errors add exactly.  For storage, each full distortion
is at least that of the slower component at budget \(B\), while allocating
\(B/2\) to each component gives the matching upper bound.
Applying \cref{full:thm:spectrum} now proves all four displayed rates.
\end{proof}

%% file: sections/2_prelim.tex
\section{Preliminaries and Notation}
\label{sec:prelim}

We begin by specifying what is observed, what is predicted, and which resources
are counted.  \Cref{sec:source} introduces the regenerative autoregressive
source.  \Cref{sec:risk-bits} defines held-out loss and learners with finitely
many data-dependent states.  \Cref{sec:actions} adds mode-dependent
representation costs, and \cref{sec:resources} distinguishes learned bits from
physical storage and arithmetic cost.

\subsection{A regenerative autoregressive source}
\label{sec:source}

This section writes the revealed-coordinate value problem with scalar
index \(j\): \(\Delta_j\) is the logit radius denoted \(R_i\) in the main
text, and \(\Excess^*(n,B)\) below is \(\mathfrak R^*_{\rm value}(n,B)\).
For a positive integer \(m\), write \([m]=\{1,\ldots,m\}\), and let
\(\N_0=\{0,1,2,\ldots\}\).  We write \(\Ber(p)\) for the Bernoulli
distribution with success probability \(p\).

One regeneration block contains \(K\) candidate tokens followed by a query and
a binary target token.  Candidate \(t\in[K]\) carries a public address \(a_t\)
and an observed mode label \(j_t\).  The query carries one address
\(a_{t_\star}\), so
its causally relevant mode is
\[
  J\defeq j_{t_\star},\qquad \P(J=j)=p_j.
\]
Conditionally on \(J=j\), the next token obeys
\[
  Y\mid J=j\sim\Ber(\sig(\theta_j)),
  \qquad \sig(z)=\frac{1}{1+e^{-z}}.
\]
The parameter is \(\theta=(\theta_j)_{j\ge1}\) with
\(\theta_j\in[-\Delta_j,\Delta_j]\subseteq[-R,R]\), where \(R<\infty\)
is fixed.  We define the predictive energy
\[
  q_j\defeq p_j\Delta_j^2
\]
and assume \(\sum_jq_j<\infty\).  We write
\(Z=(\{(a_t,j_t)\}_{t=1}^K,a_{t_\star},Y)\in\mathcal Z\) for the observed
training block; the source parameters \((\theta_j)_j\) are unknown.  A
regeneration delimiter resets the random variables, making distinct blocks
independent.  Tokens inside a block are
causally ordered, but the statistical sample size \(n\) counts independent
blocks, not tokens.  Bounded logits imply
\(0<\sig(-R)\le\P(Y=1\mid J)\le\sig(R)<1\), so the conditional entropy and
Bayes cross-entropy are strictly positive.

\subsection{Held-out risk and finite learned states}
\label{sec:risk-bits}

For a logit predictor \(f=(f_j)_j\), define held-out cross-entropy
\[
  \mathcal L_\theta(f)
  \defeq
  \E_\theta\!\left[
    -Y\log\sig(f_J)-(1-Y)\log(1-\sig(f_J))
  \right].
\]
The Bayes predictor is \(f_j=\theta_j\); its loss is denoted
\(\Linf\defeq\mathcal L_\theta(\theta)>0\).  Training and evaluation
regeneration blocks are independent.

For an integer \(B\ge0\), a \(B\)-bit learner consists of an encoder
\[
  \phi:\mathcal Z^n\longrightarrow[2^B]
\]
and a fixed decoder \(g:[2^B]\times\N\to[-R,R]\).  The decoder may share
arbitrary computation across modes; only its data-dependent state is limited
to \(2^B\) possibilities.  Randomized learners are allowed, although
conditioning on their random seed shows that they do not improve the Bayes
lower bounds used below.  Because the Bayes loss depends on \(\theta\), we
define minimax \emph{excess} risk directly and use the notation
\[
  \Excess^*(n,B)
  \defeq
  \inf_{\phi,g}\sup_{\theta_j\in[-\Delta_j,\Delta_j]}
  \E_\theta\!\left[
    \mathcal L_\theta(g(\phi(Z^n),\cdot))
    -\mathcal L_\theta(\theta)
  \right].
\]
We use \(x\asymp y\) when positive constants bound their ratio in both
directions.  These constants may depend on \(R\) and on stated
regular-variation constants, but not on \(n\), \(B\), or the number of active
modes.

\subsection{Heterogeneous representation costs}
\label{sec:actions}

The statistical theorem uses \(q_j\).  The structural theorem allows a more
general representable energy \(a_j>0\), a per-bit difficulty \(c_j>0\), and a
numerical amplification \(\kappa_j>0\).  Assigning \(b_j\ge0\) bits to mode
\(j\) incurs distortion \(a_j2^{-2b_j}\) and cost \(c_jb_j\).  Its optimal
distortion under budget \(B\) is
\begin{equation}
  D(B)
  \defeq
  \inf_{\substack{b_j\ge0\\\sum_jc_jb_j\le B}}
  \sum_j a_j2^{-2b_j}.
  \label{eq:weighted-distortion}
\end{equation}
Let \(v_j\ge0\) record a separately specified statistical residual for mode
\(j\). This auxiliary quantity is used in the complete tuple summary below,
not in the spectral formula for \(D(B)\). For a finite mode collection, the unlabeled joint measure
\begin{equation}
  \eta\defeq\sum_j\delta_{(v_j,a_j,c_j,\kappa_j)}
  \label{eq:joint-measure}
\end{equation}
records statistical residual \(v_j\), representable energy, difficulty, and
amplification together.  The cost-weighted difficulty-normalized spectrum is
\begin{equation}
  \mu\defeq\sum_jc_j\,\delta_{a_j/c_j}.
  \label{eq:spectrum}
\end{equation}
For unit costs and \(a_j=q_j\), this reduces to the coordinate spectrum
in \cref{sec:revision-model}; general costs weight each atom by \(c_j\).
We will compare unlabeled summaries of these mode-dependent quantities.  Such
a summary is sufficient for the complete budget--distortion curve if equal summary
values imply equal \(D(B)\) for every \(B\ge0\).

\subsection{Neural and arithmetic resources}
\label{sec:resources}

In \cref{sec:transformer}, \(X\in\R^{(K+1)\times d}\) denotes the matrix of
candidate and query token vectors, and \(U\in\R^{d\times d}\) denotes the
bilinear attention score matrix.  The symbol \(N\) counts physically stored
dense scalar parameter slots.  It is not identified with \(B\), the number of
bits in the data-dependent learned state.

The symbol \(C\) denotes a non-statistical cost budget: representation cost in
\cref{sec:joint-summary} and arithmetic computation in
\cref{sec:bit-complexity,sec:compute-exp}.  These meanings occur in separate
results and are never equated.  When \(C\) denotes computation, its arithmetic
cost model is stated explicitly.  In \cref{sec:compute}, one \(b\)-bit scalar operation costs
\(\Theta(b^\nu)\), where \(\nu\ge0\) is fixed.  In
\cref{sec:bit-complexity}, \(\ell\), \(d\), \(m_{\rm ff}\), \(H\), and \(L\)
denote sequence length, hidden width, feed-forward width, attention-head count,
and layer count.  The symbol \(b_{\rm mm}\) indexes a matrix-arithmetic format
family whose operand, output, and accumulator widths are prescribed functions
of \(b_{\rm mm}\);
\(b_{\rm sm}\) specifies the effective softmax precision; and \(b_{\rm st}\)
is the number of bits used to store an attention probability.  Unit roundoff is denoted \(u\). These variables specify arithmetic work
and storage; \cref{app:discussion} discusses their relation to implementations.

%% file: sections/3_finite_bit.tex
\section{Learning after Data: A Finite-Bit Autoregressive Minimax Law}
\label{sec:finite-bit}

Two irreducible losses govern finite-bit prediction.  We first lower-bound the
distortion of any learned state with at most \(2^B\) possible values, even when
its coordinates are decoded jointly.  We then combine this bound with
statistical uncertainty across regeneration blocks and evaluate the resulting
risk law for power-law energies.  Full proofs appear in
\cref{app:finite-bit}.

\begin{restatable}[Arbitrary joint finite-state representations; informal version of \cref{full:lem:joint-codebook}]{lemma}{LemJointCodebook}
\label{lem:joint-codebook}
Let \(\Theta_j\) be independent and uniform on
\([-\Delta_j,\Delta_j]\), let \(p_j>0\), and put
\(q_j=p_j\Delta_j^2\).  For \(B\in\N_0\), every encoder--decoder pair with at most \(2^B\)
reconstruction vectors satisfies
\begin{equation}
  \sum_jp_j\E(\Theta_j-\widehat\Theta_j)^2
  \ge
  \frac{2}{\pi e}
  \inf_{\substack{b_j\ge0\\\sum_jb_j\le B}}
  \sum_jq_j2^{-2b_j}.
  \label{eq:joint-codebook}
\end{equation}
The reconstruction vectors may couple all coordinates.
\end{restatable}

Here the \emph{codebook} is simply the set of at most \(2^B\) reconstruction
vectors available to the decoder. The entropy proof controls total
codebook cardinality, so it covers shared representations that compress
several modes jointly as well as coordinatewise quantizers.

\begin{restatable}[Finite-bit autoregressive minimax law; informal version of \cref{full:thm:finite-bit}]{theorem}{ThmFiniteBit}
\label{thm:finite-bit}
Fix \(R<\infty\).  For every \(n\in\N\), \(B\in\N_0\), and source family
specified by \((p_j,\Delta_j)_{j\ge1}\) as in \cref{sec:source}, there are
constants \(0<c_R\le C_R<\infty\) such that
\begin{equation}
  c_R\{S_q(n)+D_q(B)\}
  \le
  \Excess^*(n,B)
  \le
  C_R\{S_q(n)+D_q(B)\},
  \label{eq:finite-bit-law}
\end{equation}
where
\begin{equation}
  S_q(n)\defeq\sum_j\min\{q_j,n^{-1}\},
  \qquad
  D_q(B)\defeq
  \inf_{\substack{b_j\ge0\\\sum_jb_j\le B}}
  \sum_jq_j2^{-2b_j}.
  \label{eq:SqDq}
\end{equation}
The lower bound holds for arbitrary \(2^B\)-state joint encoders and arbitrary
fixed shared decoding computation.
\end{restatable}

\paragraph{Why the two terms add.}
On bounded logits, excess Bernoulli cross-entropy is equivalent to squared
logit error.  A clipped modewise likelihood estimate contributes
\(\min\{q_j,1/n\}\): a rare or weak mode is no harder to ignore than to
estimate.  Quantizing the resulting coordinates according to a feasible
\((b_j)\) contributes \(q_j2^{-2b_j}\).  For the converse, an Assouad
hypercube gives the statistical term.  Giving the encoder the true parameter
can only make representation easier; \cref{lem:joint-codebook} then gives the
bit term.  The actual risk exceeds the maximum of these two lower bounds,
which is at least half their sum. This proves the constant-factor
equivalence in \cref{eq:finite-bit-law}.

\begin{restatable}[Power-law energies; informal version of \cref{full:cor:power-profile}]{corollary}{CorPowerProfile}
\label{cor:power-profile}
Suppose that \(n,B\ge1\) and, for some \(\chi>0\) and constants
\(0<c_-\le c_+<\infty\),
\[
  c_-j^{-(1+\chi)}
  \le q_j\le
  c_+j^{-(1+\chi)}
  \qquad(j\ge1).
\]
Then
\begin{equation}
  \Excess^*(n,B)
  \asymp
  n^{-\chi/(1+\chi)}+B^{-\chi}.
  \label{eq:power-pac}
\end{equation}
The data--bit crossover occurs at
\(B\asymp n^{1/(1+\chi)}\).
\end{restatable}

To see the two exponents, split \(S_q(n)\) where \(q_j\) crosses \(1/n\).
For \(D_q(B)\), reverse water filling equalizes
\(q_j2^{-2b_j}\) on active modes.  If \(m\) modes are active, Stirling's
formula gives \(B\asymp m\), while active distortion and the inactive tail
are both \(\asymp m^{-\chi}\).

\paragraph{What the law identifies.}
The theorem identifies the data and retained-bit dependence for the stated
positive-entropy autoregressive source.  Its distinctive role is to couple
held-out log-loss with arbitrary shared decoding; the next structural results
show what additional information determines the resulting exponent.

%% file: sections/5_descriptor_iclr9.tex
\section{Why Marginal Summaries Do Not Determine Scaling}
\label{sec:joint-summary}

This section asks which information about mode difficulty identifies a
scaling law. Here cost is charged directly to selecting or representing a
coordinate. The main-text construction instead realizes difficulty through
the number of independent coordinates within an energy group.
Two problems with identical energy and cost marginals first show
what separate summaries lose.  The joint spectrum then recovers exactly the
complete budget--distortion curve and gives the smallest exact summary for
separable allocation.

\subsection{Same marginals, different exponents}

Consider first a binary selection rule.  Selecting mode \(j\) pays \(c_j\)
and removes energy \(q_j\); unselected energy remains as error.  Under total
representation-cost budget \(C\), its optimum is
\begin{equation}
  \mathcal A(C)
  \defeq
  \inf_{\substack{S\subseteq\N\\\sum_{j\in S}c_j\le C}}
  \sum_{j\notin S}q_j.
  \label{eq:selection-action}
\end{equation}

\begin{restatable}[Marginals do not identify scaling; informal version of \cref{full:thm:marginals-fail}]{theorem}{ThmMarginalsFail}
\label{thm:marginals-fail}
For every \(\chi,\gamma>0\), there are two instances of
\cref{eq:selection-action} with identical energy multisets and identical cost
multisets at every dyadic level, but
\begin{equation}
  \mathcal A_{\rm easy}(C)\asymp C^{-\chi},
  \qquad
  \mathcal A_{\rm hard}(C)\asymp C^{-\chi/(1+\gamma)}.
  \label{eq:coupling-separation}
\end{equation}
The same two exponents hold when selection is replaced by continuous bits
with distortion \(q_j2^{-2b_j}\) and cost \(c_jb_j\).  Hence no summary
that retains only the two marginal multisets identifies the optimal exponent
in either representation class.
\end{restatable}

The construction makes the missing information visible.  At level
\(\ell\), create \(2^\ell\) high-energy modes of size
\(2^{-(1+\chi)\ell}\) and \(2^\ell\) lower-energy modes of size
\(2^{-(1+\chi+\delta)\ell}\), with \(\delta>\gamma\chi\).  The costs are
\(2^\ell\) copies of \(1\) and \(2^\ell\) copies of
\(2^{\gamma\ell}\).  The easy instance pairs high energy with unit cost; the
hard instance pairs it with the larger cost.  All marginals agree level by
level.  A component with multiplicity \(2^\ell\), energy
\(2^{-(1+s)\ell}\), and cost \(2^{g\ell}\) has exponent
\(s/(1+g)\).  This yields \cref{eq:coupling-separation}; the full lower
bound also permits non-prefix selections through a fractional relaxation.

The causal-source coupling experiment is shown in
\cref{fig:mechanisms}(b,e), together with analytic distortion;
\cref{tab:coupling-main} compares theoretical and held-out slopes.

\subsection{The joint spectrum}

The full unlabeled joint measure \(\eta\) in \cref{eq:joint-measure} retains
the pairing by construction.  For the separable exponential response
\cref{eq:weighted-distortion}, it can be compressed further.

\begin{restatable}[Joint energy--difficulty characterization; informal version of \cref{full:thm:joint-spectrum}]{theorem}{ThmJointSpectrum}
\label{thm:joint-spectrum}
For a finite collection of modes, \(\eta\) is a maximal invariant under
permutations of mode labels: it retains all information except the arbitrary
names of the modes.  For a finite or countable collection, consider the
separable problem \cref{eq:weighted-distortion} and assume
\[
  \sum_ja_j<\infty,
  \qquad
  \sum_{j:a_j/c_j>t}c_j<\infty
  \quad\text{for every }t>0.
\]
Then the spectrum \(\mu=\sum_jc_j\delta_{a_j/c_j}\) determines
\(D(B)\) for every \(B\ge0\).  Conversely, the complete curve
\(B\mapsto D(B)\) determines \(\mu\) on \((0,\infty)\).
Consequently, \(\mu\) is a minimal exact unlabeled summary of the complete
budget--distortion curve for this representation class.
\end{restatable}

Put \(r_j=a_j/c_j\).  The coordinatewise optimum has a water level
\(\tau>0\):
\begin{equation}
  b_j(\tau)
  =\frac12\log_2(r_j/\tau)_+.
  \label{eq:waterfill}
\end{equation}
Writing the result in terms of \(\mu\) gives two integral identities,
\begin{equation}
  B(\tau)
  =\frac{1}{2\log2}\int\log(r/\tau)_+\,\mu(\dd r),
  \qquad
  D(B(\tau))
  =\int\min\{r,\tau\}\,\mu(\dd r).
  \label{eq:spectrum-integrals}
\end{equation}
These identities prove sufficiency.  They also make the inverse result
possible.  At differentiability points the envelope theorem gives
\begin{equation}
  D'(B)=-2\log2\,\tau(B).
  \label{eq:envelope}
\end{equation}
Thus the curve recovers the water level, hence \(B(\tau)\).  With
\(s=\log\tau\),
\begin{equation}
  \frac{\dd}{\dd s}B(e^s)
  =-\frac{1}{2\log2}\mu((e^s,\infty))
  \label{eq:tail-inversion}
\end{equation}
at every continuity point.  The right-continuous tail, including its jumps,
determines \(\mu\).

Minimality now has a precise meaning.  If another summary \(T\) determines
the complete curve, then \(T(x)=T(y)\) implies \(D_x=D_y\), which by
\cref{eq:tail-inversion} implies \(\mu_x=\mu_y\).  Therefore \(\mu\) is a
function of every sufficient summary.  It is also strictly coarser than
the full unlabeled table \(\eta\): it retains only the cost-weighted
distribution of energy per unit difficulty.

\begin{corollary}[Joint-tail exponent]
\label{cor:joint-tail}
If \(a_j\asymp j^{-\alpha}\) with \(\alpha>1\) and
\(c_j\asymp j^\gamma\) with \(\gamma>-1\), then
\[
  D(B)\asymp B^{-(\alpha-1)/(1+\gamma)}.
\]
\end{corollary}

Indeed, the active threshold \(m\) obeys \(B\asymp m^{1+\gamma}\),
while the represented distortion and omitted tail are both
\(\asymp m^{1-\alpha}\).

\paragraph{When the ingredients can be measured.}
Controlled interventions provide a route to estimating the ingredients:
counts estimate mode frequency, held-out ablations estimate predictive
energy, and deliberate rounding perturbations estimate numerical
amplification. The identification theorem specifies how these quantities
determine the resource curve; extracting them from passive
natural-language observations is an open estimation problem.

%% file: sections/9_appendix.tex
\section{Proofs for the Finite-Bit Autoregressive Minimax Law}
\label{app:finite-bit}

This appendix proves the arbitrary-codebook entropy lemma, the statistical
upper and lower bounds, and the power-law corollary in that order.

\FullLemJointCodebook
\begin{proof}
First restrict to a finite coordinate set \(A\).  Write
\[
  d_j\defeq p_j\E(\Theta_j-\widehat\Theta_j)^2.
\]
For any \(S\subseteq A\), data processing and the codebook cardinality imply
\begin{equation}
  B\log2\ge I(\Theta_S;\widehat\Theta).
  \label{eq:app-data-processing}
\end{equation}
Independence gives \(h(\Theta_S)=\sum_{j\in S}\log(2\Delta_j)\).
Moreover,
\[
\begin{aligned}
  h(\Theta_S\mid\widehat\Theta)
  &=
  h(\Theta_S-\widehat\Theta_S\mid\widehat\Theta)\\
  &\le h(\Theta_S-\widehat\Theta_S)\\
  &\le\sum_{j\in S}h(\Theta_j-\widehat\Theta_j)\\
  &\le\frac12\sum_{j\in S}
  \log\!\left(2\pi e\,\frac{d_j}{p_j}\right),
\end{aligned}
\]
where the last line is the Gaussian maximum-entropy inequality at fixed
second moment.  To see it directly, compare the error density to the centered
Gaussian of variance \(d_j/p_j\): nonnegativity of relative entropy gives
\(h(X)\le\frac12\log(2\pi d_j/p_j)+\E X^2/(2d_j/p_j)\).
If an entropy is \(-\infty\), the same bound follows by approximation.
Zero mean-square error on a nondegenerate uniform coordinate would require
infinite information, so it cannot occur for a finite codebook.
All information and entropy expressions here use natural logarithms.
Therefore
\begin{equation}
  I(\Theta_S;\widehat\Theta)
  \ge
  \frac12\sum_{j\in S}
  \log\frac{2q_j}{\pi e\,d_j}.
  \label{eq:app-entropy-lower}
\end{equation}

Choose \(S=\{j:d_j<2q_j/(\pi e)\}\) and define
\[
  b_j=
  \frac{1}{2\log2}
  \left[\log\frac{2q_j}{\pi e\,d_j}\right]_+.
\]
Equations \eqref{eq:app-data-processing}--\eqref{eq:app-entropy-lower}
give \(\sum_jb_j\le B\).  On \(S\),
\[
  d_j=\frac{2}{\pi e}q_j2^{-2b_j},
\]
while off \(S\), \(b_j=0\) and
\(d_j\ge(2/(\pi e))q_j\).  Hence
\[
  \sum_{j\in A}d_j
  \ge
  \frac{2}{\pi e}
  \inf_{\substack{b_j\ge0\\\sum_{j\in A}b_j\le B}}
  \sum_{j\in A}q_j2^{-2b_j}.
\]

For countably many coordinates, let \(A_m\uparrow\N\) and call the finite
variational value \(D_{A_m}(B)\).  Clearly
\(D_{A_m}(B)\le D_q(B)\).  Extending an optimizer on \(A_m\) by zeros gives
\[
  D_q(B)\le D_{A_m}(B)+\sum_{j\notin A_m}q_j.
\]
Since \(\sum_jq_j<\infty\), \(D_{A_m}(B)\uparrow D_q(B)\).  Monotone
convergence of the distortion completes the proof.  Randomized encoders do
not improve the Bayes risk: condition on the encoder seed and average the
deterministic bound.
\end{proof}

\FullThmFiniteBit
\begin{proof}
We proceed in four steps.

\paragraph{Step 1: loss curvature.}
The Bernoulli log-partition function \(A(t)=\log(1+e^t)\) satisfies
\[
  A''(t)=\sig(t)(1-\sig(t)).
\]
On \([-R,R]\), this derivative lies between two positive constants.
The excess cross-entropy of logit \(\widehat\theta\) is the Bregman
divergence of \(A\), or equivalently the Kullback--Leibler (KL) divergence,
which we denote by \(\KL\).  Taylor's theorem therefore yields
\begin{equation}
  c_R(\theta-\widehat\theta)^2
  \le
  \KL(\Ber(\sig(\theta))\,\|\,\Ber(\sig(\widehat\theta)))
  \le
  C_R(\theta-\widehat\theta)^2.
  \label{eq:app-curvature}
\end{equation}
Clipping a decoded logit to \([-R,R]\) also cannot increase cross-entropy:
the derivative with respect to the decoded logit \(z\) is
\(\sig(z)-\sig(\theta)\), which has the sign of \(z-\theta\).
Clipping preserves the decoder's cardinality.  Thus both minimax bounds
can be proved using bounded decoded logits and weighted squared error.

\paragraph{Step 2: statistical upper bound.}
Let \(N_j\sim\operatorname{Binomial}(n,p_j)\) be the training count of mode
\(j\).  If \(N_j\ge1\), project the empirical Bernoulli mean onto
\([\sig(-\Delta_j),\sig(\Delta_j)]\) and apply the logit map.  Projection
cannot increase probability error, and the logit is Lipschitz on
\([\sig(-R),\sig(R)]\).  Thus its conditional mean-square error is at most
\(C_R/N_j\).  The zero estimator has error at most \(\Delta_j^2\).  Choose
between them using the observed count, yielding
\begin{equation}
  \E[(\widehat\theta_j-\theta_j)^2\mid N_j]
  \le
  C_R'\min\{\Delta_j^2,(N_j+1)^{-1}\}.
  \label{eq:app-count-bound}
\end{equation}
The identity
\begin{equation}
  \E\frac1{N_j+1}
  =
  \frac{1-(1-p_j)^{n+1}}{(n+1)p_j}
  \le\frac1{(n+1)p_j}
  \label{eq:app-binomial}
\end{equation}
follows by integrating the binomial generating function.  Therefore
\[
\begin{aligned}
  p_j\E(\widehat\theta_j-\theta_j)^2
  &\le
  C_R'
  \min\left\{
    p_j\Delta_j^2,\,
    p_j\E\frac1{N_j+1}
  \right\}\\
  &\le C_R''\min\{q_j,n^{-1}\}.
\end{aligned}
\]
Summing proves the \(S_q(n)\) upper bound.

Now select an allocation within an arbitrarily small additive error of
\(D_q(B)\), round each positive \(b_j\) down to an integer, and uniformly
quantize the clipped estimate on \([-\Delta_j,\Delta_j]\).  Flooring uses at
most \(B\) bits and enlarges \(2^{-2b_j}\) by at most four.  By the squared
triangle inequality and \cref{eq:app-curvature},
\begin{equation}
  \Excess^*(n,B)
  \le C_R\{S_q(n)+D_q(B)\}.
  \label{eq:app-upper}
\end{equation}

\paragraph{Step 3: statistical lower bound.}
Fix a finite coordinate set \(A\).  On coordinate \(j\in A\), use the two
parameters \(\pm\delta_j\), where
\[
  \delta_j
  =c_R'\min\{\Delta_j,(np_j)^{-1/2}\}
\]
and \(c_R'>0\) is sufficiently small.  Neighboring hypercube vertices differ
only when \(J=j\), and bounded Bernoulli curvature gives
\[
  \KL(P_{\theta}^{\,n}\,\|\,P_{\theta^{(j)}}^{\,n})
  \le C_Rnp_j\delta_j^2\le C_R(c_R')^2.
\]
Here the design distribution is independent of the hypercube signs.  The
one-block KL is therefore exactly \(p_j\) times the changed coordinate's
Bernoulli KL, and independence of blocks multiplies it by \(n\).
Choose \(c_R'\) so that this product KL is at most \(1/8\).
For completeness, fix all signs except \(j\) and test the remaining sign by
the sign of \(\widehat\theta_j\).  A wrong sign incurs squared error at
least \(\delta_j^2\).  The sum of the two testing errors is at least
\(1-\operatorname{TV}(P_+,P_-)\ge3/4\), since
\(\operatorname{TV}(P_+,P_-)\le\sqrt{\KL(P_+\|P_-)/2}\).
This inequality follows by partitioning at the positive part of the density
difference, applying the log-sum inequality, and using
\(\mathrm{kl}(a,b)\ge2(a-b)^2\); the latter follows from its second
derivative in \(a\), \(1/[a(1-a)]\ge4\).
Averaging over the uniform signs and summing the coordinate losses yields
\begin{equation}
  \inf_{\widehat\theta}\sup_\theta
  \E\sum_{j\in A}p_j(\widehat\theta_j-\theta_j)^2
  \ge
  c_R\sum_{j\in A}\min\{q_j,n^{-1}\}.
  \label{eq:app-assouad}
\end{equation}
The resulting constant is independent of the size of \(A\).  Since the
full risk dominates each finite-coordinate risk, taking the supremum over
increasing finite sets gives the \(S_q(n)\) lower bound; no interchange of
an infinite supremum and a prior integral is needed.

\paragraph{Step 4: representation lower bound.}
Give the encoder \(\theta\) in addition to the data.  This enlarges the
encoder class and can only reduce its minimax risk.  Under the
product-uniform prior on the parameter box, \cref{lem:joint-codebook} gives
\begin{equation}
  \inf_{\abs{\operatorname{range}(\widehat\theta)}\le2^B}
  \sup_\theta
  \E\sum_jp_j(\widehat\theta_j-\theta_j)^2
  \ge cD_q(B).
  \label{eq:app-representation-lower}
\end{equation}
The original constrained risk is at least both
\cref{eq:app-assouad,eq:app-representation-lower}.  Since
\(\max\{x,y\}\ge(x+y)/2\), \cref{eq:app-curvature} proves the matching
lower bound in \cref{eq:finite-bit-law}.
\end{proof}

\FullCorPowerProfile
\begin{proof}
The comparison assumption does not require the actual sequence \(q_j\) to
be monotone.  Put \(\bar q_j=j^{-(1+\chi)}\), so that
\(c_-\bar q_j\le q_j\le c_+\bar q_j\).
Taking infima of the coordinatewise distortion inequalities gives
\[
c_-D_{\bar q}(B)\le D_q(B)\le c_+D_{\bar q}(B).
\]
Similarly, \(\min\{cq,t\}\) lies between
\(\min\{c,1\}\min\{q,t\}\) and \(\max\{c,1\}\min\{q,t\}\).
It suffices to prove both rates for \(\bar q\).
Let \(m_n=\lfloor n^{1/(1+\chi)}\rfloor\).  Then
\[
  S_{\bar q}(n)
  \asymp
  \frac{m_n}{n}+\sum_{j>m_n}j^{-(1+\chi)}
  \asymp n^{-\chi/(1+\chi)}.
\]

For \(D_{\bar q}(B)\), the water-filling solution gives
\[
  b_j=\frac12[\log_2(\bar q_j/\tau)]_+.
\]
The comparison profile is monotone, so its active set is a prefix \([m]\).
For \(m\ge1\), \((m+1)^{-(1+\chi)}\le\tau\le m^{-(1+\chi)}\).
Writing \(v=\tau^{-1/(1+\chi)}\in[m,m+1]\),
\[
  B
  =
  \frac{1+\chi}{2\ln2}\{m\ln v-\ln(m!)\}
  =
  \Theta(m)
\]
as \(m\to\infty\), by Stirling's formula and
\(0\le m\ln(v/m)\le1\).  Active distortion is \(m\tau\asymp m^{-\chi}\),
and the inactive tail has the same order.  Thus \(D_q(B)\asymp B^{-\chi}\).
The two terms cross when
\(B^{-\chi}\asymp n^{-\chi/(1+\chi)}\), namely
\(B\asymp n^{1/(1+\chi)}\).
\end{proof}

\section{Proof of Adaptation with a Known Ordering}
\label{app:adaptive}

This appendix verifies that one rounded schedule works uniformly over the
specified exponent interval.

\FullThmAdaptive
\begin{proof}
The estimator in the statistical upper bound of \cref{thm:finite-bit} uses
only \(N_j\), \(\Delta_j\), and the observed labels, not \(\chi\).  It
contributes \(S_q(n)\asymp n^{-\chi/(1+\chi)}\).

Choose \(A>\chi_{\max}/2\).  Before rounding, the schedule in
\cref{eq:universal-schedule} uses
\[
  A\sum_{j\le m}\log_2(m/j)
  =
  A\log_2(m^m/m!)
  =
  \Theta(m).
\]
Flooring cannot increase the budget.  It also preserves a linear lower
bound: every \(j\le m2^{-2/A}\) receives at least two unrounded bits and
therefore at least one rounded bit.  Hence the rounded total is
\(\Theta(m)\), and the maximal feasible \(m\) satisfies \(m\asymp B\).

Flooring enlarges each nonzero distortion by at most four.  Consequently,
uniformly over \(\chi\le\chi_{\max}\),
\[
\begin{aligned}
  \sum_{j\le m}q_j2^{-2b_j}
  &\lesssim
  m^{-2A}\sum_{j\le m}j^{2A-1-\chi}\\
  &\lesssim m^{-\chi},
\end{aligned}
\]
because \(2A-\chi\ge2A-\chi_{\max}>0\): integral comparison bounds the
power sum by \(C m^{2A-\chi}/(2A-\chi)\), with one constant throughout the
interval.  Likewise the tail constants are uniform because
\(\chi\ge\chi_{\min}>0\).  For the finitely many budgets before the
linear bound on the rounded schedule applies, enlarge the same constant;
\(B\in\N\) is positive.  The omitted tail is
\(\sum_{j>m}q_j\asymp m^{-\chi}\).  Since \(m\asymp B\), the
representation contribution is \(O(B^{-\chi})\), proving
\cref{eq:adaptive-rate}.
\end{proof}

\section{Proofs for Joint Difficulty}
\label{app:joint-summary}

This appendix proves the same-marginal separation and then the positive
joint-summary theorem.

\FullThmMarginalsFail
\begin{proof}
Fix \(\delta>\gamma\chi\).  At each level \(\ell\ge1\), create two groups,
each of multiplicity \(2^\ell\).  Their energies are
\[
  q_\ell^{\rm H}=2^{-(1+\chi)\ell},
  \qquad
  q_\ell^{\rm L}=2^{-(1+\chi+\delta)\ell},
\]
and the level's cost multiset contains \(2^\ell\) copies of \(1\) and
\(2^\ell\) copies of \(2^{\gamma\ell}\).  The easy instance pairs
\(q_\ell^{\rm H}\) with unit cost; the hard instance pairs it with
\(2^{\gamma\ell}\).  The energy and cost multisets therefore agree at every
level.

Consider a generic component of multiplicity \(2^\ell\), per-item energy
\(2^{-(1+s)\ell}\), and per-item cost \(2^{g\ell}\).  Buying the component
through level \(L\) costs
\[
  \sum_{\ell\le L}2^{(1+g)\ell}
  =\Theta(2^{(1+g)L}),
\]
and leaves energy
\[
  \sum_{\ell>L}2^{-s\ell}
  =\Theta(2^{-sL}).
\]
Thus its selection exponent is \(s/(1+g)\).  This prefix is also a matching
lower bound: within the component, energy per cost decreases with \(\ell\),
so the fractional relaxation is a threshold policy, and every integral
policy is no better.

In the easy instance, the high-energy component has exponent \(\chi\), while
the low-energy component has exponent
\((\chi+\delta)/(1+\gamma)>\chi\).  Splitting the budget by fixed fractions
therefore gives, and the high-energy relaxation lower-bounds,
\(\mathcal A_{\rm easy}(C)\asymp C^{-\chi}\).  In the hard instance, the
high-energy component has exponent \(\chi/(1+\gamma)\), while the low-energy
unit-cost component has exponent \(\chi+\delta\).  Hence
\(\mathcal A_{\rm hard}(C)\asymp C^{-\chi/(1+\gamma)}\).

For continuous bits, the optimality condition on an active coordinate is
\[
  \frac{q_j2^{-2b_j}}{c_j}=\tau.
\]
For the generic component above, the active levels form a prefix through
\(L\).  At level \(\ell\), \(b_\ell\) is proportional to \(L-\ell\);
therefore
\[
  \sum_{\ell\le L}2^{(1+g)\ell}b_\ell
  =\Theta(2^{(1+g)L}).
\]
Both active distortion and the inactive tail are
\(\Theta(2^{-sL})\).  The same exponent \(s/(1+g)\) follows, and applying it
to the four components completes the proof.
\end{proof}

\FullThmJointSpectrum
\begin{proof}
For a finite collection, equality of
\(\eta=\sum_j\delta_{(v_j,a_j,c_j,\kappa_j)}\) is equality of the multisets
of tuples, including multiplicities.  A bijection between equal multisets is
a permutation of mode labels, and every permutation preserves \(\eta\).
Thus \(\eta\) is a maximal invariant.

For the separable allocation, put \(r_j=a_j/c_j\).  The Lagrangian is
\[
  \sum_jc_jr_j2^{-2b_j}
  +\lambda\left(\sum_jc_jb_j-B\right).
\]
The coordinatewise Karush--Kuhn--Tucker equations give
\[
  b_j(\tau)
  =
  \frac12\log_2(r_j/\tau)_+,
  \qquad
  \tau=\frac{\lambda}{2\log2}.
\]
Substitution yields
\[
  B(\tau)
  =
  \frac{1}{2\log2}\sum_jc_j\log(r_j/\tau)_+,
\]
and
\[
  D(B(\tau))
  =
  \sum_jc_j\min\{r_j,\tau\}.
\]
Here is a direct countable-dimensional justification.  Set
\(A=\sum_j a_j<\infty\).  For every \(\tau>0\),
\[
\sum_{r_j>\tau}c_j\le A/\tau,
\qquad
B(\tau)\le\frac{A}{2(\ln2)\tau},
\]
where the second inequality uses \([\ln x]_+\le x\).
On any compact interval of positive water levels, the summands defining
\(B(\tau)\) have a summable dominating bound proportional to \(a_j\).
Hence \(B\) is continuous.  It is strictly decreasing whenever it is
positive, tends to zero at the upper end of its support (possibly infinity),
and tends to infinity as \(\tau\downarrow0\), since even one positive
coordinate does.  Thus every \(B>0\) has a unique water level.
The scalar convex inequality
\[
a_j2^{-2b_j}+\lambda c_jb_j
\ge a_j2^{-2b_j(\tau)}+\lambda c_jb_j(\tau)
\]
can be summed directly: all distortion and feasible cost sums are finite.
With \(\lambda=2(\ln2)\tau\) and \(B(\tau)=B\), it proves optimality
over every feasible countable allocation.  Layer integration now gives
\cref{eq:spectrum-integrals}, proving sufficiency of \(\mu\).

To recover the spectrum, let \(0<B_1<B_2\) and put
\(\lambda_i=2(\ln2)\tau(B_i)\).  Applying optimality with each of the two
Lagrange multipliers to the other's solution gives
\[
-\lambda_1\le
\frac{D(B_2)-D(B_1)}{B_2-B_1}
\le-\lambda_2.
\]
The inverse \(\tau(B)\) is continuous on \((0,\infty)\).  Letting the two
budgets coalesce therefore proves, in the countable system itself,
\[
  D'(B)=-\lambda=-2\log2\,\tau(B)
\]
at every \(B>0\).  Hence the complete curve determines \(\tau(B)\) and its
inverse \(B(\tau)\), with the zero branch above a finite support endpoint.
For \(s=\ln\tau\), difference quotients have a summable dominating bound:
only \(r_j>e^{s-h}\) contribute, whose total cost is finite.
Thus the right derivative is
\[
  \frac{\dd_+}{\dd s}B(e^s)
  =
  -\frac{1}{2\log2}
  \sum_{j:r_j>e^s}c_j
  =
  -\frac{1}{2\log2}\mu((e^s,\infty))
\]
and the left derivative replaces \(r_j>e^s\) by \(r_j\ge e^s\).
Their difference recovers every atom's weight; the tail recovers the measure
on every positive interval.  The endpoint \(D(0)=A\) follows by dominated
convergence from \(B\downarrow0\).  This argument never differentiates a
limit of finite-dimensional value functions.

Finally, suppose a summary \(T\) determines the complete curve.  If
\(T(x)=T(y)\), then \(D_x(B)=D_y(B)\) for every \(B\), so the inversion above
gives \(\mu_x=\mu_y\).  Therefore \(\mu\) is a function of \(T\), proving
minimality in the stated sense.
\end{proof}

\begin{proof}[Proof of \cref{cor:joint-tail}]
Put \(\bar a_j=j^{-\alpha}\) and \(\bar c_j=j^\gamma\), with
\(a_-\bar a_j\le a_j\le a_+\bar a_j\) and
\(c_-\bar c_j\le c_j\le c_+\bar c_j\).  Feasible-set inclusion gives
\[
a_-D_{\bar a,\bar c}(B/c_-)\le D_{a,c}(B)
\le a_+D_{\bar a,\bar c}(B/c_+).
\]
It suffices to use the monotone ratio
\(\bar a_j/\bar c_j=j^{-(\alpha+\gamma)}\).
Since \(\alpha>1\) and \(\gamma>-1\), its exponent is positive.
If \(m\) modes are active, put
\(v=\tau^{-1/(\alpha+\gamma)}\in[m,m+1]\).  The budget is
\[
  B=\frac{\alpha+\gamma}{2\ln2}
  \sum_{j\le m}j^\gamma\ln(v/j)
  \asymp m^{1+\gamma}.
\]
For the lower bound restrict to \(m/4\le j\le m/2\); for the upper bound
use integral comparison and
\(\int_0^1x^\gamma\ln(1/x)\,\mathrm dx=(1+\gamma)^{-2}\).
Replacing \(m\) by \(v\) in the logarithm adds only \(O(m^\gamma)\).
The inactive tail is \(\asymp m^{1-\alpha}\).  The active distortion is
\(\tau\sum_{j\le m}c_j\), with
\(\tau\asymp m^{-(\alpha+\gamma)}\), and has the same order.
Eliminating \(m\) gives
\[
  D(B)\asymp B^{-(\alpha-1)/(1+\gamma)}.
\]
\end{proof}

\section{Proof of the Finite-Precision Realization}
\label{app:transformer}

This appendix verifies the exact matrix map, the floating-point perturbation,
and the lower bound for the stated separable representation.

We use the candidate/query encoding and projections of
\cref{sec:transformer,eq:attention-map}. The supplied logit table is
\(\vartheta_j=w_j^{\rm mode}\) for \(j\le M\), extended by zero outside
the represented prefix; \(\theta_j\) remains the source parameter.

\FullThmTransformer
\begin{proof}
For the score matrix \(U=W_Q^\top W_K/\sqrt{d_h}\) of the stated
construction,
\[
  x_q^\top Ux_t
  =
  \begin{cases}
    \beta,&t=t_\star,\\
    0,&t\ne t_\star.
  \end{cases}
\]
Softmax therefore gives \cref{eq:attention-map}.  Since
\(\sum_{t\ne t_\star}\alpha_t=1-\alpha_\star\) and
\(\abs{\vartheta_j}\le R\),
\[
\begin{aligned}
  \abs{z-\vartheta_{j_\star}}
  &=
  \abs{
    \sum_{t\ne t_\star}\alpha_t
    (\vartheta_{j_t}-\vartheta_{j_\star})
  }\\
  &\le2R(1-\alpha_\star)\\
  &=
  2R\frac{K-1}{e^\beta+K-1}\\
  &\le2R(K-1)e^{-\beta}.
\end{aligned}
\]

Uniformly quantizing \([-\Delta_j,\Delta_j]\) with \(2^{b_j}\) values gives
\(\abs{\widetilde w_j^{\rm mode}-w_j^{\rm mode}}\le2\Delta_j2^{-b_j}\)
for \(j\in[M]\).  Write \(\widetilde\vartheta_j=\widetilde w_j^{\rm mode}\)
on \([M]\) and \(\widetilde\vartheta_j=0\) otherwise.  Distractor
quantization contributes at most
\(2R(1-\alpha_\star)\), so it is absorbed by the routing term.

It remains to justify the arithmetic term.  Stable max shifting replaces the
scores \(s_t\in\{0,\beta\}\) by
\(\bar s_t=s_t-\beta\in\{-\beta,0\}\) without changing the exact softmax.
Let \(y_t=e^{\bar s_t}\) and \(h=\ceil{\log_2K}\).  By assumption, the computed
unnormalized weight is
\[
  \widehat y_t=y_t(1+\delta_t),
  \qquad\abs{\delta_t}\le u.
\]
Balanced positive summation gives
\[
  \widehat S
  =
  \left(\sum_t\widehat y_t\right)(1+\delta_S),
  \qquad
  \abs{\delta_S}\le\gamma_h
  \defeq\frac{hu}{1-hu}.
\]
Including the final division, the relative-error factors can be bounded
explicitly.  Put \(\bar\delta=\sum_t\alpha_t\delta_t\), so
\(|\bar\delta|\le u\), and let \(|\delta_{D,t}|\le u\) be the division
error.  Then
\[
\frac{\widehat\alpha_t}{\alpha_t}
=\frac{(1+\delta_t)(1+\delta_{D,t})}
{(1+\bar\delta)(1+\delta_S)}.
\]
When \((h+4)u\le1/4\), we have \(\gamma_h\le4hu/3\), and both denominator
factors are bounded away from zero.  Subtracting numerator and denominator
and bounding the linear and quadratic terms gives
\(|\widehat\alpha_t/\alpha_t-1|\le8(h+4)u\).
Consequently
\[
  \widehat\alpha_t
  =
  \alpha_t(1+\varepsilon_t),
  \qquad
  \abs{\varepsilon_t}\le8(h+4)u.
\]
Summing against \(\alpha_t\) proves
\cref{eq:softmax-forward-error}.

For the final dot product, each multiplication has relative error at most
\(u\), and balanced signed summation has absolute error at most a
\(\gamma_h\) multiple of the sum of magnitudes.  Since
\[
  \sum_t\abs{\widetilde\vartheta_{j_t}}\widehat\alpha_t
  \le R\{1+O(u(1+\log K))\},
\]
we obtain \cref{eq:output-forward-error}.  For every matching mode, represented
or not, combining the contributions gives
\[
  \abs{\widehat z-\theta_{j_\star}}
  \le
  4R(K-1)e^{-\beta}
  +2\Delta_{j_\star}2^{-b_{j_\star}}\mathbf1\{j_\star\le M\}
  +|\theta_{j_\star}|\mathbf1\{j_\star>M\}
  +C_Ru(1+\log K).
\]
The two mode-specific terms are mutually exclusive.  The squared sum is
therefore bounded by a constant times routing error, arithmetic error, and
their squared mode-specific sum.  Averaging over the selected coordinate
adds \(\sum_{j\le M}q_j2^{-2b_j}+\sum_{j>M}q_j\).
An unrepresented matching mode can still attend to represented distractors;
that contribution is already bounded by the routing term above, rather than
being assumed zero.  Finally \(A''(z)\le1/4\) for every real logit, so the
upper KL curvature bound applies even if floating-point error moves
\(\widehat z\) slightly outside the original box.  This proves
\cref{eq:fp-bound}.

For the converse, restrict to the separable scalar table.  Coordinate \(j\)
has at most \(2^{b_j}\) reconstruction values.  Under the same product-uniform
prior on the coordinate logits, the one-dimensional specialization of
\cref{lem:joint-codebook} gives Bayes weighted squared error at least
\(cq_j2^{-2b_j}\) on each coordinate.  Expectations add under this one
common prior, so the sum lower-bounds worst-case scalar-table distortion;
it is not a pointwise assertion about every parameter.  Apply the lower
curvature bound to the clipped table decoder to obtain the stated tightness
within that representation.  The construction stores
one \(d\times d\) score matrix, one \(d\times d\) value matrix, and one
\(d\)-coordinate head, hence \(N=2d^2+d\).
\end{proof}

\section{Proof of the Transformer Bit-Complexity Lemma}
\label{app:bit-complexity}

This appendix counts the dense forward--backward graph stated in
\cref{sec:fb-count}. Its sequence length, hidden width, feed-forward width,
head count, and depth are \(\ell,d,m_{\rm ff},H,L\), respectively. The
precision indices specify the operand, accumulator, and stored-activation
formats whose costs are counted.

\FullLemBitComplexity
\begin{proof}
Suppressing constant head reshapes, the four query, key, value, and output
projections perform \(4\ell d^2\) multiply--accumulates per layer.  Let
\(Q_{\rm mat}\), \(K_{\rm mat}\), and \(V_{\rm mat}\) denote the query, key,
and value matrices, and let \(P\) denote the matrix of attention probabilities.
The query--key score product \(Q_{\rm mat}K_{\rm mat}^\top\) and
probability--value product \(PV_{\rm mat}\) each perform
\(\ell^2d\), and the two feed-forward maps perform
\(2\ell dm_{\rm ff}\).  Thus the forward matrix count is
\begin{equation}
  F_{\rm mat}
  =4\ell d^2+2\ell^2d+2\ell dm_{\rm ff}.
  \label{eq:app-forward-matmul}
\end{equation}
For any product \(Y=AB\), reverse mode forms
\(\nabla_A=\nabla_YB^\top\) and
\(\nabla_B=A^\top\nabla_Y\).  Each has the same scalar
multiply--accumulate count as the forward product.  Consequently the matrix
part of one forward--backward pass is exactly three times
\cref{eq:app-forward-matmul}, up to additions, transposes, and casts.

There are \(H\ell\) softmax rows of length \(\ell\).  Stable max shifting,
exponentiation, reduction, division, and the reverse-mode row calculation
together cost \(\Theta(H\ell^2\mathsf S(b_{\rm sm}))\) by the definition of
\(\mathsf S\).  Given an incoming row gradient
\(v\), the softmax Jacobian is applied by
\[
  p\odot\bigl(v-\inner{p}{v}\mathbf 1\bigr),
\]
which uses a constant number of length-\(\ell\) passes.  No explicit
\(\ell\times\ell\) Jacobian is formed.  Multiplying by \(L\) and absorbing fixed projection
constants proves \cref{eq:fb-bitops}.  Normalization, pointwise activation,
dropout, and casts have only \(O(L\ell(d+m_{\rm ff})+LH\ell^2)\) scalar
entries and are the stated lower-order terms for this computation.

For memory, the dense projection and feed-forward weights contain
\(\Theta(L(d^2+dm_{\rm ff}))\) scalars.  Saved residual, projection, and
feed-forward activations contain
\(\Theta(L\ell(d+m_{\rm ff}))\) scalars.  Eager reverse mode also saves
\(LH\ell^2\) attention probabilities.  Multiplying each collection by its
aggregate bits per logical slot (including same-shaped peak gradient or
workspace buffers) and adding separately retained high-precision state
gives \cref{eq:fb-memory}.  A tiled online softmax stores row normalization
statistics and tiles rather than the complete probability matrix, so its
persistent attention state is linear in \(H\ell\); recomputing tiles during
backward does not change the dense arithmetic order.  These counts also give
the two-sided bounds because each named tensor and matrix product is present
in the specified graph.
\end{proof}

\section{Proof of the Compute--Precision Pareto Law}
\label{app:precision-tradeoff}

The proof distinguishes the precision required to match a target error from
the exact precision chosen by a finite-budget optimizer.

\FullThmPrecisionTradeoff
\begin{proof}
The matrix rounding term is at most a constant multiple of the model term
exactly when
\[
  M^{2r}2^{-2b}\lesssim M^{-\chi}.
\]
Taking base-two logarithms proves \cref{eq:matching-bits}; the two-sided risk
assumption gives necessity for matching the model-error scale.
For the dense graph, positive head dimensions sum to \(d\), so
\[
S(M)\asymp LH\ell^2\le Ld\ell^2\lesssim G(M).
\]
Hence \(g\ge s\).  The general-count branch uses
\cref{eq:mixed-risk-relations,eq:mixed-cost-relations} with independent
nonnegative exponents \(g,s\).

Suppose first that the matrix term dominates the bracket in
\cref{eq:mixed-cost-relations}.  Choose both numerical errors below a constant
multiple of \(M^{-\chi}\).  By \cref{eq:matching-bits},
\[
  C\asymp M^{1+\chi+g}(\log M)^\nu.
\]
Let \(\zeta=1+\chi+g\).  Substitution into
\(M^\zeta(\log M)^\nu\) gives cost within constant factors of \(C\) at
\[
  M\asymp C^{1/\zeta}(\log C)^{-\nu/\zeta}.
\]
Substituting into \(M^{-\chi}\) proves the upper bound in
\cref{eq:matrix-dominated-phase}.  For the converse, consider any feasible
configuration with risk at most \(\varepsilon\).  The risk relation gives
\[
M\gtrsim\varepsilon^{-1/\chi},\qquad
2^{-2b}\lesssim\varepsilon,\qquad
b\gtrsim\log(1/\varepsilon),
\]
using \(r\ge0\).  The matrix contribution alone forces
\[
C\gtrsim\varepsilon^{-\zeta/\chi}[\log(1/\varepsilon)]^\nu.
\]
Inverting proves the lower rate for every configuration.  For \(\nu=0\),
the logarithmic factor equals one and the same argument applies.

Now use general operation counts with \(s>g\), fixed \(b_{\rm sm}\), and
dominating softmax work
before its numerical floor.  Matching the matrix error still costs only
\(b=\Theta(\log M)\), but the leading cost is
\[
  C\asymp M^{1+\chi+s}b_{\rm sm}^{\nu_{\rm sm}}
  \asymp M^{1+\chi+s}.
\]
Within the range where
\(M^{2r_{\rm sm}}2^{-2b_{\rm sm}}\lesssim M^{-\chi}\), substitution proves
\cref{eq:softmax-dominated-phase}.  Conversely, risk at most \(\varepsilon\)
requires \(M\gtrsim\varepsilon^{-1/\chi}\), and the softmax cost is at
least a constant times \(M^{1+\chi+s}\).  Thus
\(C\gtrsim\varepsilon^{-(1+\chi+s)/\chi}\).
Since the leading term contains no \(b\), changing matrix
precision cannot change this exponent.  Equivalently, substituting
\(M\asymp C^{1/(1+\chi+s)}\) gives exactly the pre-floor budget condition in
the theorem statement.

Finally, at fixed \(b\), the model and matrix-rounding terms cross when
\[
  M^{-\chi}\asymp M^{2r}2^{-2b},
  \qquad
  M^{\chi+2r}\asymp2^{2b}.
\]
This gives \(M_b\) in \cref{eq:fixed-bit-floor}.  Choose softmax precision
so its error is at most the other terms.  If \(M\le M_b\), the model term
is at least a constant times \(M_b^{-\chi}\); if \(M\ge M_b\), the
matrix-rounding term is, since \(r\ge0\).  At \(M_b\) both have this order.
When \(r=0\), the comparison expression \(M^{-\chi}+2^{-2b}\) decreases
to \(2^{-2b}\); its infimum has the displayed order and \(M_b\) is the
crossover scale.  Replacing \(r,b\) by
\(r_{\rm sm},b_{\rm sm}\) proves \cref{eq:softmax-horizon}.  The fixed-share
ratios in \cref{eq:bitops-ratio,eq:memory-ratio} follow by dividing the
baseline cost or storage by its quantized counterpart.
\end{proof}

\section{Proof of the Compute Phases}
\label{app:compute}

This appendix derives each phase only from the two-sided resource assumptions in
\cref{sec:compute}.

\FullThmCompute
\begin{proof}
Matching numerical error to approximation error requires
\[
  \kappa(M)^22^{-2b}\lesssim M^{-\chi}.
\]
Taking base-two logarithms gives \cref{eq:precision-requirement}.  The
two-sided amplification assumption gives the matching
necessity.

If \(1\lesssim\kappa(M)\lesssim M^\rho\), then
\(b=\Theta(\log M)\).  Multiplying per-block work, number of blocks, and
bit-operation cost yields
\[
  C
  \asymp
  M^f M^{1+\chi}(\log M)^\nu
  =
  M^\zeta(\log M)^\nu,
  \qquad\zeta=f+1+\chi.
\]
Substitution into the preceding cost relation gives the budget-matched scale
\[
  M
  \asymp
  C^{1/\zeta}(\log C)^{-\nu/\zeta}.
\]
Substitution into \(M^{-\chi}\) proves the upper bound in
\cref{eq:poly-conditioning}.  Conversely, risk at most \(\varepsilon\)
and \(\kappa(M)\gtrsim1\) require
\(M\gtrsim\varepsilon^{-1/\chi}\) and
\(b\gtrsim\log(1/\varepsilon)\).  The necessary block count and
per-block work give
\[
C\gtrsim\varepsilon^{-(f+1+\chi)/\chi}[\log(1/\varepsilon)]^\nu.
\]
Inverting proves the lower bound, including \(\nu=0\).

If \(\kappa(M)=\exp(\Theta(M^\lambda))\) for \(\lambda>0\), then
\cref{eq:precision-requirement} gives \(b=\Theta(M^\lambda)\), so
\[
  C
  \asymp
  M^{f+1+\chi+\lambda\nu}.
\]
This gives the claimed upper rate.  Conversely, risk at most
\(\varepsilon\) forces \(M\gtrsim\varepsilon^{-1/\chi}\) and
\(b\gtrsim M^\lambda+\log(1/\varepsilon)\), whence
\[
C\gtrsim M^{f+1+\chi+\lambda\nu}
\gtrsim\varepsilon^{-(f+1+\chi+\lambda\nu)/\chi}.
\]
This proves the exponent in both directions.

For adaptive precision, set \(a_j=q_j\) and \(c_j=1\) in
\cref{cor:joint-tail}; then \(D(B)\asymp B^{-\chi}\).  A uniform-width
prefix of length \(M\) and width \(b\) uses \(B=Mb\) bits.  Its omitted tail
is \(\Theta(M^{-\chi})\), and its represented-coordinate quantization is
\(\Theta(2^{-2b})\) because \(\sum_{j\le M}q_j=\Theta(1)\).
Choose \(M\asymp B/\log B\) with a sufficiently small constant and
\(b=\lfloor B/M\rfloor\) for the upper bound in
\cref{eq:uniform-penalty}.  For the converse, set
\(T_B=(B/\log_2 B)^{-\chi}\) and fix \(L>2/\chi\).
If \(M\le LB/\log_2 B\), the tail is at least \(cL^{-\chi}T_B\).
Otherwise \(b\le B/M<(\log_2 B)/L\), so
\(2^{-2b}>B^{-2/L}\gtrsim T_B\) for sufficiently large \(B\).
Every uniform-width prefix therefore pays at least a constant times
\(T_B\), including zero-width prefixes.
\end{proof}

\section{Experimental Details}
\label{app:experiments}

This appendix describes the 1,024-mode synthetic study. Its five-seed
percentile intervals apply to this study; the other experiments use the
BCa procedures and seed or item units specified in their respective captions.

\paragraph{Source families.}
Each run uses \(1024\) modes and bounded logits.  For each
\(\chi\in\{0.5,1,2\}\), two different frequency--amplitude decompositions
produce the same power \(q_j\asymp j^{-(1+\chi)}\).  Training blocks are
independent regenerations.  Held-out cross-entropy is evaluated exactly from
the population probabilities, so no test-set Monte Carlo error is mixed into
the reported training-seed variability.

\paragraph{Grid and methods.}
The grid contains four values of \(n\), five values of \(B\), three
exponents, two decompositions, four allocation methods, and five seeds, for
\(4\cdot5\cdot3\cdot2\cdot4\cdot5=2400\) configuration--seed pairs. The
known-energy method solves the finite energy sequence's water-filling problem.
The universal method
uses \cref{eq:universal-schedule}.  Uniform width assigns a common precision
to a prefix, and shuffled allocation breaks the energy--precision pairing.
Allocations never spend precision beyond the statistical resolution of a
mode.

\paragraph{Uncertainty.}
For each of the \(480\) configurations, we report the mean over five seeds.
The plotted interval is the 2.5 and 97.5 percentile envelope from \(5000\)
bootstrap resamples of the seed mean, using bootstrap seed zero and expanding
the interval when necessary to contain the observed mean.

\paragraph{Reproducibility details.}
The raw rows, summaries, deterministic analytic curves, plotting sources, and
experiment code accompany the paper.  The experiments use local CPUs, one
common exponent specification across configurations, and analytic arithmetic cost as
the theoretical resource.

\section{Mixed Precision in Practice}
\label{app:mixed-precision}

The resource variables in \cref{sec:bit-complexity} deliberately separate
compute precision from storage precision.  In the asymptotic Pareto theorem,
\(b\) is an effective accuracy parameter with unit roundoff
\(\Theta(2^{-b})\), and the modeled format width is proportional to \(b\).
For a concrete floating-point format, total width and significand precision
must instead be distinguished: the numerical horizon uses the actual unit
roundoff, not the total number of stored bits.  This distinction is necessary
in real mixed-precision systems:
\begin{itemize}
  \item A softmax may compute its maximum, exponential, reduction, and division
  in 32-bit floating point (FP32) even when the two surrounding attention
  matrix products use lower precision.  NVIDIA Transformer Engine, for
  example, documents an 8-bit floating-point (FP8) attention path with FP8
  matrix multiplications and FP32 softmax \citep{nvidia2026transformerengine}.

  \item A matrix kernel may accumulate products in a precision wider than its
  input and output formats.  Accumulation precision therefore belongs in the
  arithmetic model; it does not imply that the complete attention
  probability tensor is stored at that precision.

  \item The softmax output can be cast back to a lower-precision format before
  the value product.  Conversely, an eager implementation may retain the
  probability tensor for backward.  These are distinct choices represented by
  softmax effective accuracy and \(b_{\rm st}\), respectively.

  \item Mixed-precision training may retain full-precision master weights and
  optimizer state even when forward weights, activations, and gradients use a
  narrower format \citep{micikevicius2018mixed}.  Such residual state is
  included in \(\operatorname{Mem}_{\rm hp}\), not hidden inside an effective
  average bit width.

  \item Tiled online softmax can avoid materializing the quadratic attention
  matrix by saving normalization statistics and recomputing attention tiles in
  backward \citep{dao2022flashattention,nvidia2026transformerengine}.  This
  changes the storage model and memory traffic, while dense attention still
  performs a quadratic number of score and value arithmetic operations.
\end{itemize}

These observations explain why \cref{eq:bitops-ratio,eq:memory-ratio} are
operation and storage ratios rather than throughput predictions.  Kernel fusion, memory
traffic, device-specific integer or floating-point units, communication, and
parallel scheduling can make the same bit-operation count run at different
wall-clock speeds.  Implementations can therefore differ in speed even when
these analytic counts agree.

%% file: arxiv/sections/appendix_extended_results.tex
\section{Extended Experimental Results}
\label{app:extended-results}

This appendix examines robustness to spectral shape and ordering, numerical
formats, and pretrained-model scale, and provides detailed slope estimates
and reference losses. Captions specify uncertainty for each stochastic quantity.

\subsection{Candidate-Count Replication and Exponent Grid}

The paired \(K=8\) versus \(K=32\) comparison in
\cref{fig:joint-legacy-replication} tests how candidate count affects the
finite-sample construction and trained learner. The profile-by-budget
comparison in \cref{fig:colt-exponent-agreement} tests exponent agreement.
\Cref{tab:spectrum-slopes} gives unbinned numerical slope estimates.
The two-stage study crosses four profiles and prefixes, three sample and state
budgets, \(K=8,32\), and eight paired seeds. Each sample budget restarts the
learner, whereas state budgets reuse the fitted weights. Repeated settings at the
minimum \(n=64\) are aggregated within seed before bootstrapping. Display bins
depend only on \(G_{\mu_M}\).

\begin{figure}[H]
\centering
\includegraphics[width=\linewidth]{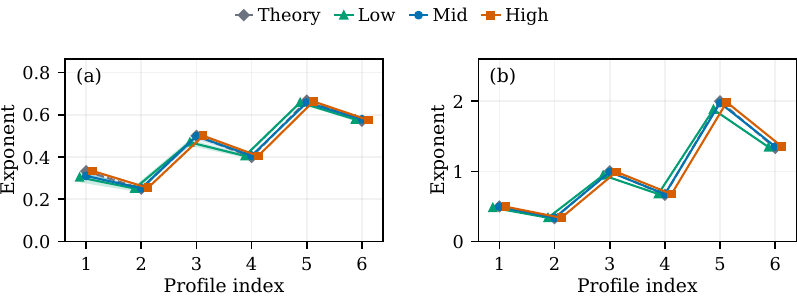}
\caption{\textbf{Six-profile exponent agreement.}
(a) Data-limited and (b) memory-limited fitted exponents for profile indices
\((0.5,0),(0.5,0.5),(1,0),(1,0.5),(2,0),(2,0.5)\). Green, blue, and orange
denote low, middle, and high fixed resource ratios; gray diamonds give the
theoretical exponent. Bands are 95\% BCa intervals over sixteen paired
slopes: eight seeds for each of two frequency--amplitude decompositions
(constant and alternating logit radius).}
\label{fig:colt-exponent-agreement}
\end{figure}

\subsection{Routing, Interventions, and Profile Robustness}

The routing measurements and learned/known/frozen comparison in
\cref{fig:mechanisms}(c,f) relate retrieval quality to prediction error.
The decomposition and interventions in \cref{tab:decomposition-main,tab:interventions-main}
isolate the error sources. \Cref{fig:revision-robustness} tests allocation
when the spectrum departs from an exact power law.

\begin{figure}[H]
\centering
\includegraphics[width=\linewidth]{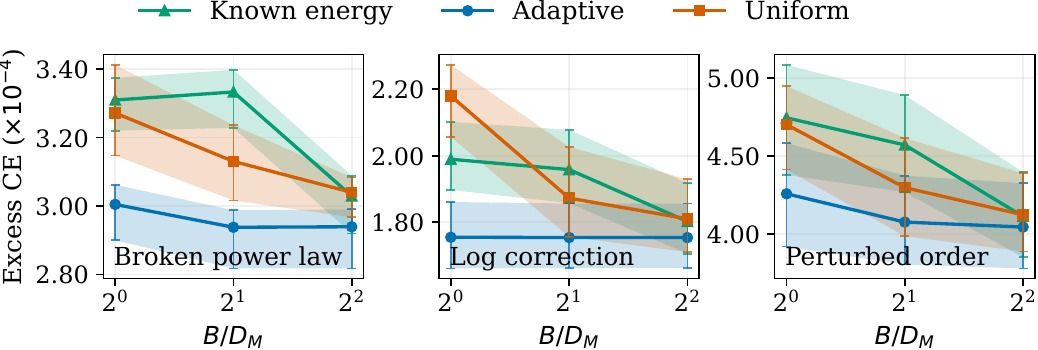}
\caption{\textbf{Allocation beyond exact power laws.}
Left, center, and right show a broken power law, logarithmic correction, and
perturbed ordering, respectively, at
\((\chi,\gamma,M)=(1,0.5,256)\).  Each panel compares known-energy, adaptive,
and uniform allocation. Known-energy uses the true \(q_i\); empirical allocation
uses the count-corrected scores in \cref{app:joint-resource-evaluation}.
Bands are 95\% BCa intervals over eight paired seeds.}
\label{fig:revision-robustness}
\end{figure}

\begin{table}[H]
\centering
\caption{\textbf{Routing and precision contributions to prediction loss.}
At \((\chi,\gamma,M,K)=(1,0.5,64,32)\), data ratio two and unit bit ratio:
augmented loss is \(L_{\mathrm{lr},q}+M^{-\chi}\); Route+ and Format+ are the positive
contrasts \(R_+\) and \(F_+\) defined in Appendix~\ref{app:joint-resource-evaluation}.
Lower is better. Construct. denotes the finite-sample construction.
Each \(\pm\) is the larger deviation to the 95\% BCa endpoints.}
\label{tab:revision-unified}
\small
\setlength{\tabcolsep}{6pt}
\IclrNineTable{interventions}{\input{generated/revision_unified_results}}
\end{table}

\FloatBarrier
\subsection{Stored Precision and Pretrained Transformers}

\begin{figure}[H]
\centering
\includegraphics[width=\linewidth]{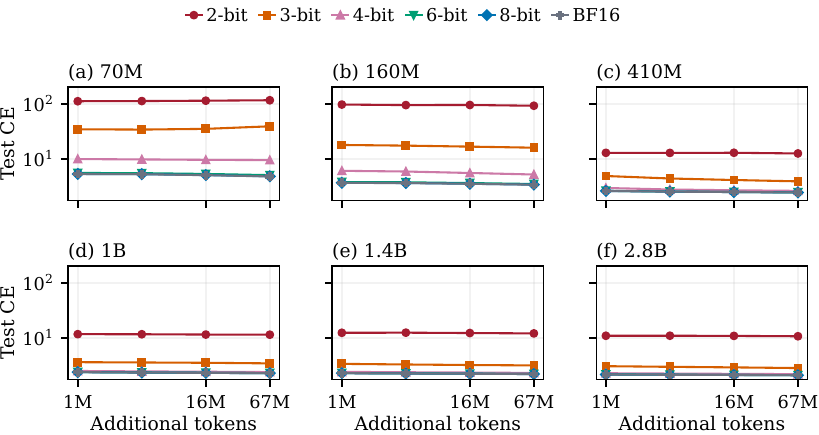}
\caption{\textbf{Complete additional-data and deployment-format scan.}
(a--f) Pythia 70M, 160M, 410M, 1B, 1.4B, and 2.8B test CE after four nested additional-token
budgets. Curves show complete group-128 2/3/4/6/8-bit deployments and BF16;
the shared logarithmic CE axis compares absolute loss and keeps low-bit collapse visible. Bands are 95\%
BCa intervals over eight paired training seeds evaluated on the same 96 test
windows.}
\label{fig:joint-pythia-full}
\end{figure}

At 70M and 160M, the 2-bit calibration CE reaches 114.67 and 97.78 nats,
respectively. These losses exceed the pre-specified \([0,20]\)-nat prediction
range, and the fitting procedure yields no converged, in-range solutions
when calibration includes 2-bit weights. At 410M--2.8B, all four response
functions yield such solutions. A sensitivity analysis fits 4/8-bit and
BF16 separately at every scale and predicts 6-bit loss and the largest
data budget. \Cref{tab:pythia-higher-bit-extrapolation,tab:pythia-higher-bit-interpolation,tab:pythia-all-format-extrapolation,tab:pythia-all-format-interpolation}
report both calibration choices; \cref{app:joint-resource-evaluation}
defines the response functions and evaluation sets.

\begin{figure}[H]
\centering
\includegraphics[width=0.80\linewidth]{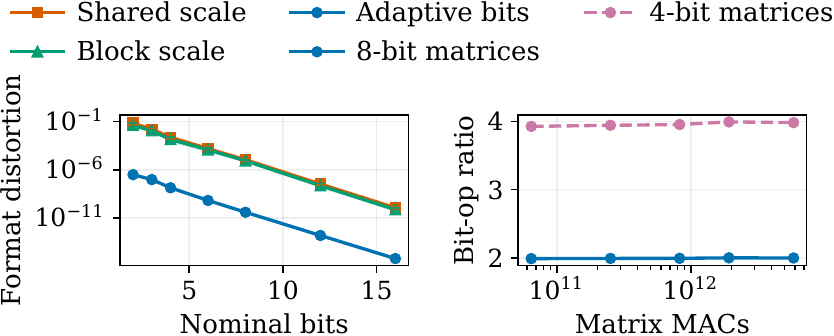}
\caption{\textbf{Stored precision and retained high-precision work.}
Left: format distortion versus nominal bits for shared-scale, block-scale,
and adaptive-bit representations of the causal source.  Right: analytic
bit-operation ratios for 8-bit and 4-bit matrix paths with FP32 softmax
retained.  The source-format curves use paired 95\% BCa intervals; the
right-hand curves evaluate the analytic operation-count formula.}
\label{fig:colt-precision-bridge}
\end{figure}

\begin{figure}[H]
\centering
\includegraphics[width=0.98\linewidth]{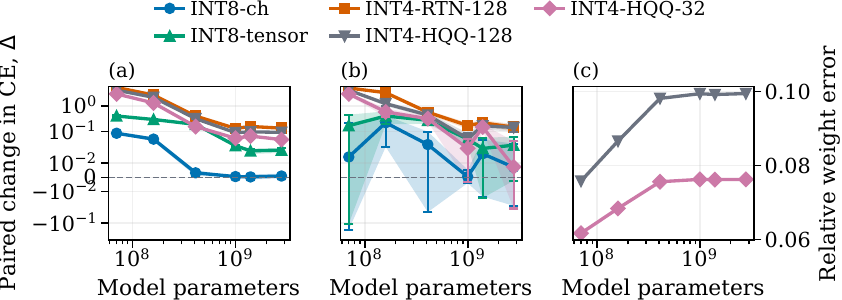}
\caption{\textbf{Finite precision across six Pythia scales.}
Left and center: paired cross-entropy changes on WikiText-103 (96 paired
windows) and LAMBADA (512 paired examples), respectively.  Right:
parameter-weighted relative Frobenius error for the two HQQ formats, averaging \(\|\widehat W-W\|_F/\|W\|_F\) over layers
with weights proportional to their parameter counts.  Blue circles,
green triangles, orange squares, gray down-triangles, and magenta diamonds
denote INT8-channel, INT8-tensor, INT4-RTN-g128, INT4-HQQ-g128, and
INT4-HQQ-g32, respectively.}
\label{fig:revision-pythia}
\end{figure}

\begin{table}[H]
\centering
\IclrPreparePair
  {{\footnotesize\setlength{\tabcolsep}{2.3pt}\input{generated/revision_pythia_bf16_compact}}}
  {{\footnotesize\setlength{\tabcolsep}{2.3pt}\input{generated/revision_spectrum_slopes}}}
\typeout{ICLR9TABLE pythia-bf16 natural=\the\wd\IclrPairA available=\the\IclrPairWidthA}
\typeout{ICLR9TABLE spectrum-slopes natural=\the\wd\IclrPairB available=\the\IclrPairWidthB}
\typeout{ICLR9PAIR secondary gap=\the\IclrPairGap total=\the\linewidth}
\caption{\textbf{Reference losses and spectral slopes.}
Left: Pythia BF16 CE over 96 WikiText-103 windows or 512 LAMBADA items;
bold/underlined marks the lowest/second-lowest value per dataset.
Right: slope against the spectral predictor (reference one), from eight
unbinned seed fits. Each \(\pm\) is the larger deviation to the 95\% BCa endpoints.}
\begin{subtable}[t]{\IclrPairWidthA}
\centering
\caption{Pythia BF16 baseline; lower is better.}
\label{tab:revision-pythia}
\resizebox{\IclrPairWidthA}{!}{\usebox{\IclrPairA}}
\end{subtable}\hspace{\IclrPairGap}%
\begin{subtable}[t]{\IclrPairWidthB}
\centering
\caption{Construction and trained-learner slopes.}
\label{tab:spectrum-slopes}
\resizebox{\IclrPairWidthB}{!}{\usebox{\IclrPairB}}
\end{subtable}
\end{table}

\input{generated/joint_scale_prediction_results}
\FloatBarrier

%% file: generated/revision_unified_results.tex
\begin{tabular}{lccc}
\toprule
Method & Aug. loss & Route+ & Format+ \\
\midrule
Two-stage & 0.0191 {\scriptsize $\pm 0.0002$} & 0.0000 {\scriptsize $\pm 0.0000$} & 0.0000 {\scriptsize $\pm 0.0000$} \\
Construct. & 0.0170 {\scriptsize $\pm 0.0002$} & 0.0002 {\scriptsize $\pm 0.0000$} & 0.0002 {\scriptsize $\pm 0.0001$} \\
Known route & 0.0192 {\scriptsize $\pm 0.0002$} & 0.0000 {\scriptsize $\pm 0.0000$} & 0.0000 {\scriptsize $\pm 0.0000$} \\
Frozen route & 0.0605 {\scriptsize $\pm 0.0008$} & 0.0339 {\scriptsize $\pm 0.0017$} & 0.0000 {\scriptsize $\pm 0.0000$} \\
No precond. & 0.0230 {\scriptsize $\pm 0.0003$} & 0.0003 {\scriptsize $\pm 0.0000$} & 0.0000 {\scriptsize $\pm 0.0000$} \\
Uniform bits & 0.0191 {\scriptsize $\pm 0.0004$} & 0.0000 {\scriptsize $\pm 0.0000$} & 0.0000 {\scriptsize $\pm 0.0000$} \\
\bottomrule
\end{tabular}

%% file: generated/revision_pythia_bf16_compact.tex
\begin{tabular}{lcc}
\toprule
Model & WikiText & LAMBADA \\
\midrule
70M & 6.201 {\scriptsize $\pm 0.047$} & 6.781 {\scriptsize $\pm 0.323$} \\
160M & 4.388 {\scriptsize $\pm 0.056$} & 4.965 {\scriptsize $\pm 0.339$} \\
410M & 2.971 {\scriptsize $\pm 0.059$} & 2.283 {\scriptsize $\pm 0.216$} \\
1B & 2.580 {\scriptsize $\pm 0.059$} & 1.504 {\scriptsize $\pm 0.179$} \\
1.4B & \underline{2.548} {\scriptsize $\pm 0.061$} & \underline{1.470} {\scriptsize $\pm 0.186$} \\
2.8B & \textbf{2.390} {\scriptsize $\pm 0.060$} & \textbf{1.346} {\scriptsize $\pm 0.170$} \\
\bottomrule
\end{tabular}

%% file: generated/revision_spectrum_slopes.tex
\begin{tabular}{lcc}
\toprule
Profile & Construct. & Trained \\
\midrule
$(0.5,0)$ & 0.971 {\scriptsize $\pm 0.040$} & 0.955 {\scriptsize $\pm 0.033$} \\
$(0.5,0.5)$ & 0.933 {\scriptsize $\pm 0.017$} & 0.855 {\scriptsize $\pm 0.028$} \\
$(1,0)$ & 0.969 {\scriptsize $\pm 0.029$} & 0.859 {\scriptsize $\pm 0.018$} \\
$(1,0.5)$ & 0.952 {\scriptsize $\pm 0.023$} & 0.879 {\scriptsize $\pm 0.026$} \\
\bottomrule
\end{tabular}

%% file: generated/joint_scale_prediction_results.tex
\paragraph{Predicting unseen data and storage budgets.}
With 4/8-bit and BF16 calibration, the additive model lowers prediction
error at the largest data budget for every scale
(\cref{tab:pythia-higher-bit-extrapolation}). For 6-bit interpolation, it
raises error at 70M/160M and lowers it at the four larger scales
(\cref{tab:pythia-higher-bit-interpolation}). Each paired 95\% BCa interval
for these additive-minus-data differences excludes zero. An interaction
further improves extrapolation at 160M, 410M and 1.4B; the corresponding
intervals at the other three scales include zero. The benefit of modeling
both resources therefore depends on scale and on which resource is extrapolated.
\Cref{tab:pythia-all-format-extrapolation,tab:pythia-all-format-interpolation}
give the predictions when calibration also includes 2-bit weights.

\begin{table}[htbp]
\centering
\caption{\textbf{Predicting the largest data budget from 4/8-bit and BF16 calibration.}
Root-mean-square error (RMSE, nats; lower is better) over 4/6/8-bit and BF16
test CE at \(2^{26}\) additional tokens. Each response function in
\cref{app:joint-resource-evaluation} is fitted on validation CE at the first
three data budgets. Eight paired training seeds and 96 fixed test windows
are used; 5,000 seed-level BCa resamples refit each function. Smaller
\(\pm\) values are rounded-up symmetric envelopes of the 95\% intervals.
Bold/underline mark the lowest/second-lowest displayed RMSE per row.}
\label{tab:pythia-higher-bit-extrapolation}
\small\setlength{\tabcolsep}{4pt}\begin{tabular}{lrrrr}\toprule Model & Data only & Storage only & Additive & Interaction \\\midrule

70M & \(2.060\,{\scriptstyle\pm\,0.061}\) & \(1.293\,{\scriptstyle\pm\,0.023}\) & \(\underline{1.151}\,{\scriptstyle\pm\,0.080}\) & \(\mathbf{1.149}\,{\scriptstyle\pm\,0.087}\) \\

160M & \(0.805\,{\scriptstyle\pm\,0.016}\) & \(0.619\,{\scriptstyle\pm\,0.034}\) & \(\underline{0.425}\,{\scriptstyle\pm\,0.037}\) & \(\mathbf{0.392}\,{\scriptstyle\pm\,0.036}\) \\

410M & \(0.086\,{\scriptstyle\pm\,0.004}\) & \(0.156\,{\scriptstyle\pm\,0.005}\) & \(\underline{0.058}\,{\scriptstyle\pm\,0.006}\) & \(\mathbf{0.054}\,{\scriptstyle\pm\,0.006}\) \\

1B & \(\underline{0.061}\,{\scriptstyle\pm\,0.003}\) & \(0.096\,{\scriptstyle\pm\,0.002}\) & \(\mathbf{0.041}\,{\scriptstyle\pm\,0.003}\) & \(\mathbf{0.041}\,{\scriptstyle\pm\,0.003}\) \\

1.4B & \(\underline{0.057}\,{\scriptstyle\pm\,0.003}\) & \(0.085\,{\scriptstyle\pm\,0.001}\) & \(\mathbf{0.033}\,{\scriptstyle\pm\,0.004}\) & \(\mathbf{0.033}\,{\scriptstyle\pm\,0.004}\) \\

2.8B & \(\underline{0.055}\,{\scriptstyle\pm\,0.003}\) & \(0.074\,{\scriptstyle\pm\,0.002}\) & \(\mathbf{0.034}\,{\scriptstyle\pm\,0.004}\) & \(\mathbf{0.034}\,{\scriptstyle\pm\,0.004}\) \\

\bottomrule\end{tabular}
\end{table}

\begin{table}[htbp]
\centering
\caption{\textbf{Predicting 6-bit loss from 4/8-bit and BF16 calibration.}
RMSE (nats; lower is better) over the first three data budgets, holding out
6-bit weights from calibration. Response functions, eight paired training
seeds, 96 fixed test windows, refitted 95\% BCa envelopes, and ranking
conventions match \cref{tab:pythia-higher-bit-extrapolation}.}
\label{tab:pythia-higher-bit-interpolation}
\small\setlength{\tabcolsep}{4pt}\begin{tabular}{lrrrr}\toprule Model & Data only & Storage only & Additive & Interaction \\\midrule

70M & \(\mathbf{1.304}\,{\scriptstyle\pm\,0.015}\) & \(1.845\,{\scriptstyle\pm\,0.021}\) & \(\underline{1.842}\,{\scriptstyle\pm\,0.023}\) & \(\underline{1.842}\,{\scriptstyle\pm\,0.023}\) \\

160M & \(\mathbf{0.649}\,{\scriptstyle\pm\,0.016}\) & \(0.732\,{\scriptstyle\pm\,0.017}\) & \(\underline{0.731}\,{\scriptstyle\pm\,0.017}\) & \(\underline{0.731}\,{\scriptstyle\pm\,0.017}\) \\

410M & \(\underline{0.093}\,{\scriptstyle\pm\,0.003}\) & \(0.094\,{\scriptstyle\pm\,0.003}\) & \(\mathbf{0.081}\,{\scriptstyle\pm\,0.003}\) & \(\mathbf{0.081}\,{\scriptstyle\pm\,0.003}\) \\

1B & \(\underline{0.054}\,{\scriptstyle\pm\,0.002}\) & \(0.055\,{\scriptstyle\pm\,0.002}\) & \(\mathbf{0.046}\,{\scriptstyle\pm\,0.002}\) & \(\mathbf{0.046}\,{\scriptstyle\pm\,0.002}\) \\

1.4B & \(0.053\,{\scriptstyle\pm\,0.002}\) & \(\underline{0.047}\,{\scriptstyle\pm\,0.001}\) & \(\mathbf{0.041}\,{\scriptstyle\pm\,0.002}\) & \(\mathbf{0.041}\,{\scriptstyle\pm\,0.002}\) \\

2.8B & \(0.051\,{\scriptstyle\pm\,0.002}\) & \(\underline{0.043}\,{\scriptstyle\pm\,0.002}\) & \(\mathbf{0.039}\,{\scriptstyle\pm\,0.002}\) & \(\mathbf{0.039}\,{\scriptstyle\pm\,0.002}\) \\

\bottomrule\end{tabular}
\end{table}

\begin{table}[htbp]
\centering
\caption{\textbf{Predicting the largest data budget with 2-bit weights included in calibration.}
RMSE (nats; lower is better) over all six deployment formats at \(2^{26}\)
additional tokens, fitted using 2/4/8-bit and BF16 validation CE at the
first three budgets. This bounded-response comparison covers 410M--2.8B.
The 70M/160M 2-bit calibration losses exceed its \([0,20]\)-nat range;
\cref{fig:joint-pythia-full} shows their measured quantization losses.
Eight paired seeds, 96 test windows, 95\% BCa envelopes and ranking
conventions match \cref{tab:pythia-higher-bit-extrapolation}.}
\label{tab:pythia-all-format-extrapolation}
\small\setlength{\tabcolsep}{4pt}\begin{tabular}{lrrrr}\toprule Model & Data only & Storage only & Additive & Interaction \\\midrule

410M & \(3.776\,{\scriptstyle\pm\,0.044}\) & \(1.354\,{\scriptstyle\pm\,0.008}\) & \(\underline{1.328}\,{\scriptstyle\pm\,0.030}\) & \(\mathbf{1.327}\,{\scriptstyle\pm\,0.036}\) \\

1B & \(3.426\,{\scriptstyle\pm\,0.032}\) & \(1.088\,{\scriptstyle\pm\,0.010}\) & \(\underline{1.050}\,{\scriptstyle\pm\,0.019}\) & \(\mathbf{1.045}\,{\scriptstyle\pm\,0.022}\) \\

1.4B & \(3.724\,{\scriptstyle\pm\,0.031}\) & \(1.094\,{\scriptstyle\pm\,0.011}\) & \(\underline{1.063}\,{\scriptstyle\pm\,0.020}\) & \(\mathbf{1.060}\,{\scriptstyle\pm\,0.028}\) \\

2.8B & \(3.255\,{\scriptstyle\pm\,0.024}\) & \(0.797\,{\scriptstyle\pm\,0.005}\) & \(\underline{0.772}\,{\scriptstyle\pm\,0.006}\) & \(\mathbf{0.771}\,{\scriptstyle\pm\,0.007}\) \\

\bottomrule\end{tabular}
\end{table}

\begin{table}[htbp]
\centering
\caption{\textbf{Predicting 3/6-bit loss with 2-bit weights included in calibration.}
RMSE (nats; lower is better) at the first three data budgets, holding out
3/6-bit weights from 2/4/8-bit/BF16 calibration. This comparison covers
410M--2.8B, with the calibration range and scale coverage described in
\cref{tab:pythia-all-format-extrapolation}. Eight paired training seeds,
96 test windows, 95\% BCa envelopes and ranking conventions match
\cref{tab:pythia-higher-bit-extrapolation}.}
\label{tab:pythia-all-format-interpolation}
\small\setlength{\tabcolsep}{4pt}\begin{tabular}{lrrrr}\toprule Model & Data only & Storage only & Additive & Interaction \\\midrule

410M & \(1.976\,{\scriptstyle\pm\,0.016}\) & \(1.573\,{\scriptstyle\pm\,0.032}\) & \(\mathbf{1.568}\,{\scriptstyle\pm\,0.033}\) & \(\underline{1.569}\,{\scriptstyle\pm\,0.033}\) \\

1B & \(1.870\,{\scriptstyle\pm\,0.008}\) & \(\underline{1.512}\,{\scriptstyle\pm\,0.014}\) & \(\mathbf{1.511}\,{\scriptstyle\pm\,0.014}\) & \(\mathbf{1.511}\,{\scriptstyle\pm\,0.014}\) \\

1.4B & \(2.152\,{\scriptstyle\pm\,0.014}\) & \(\underline{1.585}\,{\scriptstyle\pm\,0.012}\) & \(\mathbf{1.584}\,{\scriptstyle\pm\,0.012}\) & \(\mathbf{1.584}\,{\scriptstyle\pm\,0.012}\) \\

2.8B & \(1.874\,{\scriptstyle\pm\,0.003}\) & \(\underline{1.145}\,{\scriptstyle\pm\,0.003}\) & \(\mathbf{1.144}\,{\scriptstyle\pm\,0.003}\) & \(\mathbf{1.144}\,{\scriptstyle\pm\,0.003}\) \\

\bottomrule\end{tabular}
\end{table}

%% file: arxiv/sections/10_experiment_appendix.tex
\section{Experiment and Reproducibility Details}
\label{app:experiment-details}

We retain the scientific configuration and seed- or item-level measurements
for every reported result. This appendix specifies the sampling scheme,
numerical formats, software environment, and aggregation procedure; the
accompanying code instantiates the complete grids.

\paragraph{Software and execution framework.}
The local central-processing-unit (CPU) runs used Python 3.11.15 with NumPy
2.0.2, SciPy 1.13.1, Pandas 2.3.3, PyArrow 21.0.0, and Matplotlib 3.9.4.
The graphics-processing-unit (GPU) runs used Python 3.12.3, PyTorch 2.13.0
with CUDA 13.0, TorchAO 0.18.0, Transformers 5.16.1, Datasets 5.0.1, and
Hugging Face Hub 1.29.0. The \texttt{finitebit\_exp} 0.1.0 package defines the
scientific grids in Python. PyTorch automatic differentiation supplies the
Stage A and Stage V gradients in the mini-batch implementation. The
weighted-average FP64 reference implements the same explicit scalar gradient
on its theorem-specified active set. Independent worker processes run the
two-stage training jobs. Pythia models are
loaded once per model--format bundle and evaluated in inference mode. CPU
experiments ran on an Apple M5 system; GPU experiments used an NVIDIA RTX PRO
6000 Blackwell Server Edition. Raw seed- and item-level measurements are
stored in Parquet files.

\paragraph{Pythia items and operation count.}
Tokenization is performed once per model tokenizer on the raw WikiText-103
test text, after which contiguous 2,049-token windows provide 2,048 next-token
losses per item.  All numerical formats receive identical item indices.  The analytic
operation count includes the query, key, value, and output projections; the
two attention matrix products; and both feed-forward products.

\paragraph{Reproducible result generation.}
Empirical curves are generated from seed- or item-level records using the
uncertainty procedure specified in each caption.

The data--memory study crosses six energy--dimension profiles, seven model
cutoffs from 16 to 1,024, and three normalized data and state budgets. It
evaluates population excess CE over each finite source, including coordinates
predicted with zero logit. The mini-batch two-stage study additionally uses
the analytic tail proxy described in \cref{app:joint-resource-evaluation};
the fixed-source weighted-SGD study measures its omitted-coordinate loss directly. The
pretrained-model comparisons use identical items for every format, including
96 paired WikiText-103 windows per model--format configuration.

\subsection{Additional coupling, robustness, and pretrained-model experiments}

The additional experiments comprise a dyadic coupling test with 16 paired
seeds, a two-stage study with eight paired seeds, robustness tests for three
departures from exact power laws, and six Pythia sizes under six deployment
formats on two datasets. The two-stage learner differentiates
the Bernoulli loss through the attention-weighted network output and updates
the stored value parameters; the direct finite-sample parameter construction
is included as a separate likelihood-based comparator.

\paragraph{Six-scale weight-only post-training quantization.}
All Pythia models use the fixed \texttt{step143000} revision.  WikiText-103
uses 96 common non-overlapping 2,048-token windows, and LAMBADA uses 512 common
examples with last-token cross-entropy.  The formats are bfloat16 (BF16),
8-bit integer (INT8) with one tensor scale, INT8 with one scale per output
channel, 4-bit integer (INT4) round-to-nearest with groups of 128 weights, and
half-quadratic quantization (HQQ) for INT4 groups of 128 or 32
weights.  Every reported difference is paired with the BF16 result on the
same item.  Eager attention receives BF16 scores and returns FP32 softmax
probabilities in the measured software stack. The evaluation microbatch is
one for both datasets.

The item-level
records retain model, format, dataset, item index, cross-entropy, represented
weight error, quantized-parameter fraction, and output-logit divergence.
Appendix~\ref{app:extended-results} reports the detailed figures and tables.

%% file: sections/joint_resource_methods.tex
\section{Joint Resource Evaluation}
\label{app:joint-resource-evaluation}

\paragraph{Fixed-source comparisons.}
We vary data and value-state budgets while holding the generating distribution
fixed. The same-marginal experiment samples the infinite dyadic source at
\(\chi\in\{1,2\}\), \(\gamma=1\), with a common normalization for both
pairings. A sparse clipped Bernoulli estimator uses unbounded storage, isolating
the data exponent. Unobserved coordinates contribute their actual probability
mass times the zero-prediction excess cross-entropy. The seven sample budgets
are \(2^8,2^{10},\ldots,2^{20}\); the primary slope uses the last four budgets
for every seed and pairing.

\paragraph{The theorem's learning algorithm.}
For \((\chi,\gamma)=(0.5,0),(0.5,0.5),(1,0),(1,0.5)\), the source
contains 256 energy groups, each with \(d_j=\lceil j^\gamma\rceil\)
scalar coordinates. This source size is distinct from the number of sampled
regeneration blocks. We compare \(K=8,32\) candidates with eight independent seeds. At each
\(n\in\{2^8,2^{10},\ldots,2^{16}\}\), we restart the projected,
preconditioned SGD algorithm of \cref{full:thm:causal-realization} from zero,
using its active set, step sizes and weighted average. We implement the
numerical reference in FP64. Calibration uses
\(S_A=\lceil4K[\log(eK)+\sqrt n]\rceil\) independent blocks. Each fitted
state is encoded at \(B_V=8,16,32,64,128,256\) bits using the theorem's
midpoint scalar format. Inactive fitted coordinates are zero before encoding;
positive-width midpoint decoding also acts on those zero entries. An
oracle-coordinate clipped estimator uses the same labeled blocks with the
selected coordinate revealed.
We report both value and routing bits, calibration and prediction samples,
and the losses on 65,536 shared independent test contexts. Oracle and learned
routes, before and after quantization, define signed counterfactual differences
on these identical contexts.

\paragraph{Additional training and deployed storage.}
Starting from the Pythia 70M, 160M, 410M, 1B, 1.4B, and 2.8B checkpoints at
step143000, each of eight seeds follows one continued-training trajectory with checkpoints after
\(2^{20},2^{22},2^{24},2^{26}\) additional WikiText-103 tokens. Training uses
2048-token contexts, an effective batch of 32,768 tokens, FP32 parameters and
optimizer states, and BF16 computation. AdamW uses learning rate \(10^{-5}\),
\(\beta=(0.9,0.95)\), weight decay 0.1, gradient clipping at 1, and an
eight-step warm-up followed by a constant learning rate. Deployment uses BF16
or group-128 symmetric 2/3/4/6/8-bit round-to-nearest linear weights. Integer
payloads, scales, padding, retained BF16 parameters and metadata are all counted.
Evaluation uses 96 fixed test windows and 96 fixed validation windows.

\paragraph{Prediction on held-out resource configurations.}
We fit data-only, storage-only, additive and interacting response models on the
first three training budgets at 2/4/8-bit and BF16. The 3/6-bit formats test
storage interpolation, while the largest training budget tests data
extrapolation. Fits use validation losses; test-window losses remain untouched
until evaluation. Comparisons use paired training-seed uncertainty; the common
initial pretrained checkpoint is evaluated once and its uncertainty uses windows.
Let \(t=T/2^{20}\) denote additional tokens in units of the smallest budget,
and let \(s\) be deployment-package bytes divided by the BF16 package size.
Writing \(x=t^{-\alpha}\) and \(y=s^{-\beta}-1\), the four fitted CE models are
\[
\underbrace{c+ax}_{\text{data only}},\qquad
\underbrace{c+dy}_{\text{storage only}},\qquad
\underbrace{c+ax+dy}_{\text{additive}},\qquad
\underbrace{c+ax+dy+hxy}_{\text{interaction}}.
\]
We fit each model scale separately by least squares on validation CE, with
\(c,a,d\in[0,20]\), \(\alpha,\beta\in[0.01,4]\), and \(h\in[-20,20]\).
The fitting procedure requires convergence and calibration predictions in
the pre-specified range \([0,20]\) nats; held-out predictions are checked
against the same range. After observing the large 2-bit losses at 70M and
160M, we also analyze 4/8-bit and BF16 calibration, with 6-bit interpolation
and 4/6/8-bit/BF16 extrapolation. This sensitivity analysis uses the same
response functions and parameter bounds. Root-mean-square error (RMSE)
compares predicted and measured test CE across the held-out resource
configurations, after averaging the 96 windows within each configuration.

\paragraph{Route-adjusted and complete-source losses.}
Let \(L_{\mathrm{lr},q},L_{\mathrm{lr},u},L_{\mathrm{or},u}\) denote
finite-source excess CE for learned-route quantized, learned-route
unquantized, and known-route unquantized predictions. All subtract the same
source Bayes loss. The
route-adjusted analysis defines \(E=L_{\mathrm{lr},q}+M^{-\chi}\), \(R_+=(L_{\mathrm{lr},u}-L_{\mathrm{or},u})_+\), and
\(F_+=(L_{\mathrm{lr},q}-L_{\mathrm{lr},u})_+\). Its spectrum comparison uses \(E-R_+\); its component sum is
an upper envelope of \(E\). Here \(M\) is the energy-group cutoff and
\(M^{-\chi}\) is an analytic tail proxy. Our fixed-source study instead
evaluates the generating distribution without adding that proxy and retains all four
counterfactual losses. The oracle loss splits over actual held-out targets
inside and outside \(A_n\); adding the signed differences \(L_{\mathrm{lr},u}-L_{\mathrm{or},u}\) and
\(L_{\mathrm{lr},q}-L_{\mathrm{lr},u}\)
gives exactly the complete-source excess loss \(L_{\mathrm{lr},q}\). The tail term is the actual
zero-prediction KL on omitted targets. This identity holds per test context,
before any averaging, by \cref{prop:fixed-source-counterfactual}.

\paragraph{Empirical and known-energy allocation.}
The robustness study's adaptive allocator uses the count-corrected score
\[
s_i=\left[p_i\widehat\theta_i^2-\frac{p_i}{\max\{1,N_i\}}\right]_+,
\]
where \(N_i\) is the observed count and \(\widehat\theta_i\) the bounded
Bernoulli maximum-likelihood estimate. Both this rule and the known-energy
reference water-fill \(\sum_i s_i4^{-b_i}\), floor widths, and assign
remaining bits by marginal gains; the reference substitutes \(s_i=q_i\).
The empirical subtraction prioritizes resolved coordinates, whereas the
known-energy reference minimizes the true-energy distortion objective. These distinct
objectives allow different finite-sample CE even at the same state budget.
The unknown-exponent schedule in \cref{thm:adaptive} is a separate,
deterministic construction.

\begin{figure}[htbp]
\centering
\includegraphics[width=\linewidth]{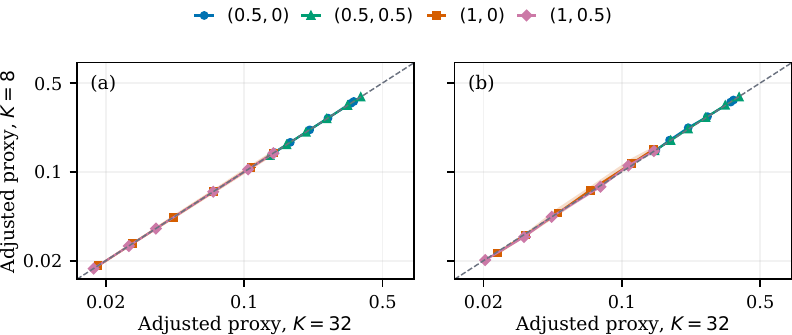}
\caption{\textbf{Candidate-count replication under route-adjusted loss.}
(a) Direct finite-sample construction; (b) mini-batch two-stage implementation.
Horizontal and vertical axes show \(E-R_+\) at \(K=32\) and \(K=8\),
respectively, using identical spectral bins.
Legends give \((\chi,\gamma)\); the dashed identity line is deterministic.
Intervals on both axes are 95\% BCa intervals over eight paired seeds.}
\label{fig:joint-legacy-replication}
\end{figure}

\begin{figure}[htbp]
\centering
\includegraphics[width=\linewidth]{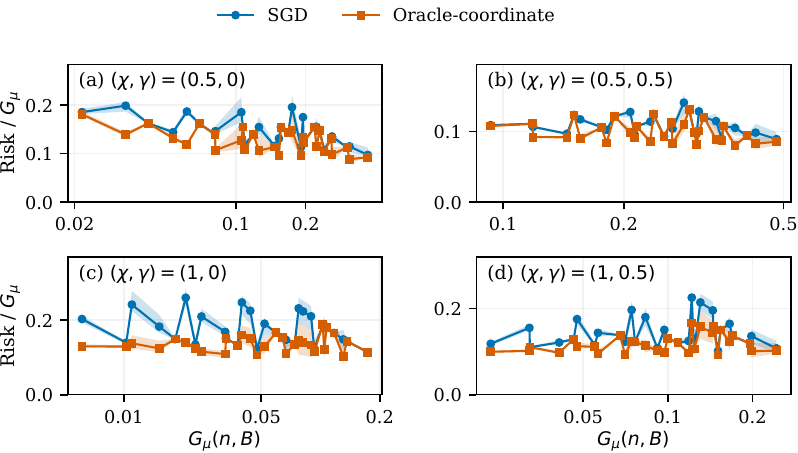}
\caption{\textbf{Unbinned risk-to-spectrum ratios on fixed sources.}
Panels (a)--(d) use the four profiles of \cref{fig:spectrum-collapse}, at \(K=32\).
The horizontal axis is \(G_\mu\); the vertical axis is measured excess CE
 divided by \(G_\mu\). Each point represents a distinct sample--state budget pair;
lines connect budget pairs ordered by
\(G_\mu\), rather than a fitted response. Blue/orange denote weighted SGD and
the oracle-coordinate estimator. Bands retain the raw 95\% BCa intervals
over eight paired seeds.}
\label{fig:joint-fixed-diagnostics}
\end{figure}

%% file: sections/joint_resource_proofs.tex
\section{Total Resources and Fixed-Source Evaluation}
\label{app:joint-resource-proofs}

The following consequence accounts jointly for the value-learning and routing
stages. Write \(q_{(r)}\) for the decreasing coordinate energies and \(\xi\)
for their decay exponent.

\begin{proposition}[Total sample and learned-state requirements]
\label{prop:total-resources}
Fix an integer \(K\ge1\), \(R>0\), and \(0<\eta\le1\). Use the
iid-candidate source, public frequencies and calibration supervision of
\cref{full:thm:causal-realization}, with unit midpoint formats. Suppose
\(cr^{-(1+\xi)}\le q_{(r)}\le Cr^{-(1+\xi)}\), where \(\xi>0\)
and \(0<c\le C<\infty\) are fixed. Let \(C_L\ge1\) be a constant in
its four-term risk upper bound. Set
\[
C_D=2+C/\xi,\quad a=\max\{1,(1+\xi)/(2\ln2)\},\quad
C_B=C(1+1/\xi)(2a)^\xi.
\]
Let \(C_A(\eta)\) be an upper-bound constant in the routing hitting-time
bound of \cref{full:prop:online-routing} for thresholds at most \(1/16\).
For \(0<\varepsilon\le1\), choose
\begin{align*}
h&=\min\{1/16,\varepsilon/(4C_L)\},\\
n&=\max\{1,\lceil(4C_LC_D/\varepsilon)^{(1+\xi)/\xi}\rceil\},\\
B_V&=\max\{1,\lceil(4C_LC_B/\varepsilon)^{1/\xi}\rceil\}.
\end{align*}
For \(K\ge2\), take
\begin{align*}
S_A&=\lceil C_A(\eta)K\{h^{-1/2}+\ln(eK)\}\rceil,\\
B_A&=K\{\lceil\log_2(S_A+1)\rceil+\lceil\log_2(K+1)\rceil\};
\end{align*}
for \(K=1\), take \(S_A=B_A=0\). If deployment arithmetic has mean
squared logit error at most \(\varepsilon/(4C_L)\), the weighted-SGD
learner has expected excess risk at most \(\varepsilon\), uses
\(N=n+S_A\) blocks and retains at most \(B=B_V+B_A\) learned bits.
At fixed source and routing constants,
\[
N=O(\varepsilon^{-(1+\xi)/\xi}),\quad B=O(\varepsilon^{-1/\xi}),
\qquad S_A/n\to0,\quad B_A/B_V\to0.
\]
For \(\chi>0\), \(\gamma\ge0\), block energies
\(a_j\asymp j^{-(1+\chi)}\), dimensions \(d_j\asymp j^\gamma\)
and equal coordinate energies \(q_{j,\ell}=a_j/d_j\) within each block,
the coordinate exponent is \(\xi=\chi/(1+\gamma)\).
\end{proposition}

\begin{proof}
For any integer \(m\ge1\), the energy envelope and integral comparison give
\[
\Phi_\mu(n^{-1})\le m/n+C\sum_{r>m}r^{-1-\xi}
\le m/n+(C/\xi)m^{-\xi}.
\]
Choose \(m=\lceil n^{1/(1+\xi)}\rceil\). Since
\(n^{1/(1+\xi)}\le m\le2n^{1/(1+\xi)}\), this is at most
\(C_Dn^{-\xi/(1+\xi)}\), including \(n=1\).

For a storage budget \(b\ge2a\), put \(m=\lfloor b/a\rfloor\) and
allocate \(b_r=(1+\xi)\log_2(m/r)/2\) for \(1\le r\le m\), with
zero widths thereafter. The inequality
\[
\sum_{r=1}^m\ln(m/r)\le\int_0^m\ln(m/x)\,dx=m
\]
shows that total width is at most \((1+\xi)m/(2\ln2)\le am\le b\).
Each retained distortion term is at most \(Cm^{-1-\xi}\), and the tail
is at most \((C/\xi)m^{-\xi}\). Since \(m\ge b/(2a)\),
\[
D_q(b)\le C(1+1/\xi)m^{-\xi}\le C_Bb^{-\xi}.
\]
For \(1\le b<2a\), zero allocation has distortion at most
\(\sum_rq_{(r)}\le C(1+1/\xi)\le C_Bb^{-\xi}\), proving the
same bound for all \(b\ge1\). Summability controls every countable tail.
The theorem's scalar implementation floors public continuous widths; this
cannot exceed the bit budget and multiplies each distortion term by at most
four. Its constant is already included in \(C_L\). A public allocation
needs no additional learned index code.

The ceilings in \(n,B_V\) can only decrease their nonincreasing error
functionals, so each is at most \(\varepsilon/(4C_L)\). For \(K\ge2\),
the positive threshold \(h\le1/16\) lies in the routing theorem's range.
The prescribed \(S_A\) exceeds its hitting-time bound. Margin increases
with row visits, and counts couple monotonically as calibration blocks are
added; leakage therefore remains below \(h\le\varepsilon/(4C_L)\).
For \(K=1\), leakage is zero. The arithmetic premise bounds the fourth
term. Substitution in \cref{full:eq:causal-realization} now gives risk at
most \(C_L\cdot4\varepsilon/(4C_L)=\varepsilon\).

The independent stages consume distinct blocks, totaling \(n+S_A\).
Each routing row is reconstructed from its count in \(\{0,\ldots,S_A\}\)
and its destination or unseen symbol. The displayed integer-width fields
encode all these possibilities. Concatenating their \(B_A\)-bit code with
the value code gives at most \(2^{B_A+B_V}\) retained states. Numerical
reconstruction is charged to the arithmetic term.

For sufficiently small \(\varepsilon\), \(h=\varepsilon/(4C_L)\).
The explicit resource choices then yield
\begin{align*}
S_A/n&=O_K(\varepsilon^{1/2+1/\xi}+\varepsilon^{1+1/\xi}),\\
B_A/B_V&=O_K(\varepsilon^{1/\xi}\ln(1/\varepsilon)).
\end{align*}
Both tend to zero since \(\xi>0\). Adding the routing costs gives the
two total-resource orders.

Finally, every coordinate in block \(j\) has energy comparable to
\(j^{-(1+\chi+\gamma)}\), uniformly over its \(d_j\asymp j^\gamma\)
coordinates. Upper and lower threshold comparisons therefore give
\[
\#\{q_i\ge t\}\asymp
\sum_{j\le c_0t^{-1/(1+\chi+\gamma)}}j^\gamma
\asymp t^{-(1+\gamma)/(1+\chi+\gamma)},
\]
with separate positive \(c_0\) for the upper and lower comparisons.
Inversion proves \(q_{(r)}\asymp r^{-(1+\chi+\gamma)/(1+\gamma)}\)
and the asserted \(\xi\). The argument allows bounded nonmonotonicity
in the original block labels.
\end{proof}

\begin{proposition}[Fixed-source loss and paired counterfactuals]
\label{prop:fixed-source-counterfactual}
Let \(p\) be a fixed probability distribution on a finite or countable
coordinate set, \(|\theta_i|\le\Delta_i\le R\),
\(q_i=p_i\Delta_i^2\), and
\(d(\theta,z)=D_{\rm KL}(\operatorname{Ber}(\sigma(\theta))
\Vert\operatorname{Ber}(\sigma(z)))\).
Conditional on a fitted table \(|w_i|\le R\) supported on \(S\), a fresh
oracle-routed example has excess risk
\begin{align*}
L(w,\mathrm{oracle})
&=\sum_{i\in S}p_i d(\theta_i,w_i)+\sum_{i\notin S}p_i d(\theta_i,0),\\
0\le\sum_{i\notin S}p_i d(\theta_i,0)&\le\tfrac18\sum_{i\notin S}q_i.
\end{align*}
If all \(|\theta_i|=a\), the omitted term equals \(p(S^c)d(a,0)\),
also for a data-dependent \(S\).

On the same full candidate context, let \(L_{rv}\) be excess loss with
route \(r\in\{0,1\}\) (oracle, learned) and value state
\(v\in\{0,1\}\) (unquantized, quantized), with all four predicted
logits in \([-R,R]\). Pointwise and under every
common empirical or population average,
\begin{align*}
L_{11}&=L_{00}+(L_{10}-L_{00})+(L_{11}-L_{10})\\
      &=L_{00}+(L_{01}-L_{00})+(L_{11}-L_{01}).
\end{align*}
The two format contrasts differ by \(L_{11}-L_{10}-L_{01}+L_{00}\);
individual contrasts may be negative. If two predictions in \([-R,R]\)
agree outside an event \(A\), their expected CE differs by at most
\(2R\Pr(A)\). If \(A\) is that some candidate lies in \(T\), then
\(\Pr(A)\le Kp(T)\).
\end{proposition}

\begin{proof}
For target coordinate \(i\), conditional expected CE minus Bayes CE is
exactly \(d(\theta_i,z)\). Conditional on training, a fresh target retains
law \(p\), giving \(\sum_i p_i d(\theta_i,w_i)\). Nonnegative terms
permit countable summation by monotone convergence. Splitting at \(S\),
where the fitted table is zero on its complement, proves the oracle identity.

Put \(A(z)=\ln(1+e^z)\). Direct substitution gives
\(d(\theta,z)=A(z)-A(\theta)-A'(\theta)(z-\theta)\).
Since \(0<A''(z)=\sigma(z)(1-\sigma(z))\le1/4\), Taylor's integral
remainder implies \(0\le d(\theta,0)\le\theta^2/8\).
Multiply by \(p_i\), use \(\theta_i^2\le\Delta_i^2\), and sum to
obtain the tail bound. Exchanging Bernoulli outcomes gives
\(d(-a,0)=d(a,0)\), proving the constant-magnitude formula conditional
on each fitted support.

For learned routing, \(\sum_s a_{Qs}w_{C_s}\) can be nonzero even when
the target is outside \(S\), because a distractor may lie in \(S\).
The four losses therefore evaluate the full context with a common target
law. Cancelling intermediate terms proves both identities pointwise;
linearity preserves them under averaging. Subtracting the two format
contrasts gives their stated interaction. For a concrete sign example, the
one-bit midpoint format on \([-1,1]\) sends \(w=0.1\) to \(0.5\).
At \(\theta=0.5\), the contrast is \(-d(0.5,0.1)<0\); at
\(\theta=0.1\), it is \(d(0.1,0.5)>0\).

Finally, per-label CE \(\ell_y(z)=A(z)-yz\) has derivative
\(\sigma(z)-y\) of absolute value at most one. Hence
\(|\ell_y(z)-\ell_y(z')|\le|z-z'|\le2R\). Multiplying by the
indicator of \(A\) and averaging proves the approximation bound, also
after subtracting the common Bayes risk. The union bound gives
\(\Pr(A)\le\sum_{s=1}^K\Pr(C_s\in T)=Kp(T)\); it needs only
the candidate marginals, not independence.
\end{proof}